\documentclass{article}

\usepackage[preprint]{neurips_2026}
\usepackage{subcaption}
\usepackage{wrapfig}
\usepackage{amsmath,amssymb}
\usepackage{booktabs}
\usepackage{longtable}
\usepackage{array}

\usepackage[utf8]{inputenc} 
\usepackage[T1]{fontenc}    
\usepackage{hyperref}       
\usepackage{adjustbox}
\usepackage{url}            
\usepackage{booktabs}       
\usepackage{amsfonts}       
\usepackage{nicefrac}       
\usepackage{microtype}      
\usepackage{xcolor}         
\usepackage{amsmath}
\usepackage{amssymb}
\usepackage{subcaption}
\usepackage{siunitx}
\usepackage{multirow}
\usepackage{wrapfig}
\usepackage{array}
\usepackage{amsthm}
\usepackage{enumitem}
\usepackage{amsthm}
\usepackage{hyperref}
\usepackage{booktabs}   
\usepackage{makecell}   

\usepackage[linesnumbered,ruled,vlined]{algorithm2e}
\usepackage[ruled,vlined]{algorithm2e}
\SetKwProg{Procedure}{Procedure}{}{}  
\SetKw{KwRet}{Return}

\usepackage[most]{tcolorbox}

\newtcolorbox{insightbox}{
    colback=red!6,
    colframe=black!75,
    boxrule=0.8pt,
    arc=4pt,
    left=7pt,
    right=7pt,
    top=5pt,
    bottom=5pt,
    before skip=6pt,
    after skip=6pt,
    breakable
}

\newtheorem{theorem}{Theorem}

\newtheorem{proposition}{Proposition}
\newtheorem{corollary}{Corollary}
\newtheorem{definition}{Definition}

\title{A Mixture-Gated Theory of Anisotropic Representation Dynamics in Transformers}

\author{%
  Nguyen Do \qquad My T. Thai \\
  Department of Computer \& Information Science \& Engineering \\
  University of Florida, Gainesville, FL, USA
}

\begin{document}

\maketitle

\begin{abstract}
Small differences between internal representations can shrink or grow as they pass through a Transformer layer, affecting how information is preserved and how sensitive the model is to perturbations. Yet it remains unclear when changes in attention weights can turn a representation difference that would otherwise shrink into one that grows. We develop a \textbf{Mixture-Gated Theory of Anisotropic Representation Dynamics} that separates self-attention into value transport, query gating, and key gating. Our analysis reveals how the geometry of attended representations constrains gating and how interactions between these mechanisms can reverse contraction induced by value transport. We formalize this phenomenon as gating-induced regime switching, where a representation difference contracts under fixed attention routing but expands under full attention. We derive a necessary and sufficient spectral criterion that certifies the existence or absence of switching and identifies switching directions without exhaustive search. Beyond existence, our theory characterizes maximum relative switching strength, certifies the severity of contraction that gating can overcome, and determines the largest dimension of a subspace in which every nonzero direction switches. Experiments across four Transformer architectures validate these predictions through analytical and finite-forward comparisons, revealing distinct and joint query--key gating mechanisms and nontrivial switching subspaces. Our results provide a mechanistic and spectral understanding of how attention routing can alter local representation dynamics.
\end{abstract}

\section{Introduction}
\label{sec:introduction}

Many practical problems involving Transformers depend on how small changes to internal representations evolve during computation. Some changes gradually disappear, while others grow and affect the model's behavior. For example, in model editing, correcting a wrong fact should change the intended answer without affecting unrelated behaviors~\cite{meng2022rome}. An overly strong edit may unintentionally change other answers, so we want useful and independent changes to be amplified while unwanted changes are suppressed. A similar challenge appears in adversarial attacks, where an attacker may need to query a target model many times to find a small perturbation that flips its prediction~\cite{ilyas2018blackbox}. Understanding which perturbation directions are likely to be amplified could help identify promising attacks without relying entirely on repeated trial-and-error queries. This problem also appears in many other Transformer applications, including representation steering, model robustness, and long-context modeling, where preserving useful changes and suppressing unwanted ones is important~\cite{turner2024activation,liu2024lost}. 

Previous studies have provided important insights into how Transformer representations evolve. Research on representation collapse and signal propagation explains how repeated attention can make token representations increasingly similar or reduce their diversity~\cite{dong2021attention,noci2022signal,kedia2024transformers}. Other work studies when token representations converge or diverge across layers~\cite{pham2025dynamical,tomihari2025recurrent}. Jacobian and spectral analyses further reveal how sensitive Transformer representations are to perturbations and how this sensitivity depends on attention computations~\cite{kim2021lipschitz,pmlrv235bao24b,aubry2025transformer,saratchandran2025spectral,ormaniec2025does}. However, knowing that a representation difference grows does not necessarily explain what causes it to grow. In particular, we still need to distinguish changes caused by the information being passed through attention from those caused by attention selecting different information. This leads to our central question: \emph{\textbf{Can attention routing turn a representation difference that would otherwise shrink into one that grows?}}

To answer this question, we develop a \emph{Mixture-Gated Theory of Anisotropic Representation Dynamics}.
Our starting point is that the same Transformer computation can shrink some representation differences while amplifying others.
We refer to this direction-dependent behavior as \emph{anisotropic representation dynamics}.
To understand why this happens, we examine how self-attention changes its output.
Attention combines information from different tokens, and a small perturbation can affect this process in two ways.
First, it can change the information carried by the tokens themselves, which we call \emph{value transport}.
Second, it can change how much attention each token receives, which we call \emph{gating}.
We further distinguish \emph{query gating}, which changes attention through the query, from \emph{key gating}, which changes attention through the keys.
This decomposition allows us to compare how representation differences evolve when only the transported information changes, and when changes in attention allocation also contribute to their evolution.
A representation difference may therefore evolve very differently depending on whether attention routing responds to it (See Figure~\ref{fig:gating_regime_switch}).

\begin{figure}[htp]
    \centering
    \includegraphics[width=0.9\linewidth]{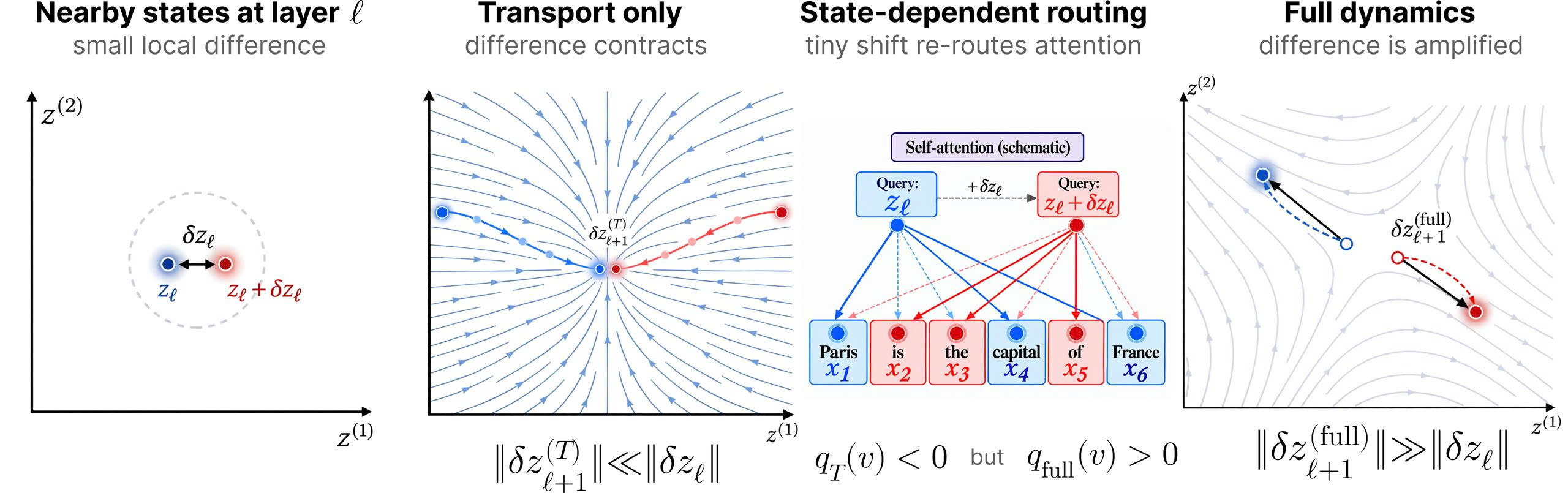}
    \caption{
    Conceptual illustration of gating-induced regime switching.
    Two nearby representations initially differ by a small
    perturbation $\delta z_\ell$. Under fixed attention weights,
    value transport contracts this difference. However, a small
    change in the query can redistribute attention among token
    representations, introducing an additional gating contribution.
    When this contribution is sufficiently strong, the full
    attention computation amplifies the same representation
    difference rather than suppressing it. The illustration
    motivates our analysis of how attention routing can reverse
    transport contraction.
    }
    \label{fig:gating_regime_switch}
\end{figure}

Our analysis shows that changing attention weights does not always produce a change in the output. If attended tokens have nearly identical representations, shifting attention between them has little effect. We characterize this limitation geometrically, showing that differences among attended representations restrict which perturbation directions can affect the output through query gating. Even when attended tokens differ, some perturbations produce no local query-gating response. More importantly, gating can either reinforce or counteract value transport. Consider a representation difference that becomes smaller under value transport alone. By changing which tokens receive attention, gating can oppose this contraction. If the gating response is sufficiently strong, the full attention computation makes the same difference grow instead of shrink. We call this reversal a \emph{\textbf{gating-induced regime switch}}. Our theory establishes a spectral condition for when switching is possible, identifies directions that exhibit this reversal, and characterizes the largest subspace in which every nonzero direction switches.

Our theory also shows that not all regime switches are equally significant. In some directions, value transport only weakly suppresses a representation difference, making it relatively easy for gating to amplify it. In other directions, gating can still amplify a difference even when value transport strongly suppresses it. To distinguish these cases, we introduce two complementary measures: \textit{relative switching strength}, which compares the gating effect against the suppression caused by value transport, and \textit{absolute reversal severity}, which measures how strongly transport can suppress a difference that gating can still amplify. Our theory characterizes both measures, establishes conditions under which switching cannot occur, and examines the reliability of switching directions when the underlying operators are estimated with error. We validate these predictions on pretrained Transformer models by applying small perturbations to hidden representations and comparing the observed changes with our theoretical predictions. The experiments show that high relative switching strength does not necessarily indicate that gating overcomes strong suppression. Moreover, unguided random search often fails to identify directions that are strongly suppressed by transport but amplified by full attention, whereas spectral analysis identifies such directions directly and can be more efficient when many directions must be examined. Our contributions are summarized as follows:
\begin{itemize}
    \item We explain why the same Transformer attention computation can amplify some representation differences while suppressing others. To understand this behavior, we decompose self-attention into three mechanisms: value transport, query gating, and key gating. We then show that query gating is constrained by the geometry of the attended representations: it can only produce changes through directions in which those representations differ.

    \item We develop a spectral theory to explain when attention gating can reverse the contraction caused by value transport. Our results provide a precise condition for this reversal, identify the directions where gating is strongest relative to transport contraction, and determine how severe a contraction can still be reversed. We also characterize the largest subspace in which every nonzero direction exhibits such a reversal.

    \item We validate our theoretical predictions across different Transformer models. Through finite-forward experiments, we verify the predicted geometric and gating mechanisms and examine how switching behavior varies across architectures and layers. We further show that spectral analysis can identify directions where gating reverses particularly severe contraction, even when these directions are difficult to discover through unguided perturbation search.
\end{itemize}

\section{Related Work}
\label{sec:related_work}

\paragraph{Representation collapse and dynamical analyses.}
Transformer representation dynamics have been studied through rank collapse, signal propagation, and token evolution~\cite{dong2021attention,noci2022signal,kedia2024transformers}. Pham et al.~\cite{pham2025dynamical} analyze convergence and divergence of token trajectories, including positional encoding; Tomihari and Karakida~\cite{tomihari2025recurrent} study recurrent attention through Jacobians and Lyapunov exponents. Other work connects query--key spectra to attention localization and avoidance of rank or entropy collapse~\cite{pmlrv235bao24b}. These studies address token configurations and stability. We instead examine how a specific representation perturbation evolves at a fixed operating point, distinguishing directions suppressed by fixed-routing transport from those whose suppression is reversed when attention routing changes.

\paragraph{Sensitivity, Jacobians, and attention derivatives.}
Previous studies have examined how sensitive Transformers are to changes in their inputs and hidden representations. Kim et al.~\cite{kim2021lipschitz} study Lipschitz properties of self-attention, while Aubry et al.~\cite{aubry2025transformer} and Saratchandran and Lucey~\cite{saratchandran2025spectral} analyze attention Jacobians and their spectral properties. Ormaniec et al.~\cite{ormaniec2025does} further investigate the Hessian of the attention loss and its dependence on query, key, and value projections. Rather than simply measuring sensitivity, we investigate which attention mechanisms cause a representation difference to be suppressed or amplified. Specifically, we compare what happens when attention weights remain fixed with what happens when query gating, key gating, or both are allowed to change them. Although these mechanisms can be separated mathematically, their effects on representation growth interact and cannot generally be added independently. This comparison allows us to determine when changes in attention routing reverse the suppression caused by value transport.

\paragraph{Mixture interpretations and state-dependent routing.}
Prior work has explored mixture-based interpretations of attention, including softmax and sigmoid attention~\cite{yan2025sigmoid}, hierarchical mixtures for gated attention~\cite{nguyen2026statistical}, and Mixture-of-Head architectures~\cite{jin2025moh}. We view standard self-attention through a \emph{Mixture-of-Experts (MoE)} perspective, where each attended token provides a value representation and the attention weights act as input-dependent routing coefficients. This viewpoint allows us to distinguish two mechanisms: transporting information under fixed attention weights and changing how information is selected through query and key gating. By analyzing how these mechanisms affect the evolution of representation differences, we show that query gating is constrained by the variation among attended representations. More importantly, we characterize how changes in attention routing interact with value transport, explaining when a representation difference suppressed by transport can become amplified under full attention.

\paragraph{Positioning: spectral certification of regime switching.}
Our objective is to determine when state-dependent routing turns a direction suppressed by transport into one amplified by full attention. On a transport-contractive subspace, comparing transport contraction with the gating-induced growth correction yields a generalized spectral criterion for whether switching occurs and which direction has the strongest relative reversal. We also characterize absolute reversal severity, which distinguishes overcoming strong suppression from reversing only weak suppression, the maximal dimension of an all-switching subspace, and stability under operator-estimation error. The eigenvalue and perturbation tools are classical; our contribution is their attention-specific, mechanistically interpretable formulation, which can identify severe reversals without unguided direction-by-direction probing.

\section{Problem Setting}
\label{sec:problem_setting}

Consider a Transformer block at layer $\ell$ and a representation
$z_\ell\in\mathbb{R}^d$ at the operating point of interest.
We suppress the surrounding sequence context and focus on the local
representation map
$T_\ell:\mathbb{R}^d\rightarrow\mathbb{R}^d$, so that
$z_{\ell+1}=T_\ell(z_\ell)$.
We study how a small difference
$\delta z_\ell\in\mathbb{R}^d$ between two nearby representations,
$z_\ell$ and $z_\ell+\delta z_\ell$, evolves through the block.
Define
\begin{equation}
\delta z_{\ell+1}
:=
T_\ell(z_\ell+\delta z_\ell)-T_\ell(z_\ell).
\end{equation}
For sufficiently small $\delta z_\ell$,
\begin{equation}
\delta z_{\ell+1}
=
J_\ell\delta z_\ell
+
\mathcal{O}(\|\delta z_\ell\|^2),
\qquad
J_\ell
:=
\frac{\partial T_\ell}{\partial z_\ell},
\label{eq:local_dynamics}
\end{equation}
where $J_\ell$ is the local Jacobian of the block. To determine whether a representation difference grows or shrinks, we
must first specify how its magnitude is measured. We therefore introduce
a symmetric positive-definite matrix
$M=M^\top\succ0$ and define the representation metric $\|v\|_M^2
:=
v^\top Mv.$ The matrix $M$ defines the geometry in which representation differences
are measured: $M=I$ recovers the usual Euclidean norm, while a general
$M$ allows different representation directions to be weighted
differently. Under the local transformation $v\mapsto J_\ell v$, the squared
$M$-magnitude after the block is
\begin{equation}
\|J_\ell v\|_M^2
=
(J_\ell v)^\top M(J_\ell v)
=
v^\top J_\ell^\top M J_\ell v.
\end{equation}
Thus $J_\ell^\top M J_\ell$ is the metric induced on the incoming
perturbation by the local Transformer dynamics. Comparing the outgoing
and incoming magnitudes gives
\begin{equation}
\|J_\ell v\|_M^2-\|v\|_M^2
=
v^\top\mathcal{G}_\ell^M v,
\qquad
\mathcal{G}_\ell^M
:=
J_\ell^\top M J_\ell-M.
\label{eq:growth_operator}
\end{equation}
Accordingly,
$v^\top\mathcal{G}_\ell^M v>0$ indicates local amplification,
whereas
$v^\top\mathcal{G}_\ell^M v<0$ indicates local suppression. Importantly, the sign of
$v^\top\mathcal{G}_\ell^M v$
may depend on the direction $v$.
A single Transformer block can therefore suppress some representation
differences while amplifying others.
We refer to this direction-dependent behavior as
\emph{\textbf{anisotropic representation dynamics}}. Equation~\eqref{eq:growth_operator} characterizes \emph{which}
representation differences grow or shrink, but does not yet explain
\emph{why}. Our goal is to connect this directional growth to the
internal computation of self-attention. We separate the local dynamics
into \emph{transport}, which propagates changes in the information
carried by attended representations, and state-dependent
\emph{gating}, which arises when the perturbation changes the attention
allocation itself. This leads to the central question of the paper:

\begin{quote}
\emph{When and why can attention routing turn a representation
difference that would shrink dramatically under value transport
into one that grows?}
\end{quote}

We call such a transition a
\emph{gating-induced regime switch}, and seek to characterize both
when such switching directions exist and which directions exhibit the
strongest reversal.

\section{Mechanism of State-Dependent Attention}
\label{sec:attention_mechanism}

The growth operator in Section~\ref{sec:problem_setting} tells us whether a representation difference grows or shrinks, but not which part of self-attention causes this behavior. To understand why, we decompose the attention Jacobian into three pathways: \textit{value transport}, which propagates changes in the information carried by attended tokens; \textit{query gating}, which changes attention allocation through the query; and \textit{key gating}, which changes attention allocation through the keys. This decomposition allows us to understand how each pathway contributes to representation growth.

\subsection{Transport and Gating in Self-Attention}
\label{subsec:transport_gating}

Consider one self-attention head in a Transformer layer. Let \(x_i\in\mathbb{R}^d\) denote the representation of token \(i\), \(r\) the query-token index, and \(m\) the index of the token whose representation is perturbed. For query token \(r\), the attention score assigned to token \(i\) is \(s_{ri}=x_r^\top W_{QK}x_i\), where \(W_{QK}:=W_QW_K^\top/\sqrt{d_k}\) combines the query and key projections, and \(d_k\) is the query/key dimension. The corresponding attention weight is \(\alpha_{ri}=\exp(s_{ri})/\sum_j\exp(s_{rj})\). Let \(W_{VO}\) denote the combined value and output transformation. The attention output for query token \(r\) is
\begin{equation}
F_r
=
\sum_i \alpha_{ri}W_{VO}x_i .
\label{eq:attention_output}
\end{equation}
Equation~\eqref{eq:attention_output} shows that a perturbation can change the attention output in two ways. First, it can change the information carried by the attended tokens through the value term \(W_{VO}x_i\). Second, it can change the attention weights \(\alpha_{ri}\), causing the model to place more or less weight on different tokens. We call these two effects \emph{value transport} and \emph{gating}, respectively. While transport changes the information being passed forward, gating changes how that information is selected.

To understand how each pathway affects the output, lets consider a small change in the representation of token \(m\). The Jacobian \(J_{rm}:=\partial F_r/\partial x_m\) captures how this change affects the attention output of query token \(r\). For gating, the effect also depends on how different the attended representations are from one another. To capture these differences, we use the attention-weighted mean \(\mu_r:=\sum_i\alpha_{ri}x_i\) and covariance \(\Sigma_r:=\sum_i\alpha_{ri}(x_i-\mu_r)(x_i-\mu_r)^\top\). The mean describes the average information selected by attention, while the covariance measures how much the attended representations differ around this average. Differentiating Eq.~\eqref{eq:attention_output} then gives

\begin{equation}
\begin{aligned}
J_{rm}
={}&
\underbrace{\alpha_{rm}W_{VO}}_{\text{value transport}}
+
\underbrace{\mathbf{1}_{\{m=r\}}W_{VO}\Sigma_rW_{QK}^{\top}}_{\text{query gating}} +
\underbrace{\alpha_{rm}W_{VO}(x_m-\mu_r)(W_{QK}^{\top}x_r)^\top}_{\text{key gating}} .
\end{aligned}
\label{eq:attention_jacobian_decomposition}
\end{equation}

Here, \(\mathbf{1}_{\{m=r\}}\) equals one when the perturbed token is also the query token and zero otherwise. The full derivation is provided in Appendix~\ref{app:attention_jacobian}. Equation~\eqref{eq:attention_jacobian_decomposition} separates the attention Jacobian into three pathways. The \emph{value-transport pathway} directly propagates the perturbation through the value transformation. The \emph{query-gating pathway} captures how changing the query alters its attention to the available tokens, and contributes only when \(m=r\). The \emph{key-gating pathway} captures how changing token \(m\) affects the attention it receives as a key. When \(m=r\), both gating pathways can contribute to the same output change. This decomposition allows us to distinguish changes caused by the information being transported from those caused by changes in attention allocation. In the next subsection, we show how the differences among attended tokens limit which perturbation directions can affect the output through query gating.

\subsection{Geometry of Query Gating}
\label{subsec:query_gating_geometry}

We now investigate how query gating changes a representation difference and which perturbation directions can influence the attention output through this mechanism. From Eq.~\eqref{eq:attention_jacobian_decomposition}, we isolate the query-gating component
\begin{equation}
J_r^{(q)}=W_{VO}\Sigma_rW_{QK}^{\top},
\label{eq:query_gating_geometry}
\end{equation}
which contributes to the full Jacobian \(J_{rm}\) only when the perturbed token is the query itself, i.e., \(m=r\). This decomposition reveal that query gating has three components. The matrix \(W_{QK}^{\top}\) determines how a query perturbation changes attention scores, \(\Sigma_r\) captures variation among the attended representations, and \(W_{VO}\) maps the resulting change into the output space. We are particularly interested in the role of \(\Sigma_r\), since it connects changes in attention routing to the representations being attended to. To characterize this connection, we define the \emph{attended variation subspace} as \(\mathcal U_r:=\operatorname{span}\{x_i-\mu_r:\alpha_{ri}>0\}\), where \(\mu_r\) is the attention-weighted mean of the attended representations. This subspace describes the directions along which those representations differ from their mean. The following proposition characterizes how this geometry constrains query gating.

\begin{figure*}[htp]
    \centering
    \includegraphics[width=0.9\textwidth]{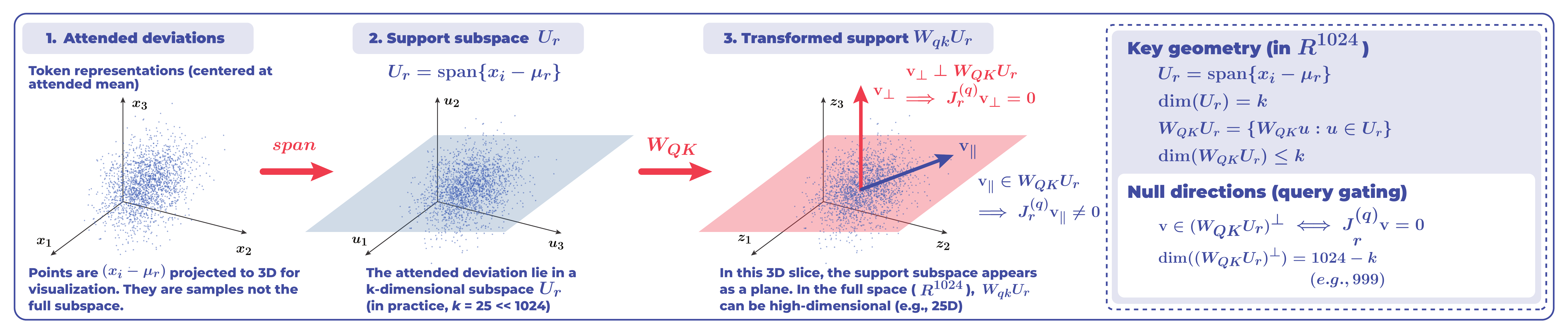}
    \caption{
    \textbf{Geometric support and null directions of query gating.}
    Centered attended representations span a subspace
    \(\mathcal U_r\), whose dimension may be substantially smaller
    than the ambient representation dimension.
    The transformed subspace \(W_{QK}\mathcal U_r\)
    determines which query perturbation directions can influence
    the attention output through query gating.
    Perturbations orthogonal to this subspace produce no
    first-order query-gating change.
    Consequently, query gating is insensitive to at least
    \(d-\dim(\mathcal U_r)\) independent input directions.
    The three-dimensional projections provide a schematic
    visualization of the higher-dimensional geometry.
    }
    \label{fig:query_gating_support_geometry}
\end{figure*}

\begin{proposition}[Geometric Support of Query Gating]
\label{prop:query_gating_support}
Let \(\mathcal U_r\) be the attended variation subspace and \(J_r^{(q)}=W_{VO}\Sigma_rW_{QK}^{\top}\) the query-gating Jacobian. Then
\begin{equation}
\begin{aligned}
\operatorname{Range}(J_r^{(q)})
&\subseteq W_{VO}\mathcal U_r, \qquad\operatorname{rank}(J_r^{(q)})
\leq\dim(\mathcal U_r),\\
W_{QK}^{\top}v\perp\mathcal U_r
&\Longrightarrow J_r^{(q)}v=0.
\end{aligned}
\label{eq:query_gating_support}
\end{equation}
\end{proposition}

\begin{insightbox}
\textbf{Insight.}
Proposition~\ref{prop:query_gating_support} shows that query gating is constrained by the geometry of attended representations in two ways. First, the range restriction \(\operatorname{Range}(J_r^{(q)})\subseteq W_{VO}\mathcal U_r\) limits the possible output changes, while the rank bound \(\operatorname{rank}(J_r^{(q)})\leq\dim(\mathcal U_r)\) limits their number of independent directions. For example, if several attended representations differ along only one direction, query gating can produce changes along at most one independent output direction, even in a high-dimensional representation space. If all attended representations are identical, then \(\mathcal U_r=\{0\}\), implying \(J_r^{(q)}=0\): there is no difference for attention to redistribute. Second, the condition \(W_{QK}^{\top}v\perp\mathcal U_r\Rightarrow J_r^{(q)}v=0\) shows that not every query perturbation can produce an output change through gating. Therefore, changing attention weights does not always change the attention output. Query gating can only exploit differences among attended representations, and its effect depends on which perturbation directions interact with those differences.
\end{insightbox}

Figure~\ref{fig:query_gating_support_geometry}
illustrates how the geometry of attended representations
restricts the perturbation directions that can influence
the attention output through query gating.
The complete proof, pairwise covariance interpretation, and additional geometric bounds are provided in Appendix~\ref{subsec:app_gating_geometry}. While the proposition identifies which directions query gating can affect, it does not determine whether these changes suppress or amplify a representation difference. We next examine how gating interacts with value transport and when this interaction leads to regime switching.

\section{Theory of Gating-Induced Regime Switching}
\label{sec:regime_switching}

The previous section established how attention routing changes representations and how the geometry of attended tokens constrains query gating. However, changing attention allocation does not necessarily make a representation difference grow. The effect depends on how gating combines with value transport. We now study when this interaction can reverse transport contraction, turning a representation difference that would otherwise shrink into one that grows.

\subsection{From Transport Contraction to Regime Switching}
\label{subsec:switch_definition}

\begin{figure*}[htp]
    \centering
    \includegraphics[width=0.9\textwidth]{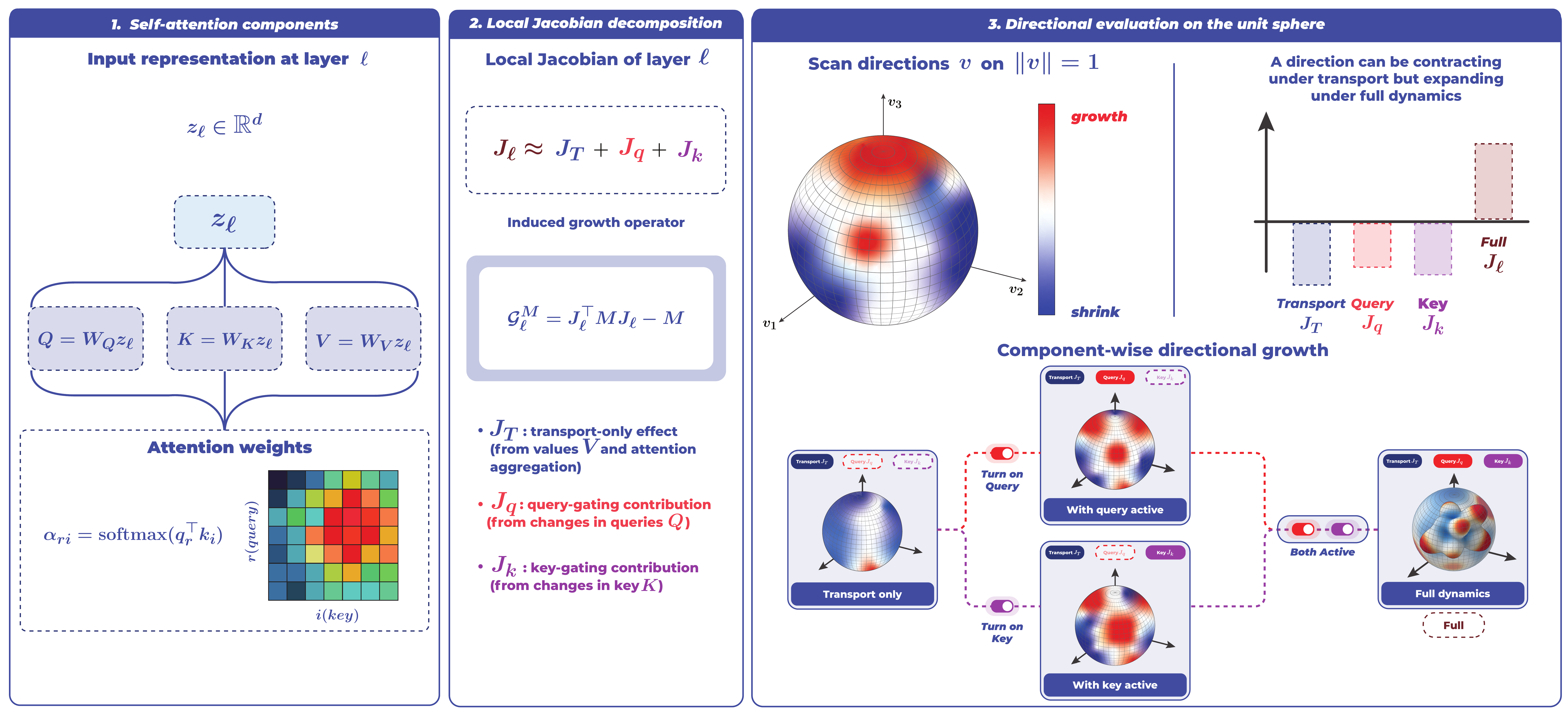}
    \caption{
    \textbf{From self-attention decomposition to gating-induced regime switching.}
    Self-attention induces three local mechanisms: value transport,
    query gating, and key gating, represented by their corresponding
    Jacobian components. These components jointly determine directional
    representation growth through the induced growth operator
    \(G^M=J^\top MJ-M\).
    The three-dimensional illustration shows how different perturbation
    directions can exhibit contraction or expansion.
    Interactions between value transport and attention gating can
    reverse transport contraction, causing a representation difference
    to grow under full attention.
    }
    \label{fig:jacobian_growth_switching}
\end{figure*}

We now investigate whether attention routing can turn a representation difference suppressed by value transport into one amplified by full attention. Consider a small perturbation to the token representations. We first study how this perturbation changes the attention output when the attention weights are held fixed, considering only value transport. We then include the additional changes caused by query and key gating. Both cases use the same initial representations and perturbation, allowing us to distinguish the effect of value transport from the additional contribution of attention routing. This comparison helps us understand when gating reinforces transport contraction and when it reverses that contraction, causing a representation difference to grow instead of shrink.

Recall the attention Jacobian decomposition in Eq.~\eqref{eq:attention_jacobian_decomposition}. Let \(J_T\), \(J_Q\), and \(J_K\) denote the value-transport, query-gating, and key-gating Jacobians, respectively, with token indices omitted for simplicity. We define the combined gating Jacobian as \(J_G:=J_Q+J_K\), so that the full attention Jacobian is \(J_{\mathrm{full}}=J_T+J_G\). When analyzing a residual attention sublayer, the shared residual contribution is included consistently in the transport-only and full Jacobians. Using the representation metric \(M\succ0\) introduced in Section~\ref{sec:problem_setting}, we define the corresponding growth operators as
\begin{equation}
G_T^M:=J_T^\top M J_T-M,\qquad
G_{\mathrm{full}}^M:=J_{\mathrm{full}}^\top M J_{\mathrm{full}}-M.
\label{eq:transport_full_growth}
\end{equation}
For a perturbation direction \(v\), the quadratic forms \(v^\top G_T^M v\) and \(v^\top G_{\mathrm{full}}^M v\) measure the changes in its squared magnitude under transport-only and full attention, respectively. We are interested in directions for which these two quantities have opposite signs.

\begin{definition}[Gating-Induced Regime Switch]
\label{def:gating_regime_switch}
A nonzero perturbation direction \(v\) undergoes a \emph{gating-induced regime switch} if
\begin{equation}
v^\top G_T^M v<0,
\qquad
v^\top G_{\mathrm{full}}^M v>0.
\label{eq:regime_switch_definition}
\end{equation}
\end{definition}

Definition~\ref{def:gating_regime_switch} describes a change from suppression to amplification along the same direction. Under fixed attention weights, value transport suppresses the representation difference. When query and key gating are included, full attention amplifies it. To determine when this reversal occurs, we compare the two growth operators and isolate the additional contribution caused by gating.

\begin{proposition}[Exact Gating-Induced Growth Correction]
\label{prop:gating_growth_correction}
Let \(\Delta G_G^M:=G_{\mathrm{full}}^M-G_T^M\) denote the gating-induced growth correction. Then $\Delta G_G^M
=
J_T^\top M J_G
+
J_G^\top M J_T
+
J_G^\top M J_G.$ For any nonzero perturbation direction \(v\), the corresponding change in directional growth is
\begin{equation}
v^\top\Delta G_G^M v
=
2\langle J_Tv,J_Gv\rangle_M
+
\|J_Gv\|_M^2.
\label{eq:directional_gating_correction}
\end{equation}
Consequently, if \(v^\top G_T^M v<0\), a gating-induced regime switch occurs if and only if $v^\top\Delta G_G^M v
>
-v^\top G_T^M v.$ (Proof in Appendix~\ref{subsec:app_transport_full_growth})
\end{proposition}

\begin{insightbox}
\textbf{Insight.}
Proposition~\ref{prop:gating_growth_correction} shows that the effect of attention gating depends on how it combines with value transport. In Eq.~\eqref{eq:directional_gating_correction}, the interaction term \(2\langle J_Tv,J_Gv\rangle_M\) can increase or decrease representation growth, depending on the alignment between transport and gating. Furthermore, since \(J_G=J_Q+J_K\), the term \(\|J_Gv\|_M^2\) contains an additional interaction between query and key gating, given by \(2\langle J_Qv,J_Kv\rangle_M\). Therefore, the combined effect of query and key gating on representation growth is not simply the sum of their individual effects. In particular, query or key gating may each fail to reverse transport contraction, yet succeed when acting together. 
Such a reversal occurs when the combined gating effect is strong enough to overcome the contraction caused by value transport ($v^\top\Delta G_G^M v
>
-v^\top G_T^M v$).
\end{insightbox}

Proposition~\ref{prop:gating_growth_correction} establishes when a given perturbation direction undergoes a regime switch. However, identifying such directions individually can be difficult in a high-dimensional representation space. We next develop a spectral criterion that determines whether switching is possible and identifies switching directions without testing perturbations one by one.

\subsection{Spectral Certification of Gating-Induced Regime Switching}
\label{subsec:spectral_switching}

Proposition~\ref{prop:gating_growth_correction} provides an exact condition for determining whether a given perturbation direction undergoes a regime switch. However, it does not tell us whether such a direction exists without first selecting and testing individual perturbations. Since representation spaces can be high-dimensional, searching over possible directions is difficult. We therefore seek a spectral condition that determines whether gating can reverse transport contraction along at least one direction and, if so, identifies such a direction directly. We focus on a \emph{transport-contractive subspace} \(\mathcal S\subseteq\mathbb R^d\), in which value transport suppresses every nonzero perturbation direction. Let \(P\in\mathbb R^{d\times k}\) have orthonormal columns spanning \(\mathcal S\), so that every \(v\in\mathcal S\) can be expressed as \(v=Pu\), where \(u\in\mathbb R^k\). By definition, transport contraction throughout this subspace requires \(P^\top G_T^M P\prec0\). Such a subspace can be constructed from the eigenvectors of \(G_T^M\) associated with its negative eigenvalues.

To determine whether gating can reverse this contraction, we compare the transport-induced suppression with the additional growth introduced by gating. We define the corresponding operators restricted to \(\mathcal S\) as
\begin{equation}
C:=-P^\top G_T^M P\succ0,\qquad
R:=P^\top\Delta G_G^M P.
\label{eq:restricted_switching_operators}
\end{equation}
Here, \(C\) measures how strongly transport suppresses perturbations, while \(R\) measures how much the inclusion of query and key gating changes their growth. For any direction \(v=Pu\in\mathcal S\), we have
\begin{equation}
v^\top G_T^M v=-u^\top Cu,\qquad
v^\top G_{\mathrm{full}}^M v=u^\top(R-C)u.
\label{eq:restricted_directional_growth}
\end{equation}
Since \(C\succ0\), every nonzero direction in \(\mathcal S\) is contracted by transport. A regime switch therefore exists precisely when the growth correction introduced by gating exceeds this contraction along at least one direction, i.e., \(u^\top Ru>u^\top Cu\) for some \(u\neq0\). The following theorem establishes a necessary and sufficient spectral condition for this reversal.

\begin{theorem}[Spectral Criterion for Gating-Induced Regime Switching]
\label{thm:spectral_regime_switching}
For the \emph{transport switching operator}
\(W:=C^{-1/2}RC^{-1/2}\), there exists a nonzero direction
\(v\in\mathcal S\) undergoing a gating-induced regime switch
if and only if \(\lambda_{\max}(W)>1\).
Moreover, when this condition holds, a switching direction is
\(v_{\mathrm{sw}}=PC^{-1/2}w_{\max}\), where \(w_{\max}\) is a unit
eigenvector of \(W\) associated with \(\lambda_{\max}(W)\).
\end{theorem}

The theorem follows by expressing the switching condition as \(u^\top Ru/(u^\top Cu)>1\), which is well-defined because \(C\succ0\). The maximum of this ratio over all nonzero directions equals \(\lambda_{\max}(W)\), yielding the necessary and sufficient condition. The complete proof is provided in Appendix~\ref{subsec:app_spectral_switch}.

\begin{insightbox}
\textbf{Insight.}
Theorem~\ref{thm:spectral_regime_switching} addresses a fundamental limitation of testing perturbation directions individually: even if many directions are tested and no regime switch is found, this does not prove that switching is impossible. In contrast, the spectral criterion provides a necessary and sufficient condition for the existence of switching within the transport-contractive subspace \(\mathcal S\). The operator \(W=C^{-1/2}RC^{-1/2}\) compares the growth introduced by gating with the contraction caused by transport across this subspace. If \(\lambda_{\max}(W)\leq1\), no switching direction exists in \(\mathcal S\). If \(\lambda_{\max}(W)>1\), switching is guaranteed to occur, and the leading eigenvector allows us to construct a switching direction directly. Thus, spectral analysis can certify both the existence and absence of switching without relying on direction-by-direction search. However, the identified direction does not necessarily correspond to the strongest transport contraction that gating can reverse, motivating our subsequent analysis of reversal severity.
\end{insightbox}

We next distinguish the relative strength of gating-induced switching from the absolute severity of transport contraction that gating can overcome.

\subsection{Relative Switching Strength and Certified Reversal Severity}
\label{subsec:switching_strength}

Theorem~\ref{thm:spectral_regime_switching} establishes when gating can reverse transport contraction and identifies a switching direction when one exists. However, this direction does not necessarily correspond to the strongest contraction that gating can overcome. For instance, a direction may show a large gating effect relative to transport simply because transport barely suppresses it. In contrast, another direction may remain expanding under full attention even after transport removes most of its original magnitude. To distinguish these cases, we introduce two different measures: \emph{relative switching strength}, which compares the gating-induced growth correction against transport contraction, and \emph{absolute reversal severity}, which measures how strongly transport suppresses a difference that gating can still amplify.

\paragraph{Relative switching strength.}
For a nonzero direction \(v\in\mathcal S\), the gating-induced growth correction is \(v^\top\Delta G_G^M v\), while the magnitude of transport contraction is \(-v^\top G_T^M v>0\). Their ratio defines the relative switching strength:
\begin{equation}
\Gamma(v)
:=
\frac{v^\top\Delta G_G^M v}
{-v^\top G_T^M v}
=
\frac{u^\top Ru}{u^\top Cu},
\qquad v=Pu.
\label{eq:relative_switching_strength}
\end{equation}
This ratio tells us how much gating changes representation growth compared with the contraction it must overcome. In particular, \(\Gamma(v)>1\) means that the gating-induced correction exceeds transport contraction, resulting in a regime switch. The spectral analysis in Theorem~\ref{thm:spectral_regime_switching} also allows us to identify the direction that maximizes this ratio.

\begin{corollary}[Maximum Relative Switching Strength]
\label{cor:max_relative_switching}
Let \(W=C^{-1/2}RC^{-1/2}\), and let \(w_{\max}\) be a unit eigenvector associated with \(\lambda_{\max}(W)\). Then
\begin{equation}
\max_{v\in\mathcal S\setminus\{0\}}\Gamma(v)
=
\lambda_{\max}(W),
\qquad
v_{\mathrm{rel}}=PC^{-1/2}w_{\max}.
\label{eq:maximum_relative_switching}
\end{equation}
The direction \(v_{\mathrm{rel}}\), up to nonzero scaling, maximizes relative switching strength and undergoes a regime switch whenever \(\lambda_{\max}(W)>1\).
\end{corollary}

Corollary~\ref{cor:max_relative_switching} follows from the generalized Rayleigh quotient in Eq.~\eqref{eq:relative_switching_strength}, with the complete proof provided in Appendix~\ref{subsec:app_switch_strength}. Although the corollary identifies the direction with the largest relative switching strength, a large value of \(\Gamma(v)\) does not necessarily indicate that gating has reversed strong contraction. When transport barely reduces a representation difference, even a modest gating-induced correction can produce a large ratio. Therefore, maximizing relative switching strength does not directly answer how much contraction gating can overcome.

\paragraph{Absolute reversal severity.}
To measure the contraction itself, we introduce the \emph{absolute reversal severity}, defined for any nonzero transport-contractive direction \(v\in\mathcal S\) as
\begin{equation}
\tau(v)
:=
\frac{-v^\top G_T^M v}{v^\top Mv}
=
1-\frac{\|J_Tv\|_M^2}{\|v\|_M^2}.
\label{eq:absolute_reversal_severity}
\end{equation}
Unlike \(\Gamma(v)\), which compares gating against transport, \(\tau(v)\) measures the fraction of incoming squared representation magnitude removed by transport. Since every nonzero direction in \(\mathcal S\) is transport-contractive, \(0<\tau(v)\leq1\). For example, \(\tau(v)=0.8\) means that transport removes \(80\%\) of the incoming squared magnitude. If full attention nevertheless amplifies this direction, gating has reversed a substantial contraction.

Our goal is to determine the strongest transport contraction that gating can still reverse. Let \(\tau_{\max}\) denote the supremum of \(\tau(v)\) over all switching directions in \(\mathcal S\), with \(\tau_{\max}=0\) if no switching direction exists. Finding this quantity requires a direction that is both strongly suppressed by transport and amplified by full attention. Rather than searching over individual directions, we develop a spectral certificate that determines whether switching is possible among directions whose contraction exceeds a prescribed severity threshold. To construct this certificate, we first need to identify directions that transport suppresses sufficiently strongly. Recall that every \(v\in\mathcal S\) can be expressed as \(v=Pu\). Let \(B:=P^\top MP\succ0\), which measures the original squared representation magnitude in the coordinates of \(\mathcal S\). Since \(C=-P^\top G_T^M P\) measures the magnitude removed by transport, the severity can be written as \(\tau(Pu)=u^\top Cu/(u^\top Bu)\). Thus, \(C\) measures how much magnitude is removed, while \(B\) measures how much was present initially.

We identify directions with different contraction severities by solving the generalized eigenvalue problem \(Cq_i=\eta_iBq_i\), with \(q_i^\top Bq_j=\delta_{ij}\), and order the eigenvalues as \(\eta_1\geq\cdots\geq\eta_k>0\). Each eigenvector \(q_i\) specifies a direction \(Pq_i\) in the original representation space, while its eigenvalue \(\eta_i=\tau(Pq_i)\) gives the fraction of squared magnitude removed by transport along that direction. Larger eigenvalues therefore identify directions that transport suppresses more strongly. Given a desired severity threshold \(\tau_0\in(0,1]\), we retain the eigenvectors whose contraction severity is at least \(\tau_0\) and construct
\begin{equation}
\mathcal S_{\tau_0}
:=
P\,\operatorname{span}
\{q_i:\eta_i\geq\tau_0\}.
\label{eq:severity_restricted_subspace}
\end{equation}
This construction guarantees that every nonzero direction in \(\mathcal S_{\tau_0}\), including linear combinations of the selected eigenvectors, satisfies \(\tau(v)\geq\tau_0\). For example, choosing \(\tau_0=0.8\) gives a subspace in which transport removes at least \(80\%\) of the incoming squared magnitude along every direction. We can focus on strongly suppressed differences while excluding directions whose contraction is weaker than the chosen threshold. Having identified directions with strong transport contraction, the remaining question is whether gating can reverse this contraction. We answer this by applying the spectral criterion from Theorem~\ref{thm:spectral_regime_switching} to \(\mathcal S_{\tau_0}\). Let \(P_{\tau_0}\) have orthonormal columns spanning this subspace, assuming \(\mathcal S_{\tau_0}\neq\{0\}\). We define the restricted transport-contraction and gating-growth operators as
\begin{equation}
\begin{aligned}
C_{\tau_0}
&:=-P_{\tau_0}^\top G_T^M P_{\tau_0}\succ0, \qquad R_{\tau_0} :=P_{\tau_0}^\top\Delta G_G^M P_{\tau_0}.
\end{aligned}
\label{eq:severity_restricted_operators}
\end{equation}

Here, \(C_{\tau_0}\) measures how strongly transport suppresses directions in the selected subspace, while \(R_{\tau_0}\) measures the additional growth introduced by gating. Following Section~\ref{subsec:spectral_switching}, we construct the restricted switching operator \(W_{\tau_0}:=C_{\tau_0}^{-1/2}R_{\tau_0}C_{\tau_0}^{-1/2}\), which compares the gating-induced growth correction against transport contraction within this strongly contracting subspace. The following theorem establishes whether gating can reverse contraction of at least the prescribed severity.

\begin{theorem}[Certified Absolute Reversal Severity]
\label{thm:absolute_reversal_severity}
Fix \(\tau_0\in(0,1]\) and suppose \(\mathcal S_{\tau_0}\neq\{0\}\). Every nonzero direction \(v\in\mathcal S_{\tau_0}\) satisfies \(\tau(v)\geq\tau_0\). Moreover, there exists a gating-induced regime-switch direction in \(\mathcal S_{\tau_0}\) if and only if $\lambda_{\max}(W_{\tau_0})>1.$ Consequently, whenever this condition holds, gating can reverse transport contraction of severity at least \(\tau_0\), implying \(\tau_{\max}\geq\tau_0\).
\end{theorem}

The theorem follows by combining two results: the construction of \(\mathcal S_{\tau_0}\) guarantees the required contraction severity for every nonzero direction in this subspace, while the spectral criterion determines whether at least one of these directions becomes expanding under full attention. The complete proof is provided in Appendix~\ref{subsec:app_switch_strength}. By increasing the threshold \(\tau_0\), we can determine whether gating can still reverse progressively stronger transport contraction. This motivates the strongest severity threshold certified by our method:
\begin{equation}
\begin{aligned}
\tau_{\mathrm{cert}}
:=
\sup\bigl\{
\tau_0\in(0,1]:\,
&\mathcal S_{\tau_0}\neq\{0\},\lambda_{\max}(W_{\tau_0})>1
\bigr\},
\quad
\tau_{\max}\geq\tau_{\mathrm{cert}}.
\end{aligned}
\label{eq:certified_reversal_severity}
\end{equation}
We set \(\tau_{\mathrm{cert}}=0\) when no threshold is certified. Importantly, \(\tau_{\mathrm{cert}}\) is a lower bound on the strongest reversal severity, rather than necessarily its exact value. Our construction considers subspaces where every direction satisfies the required severity threshold. However, individual switching directions with equally strong contraction may exist outside these subspaces. Therefore, failure to obtain a certificate at a particular threshold does not rule out the existence of switching directions with that severity elsewhere in \(\mathcal S\).

\begin{insightbox}
\textbf{Insight.}
The strongest regime switch in relative terms need not be the most substantial reversal of transport contraction. In particular, a direction that transport barely suppresses may achieve a much larger relative switching strength than one where gating overcomes substantial contraction. This distinction reveals why the leading eigenvalue of the switching operator alone cannot characterize the severity of a reversal. A more demanding test is whether switching remains possible after directions with weak transport contraction have been excluded. Our severity certificate provides precisely this guarantee: when the spectral condition holds within the restricted subspace, gating can reverse contraction of at least the prescribed severity, even though not every strongly contracted direction must undergo a switch.
\end{insightbox}

\subsection{No-Switch Certification and Effective Switching Dimension}
\label{subsec:switching_dimension}

The preceding results establish when gating can reverse transport contraction and how severe that contraction can be. However, identifying a switching direction does not reveal how broadly this reversal extends across representation space. In particular, switching along several individual directions does not guarantee that their linear combinations also undergo a regime switch. We therefore ask two questions: when can switching be ruled out entirely, and when switching is possible, what is the largest subspace in which every nonzero direction undergoes the contraction-to-expansion reversal? To address these, we return to the transport switching operator \(W=C^{-1/2}RC^{-1/2}\) introduced in Section~\ref{subsec:spectral_switching}. Every nonzero \(v=Pu\in\mathcal S\) is contracted by value transport, while its growth under full attention is \(v^\top G_{\mathrm{full}}^M v=u^\top(R-C)u\). Using the transformed coordinates \(w=C^{1/2}u\), this growth becomes \(w^\top(W-I)w\). Thus, the spectrum of \(W-I\) determines not only whether full attention can reverse contraction, but also the number of independent directions that can form a subspace where switching holds uniformly. The following theorem provides an exact characterization.

\begin{theorem}[No-Switch Criterion and Effective Switching Dimension]
\label{thm:effective_switching_dimension}
No gating-induced regime switch exists in \(\mathcal S\) if and only if
\begin{equation}
\lambda_{\max}(W)\leq 1
\quad\Longleftrightarrow\quad
R\preceq C.
\label{eq:no_switch_condition}
\end{equation}
Moreover, the \emph{effective switching dimension}
\(d_{\mathrm{sw}}:=\#\{i:\lambda_i(W)>1\}\)
is the largest dimension of a linear subspace of \(\mathcal S\)
in which every nonzero direction undergoes a gating-induced regime switch.
In particular, the subspace
\begin{equation}
\mathcal V_{\mathrm{sw}}
:=
PC^{-1/2}
\operatorname{span}\{w_i:\lambda_i(W)>1\}
\label{eq:all_switching_subspace}
\end{equation}
achieves this dimension, where \(w_i\) are orthonormal eigenvectors of \(W\).
Every nonzero \(v\in\mathcal V_{\mathrm{sw}}\) undergoes a gating-induced regime switch.
\end{theorem}

The theorem follows from the quadratic form \(w^\top(W-I)w\). If all eigenvalues of \(W\) are at most one, this quantity is nonpositive in every direction, ruling out switching throughout \(\mathcal S\). Otherwise, the eigenvectors associated with eigenvalues greater than one span a subspace on which full directional growth is strictly positive. Since the number of positive eigenvalues of \(W-I\) equals the maximum dimension of a subspace on which its quadratic form is positive for every nonzero vector, no larger linear subspace can satisfy the switching condition uniformly. The proof is in Appendix~\ref{subsec:app_no_switch}.

The effective switching dimension therefore provides a complementary characterization to the existence and severity results developed earlier. When \(d_{\mathrm{sw}}=0\), no switching direction exists within \(\mathcal S\). When \(d_{\mathrm{sw}}=1\), at least one one-dimensional subspace exhibits switching along every nonzero direction, but no two-dimensional linear subspace has this guarantee. More generally, a larger \(d_{\mathrm{sw}}\) certifies that the contraction-to-expansion reversal persists throughout a higher-dimensional subspace, rather than only along an individually identified direction.

\begin{insightbox}
\textbf{Insight.}
Theorem~\ref{thm:effective_switching_dimension} reveals an important distinction between finding switching directions and identifying a subspace where switching is guaranteed. Even if many individual directions undergo a regime switch, their linear combinations may not. The effective switching dimension \(d_{\mathrm{sw}}\) captures this distinction by identifying the largest number of independent degrees of freedom that can be combined while preserving the reversal for every nonzero combination. For example, \(d_{\mathrm{sw}}=1\) does not mean that only one direction can switch; many switching directions may exist, but no two-dimensional subspace is guaranteed to consist entirely of switching directions. Thus, beyond certifying whether switching exists, the spectrum of \(W\) characterizes the dimensional extent over which this reversal holds uniformly.
\end{insightbox}

\section{Experiments}
\label{sec:experiments}

Experiments follow the progression from geometric constraints on attention gating to the emergence and spectral characterization of regime switching. We examine how attended-representation geometry restricts query-gating responses, and whether query and key gating can reverse contraction induced by value transport. We move beyond individual switching directions to investigate how spectral analysis certifies their existence, dimensionality, and reversal severity. Finally, we examine whether these mechanisms and their structure persist across Transformer architectures and depths. We evaluate these predictions through analytical computations, finite-forward perturbations, and pathway ablations, with experimental details and analyses provided in Appendix~\ref{sec:app_experiments}.
    
\paragraph{Experimental Setup.}
We evaluate pretrained GPT-2 Medium, Gemma-2-2B, Llama-3.2-1B, and
Qwen2.5-0.5B on the WikiText-2 validation set. Detailed mechanistic
and spectral analyses focus on a selected operating point in layer 23
of GPT-2 Medium; cross-architecture analyses examine multiple layers
of all four models. We construct the value-transport, query-gating,
and key-gating Jacobians $J_T$, $J_Q$, and $J_K$ and compare
transport-only, transport plus query gating, transport plus key
gating, and full-attention dynamics. For residual attention
sublayers, the common residual contribution is included consistently
in these comparisons. Unless otherwise specified, we use the
Euclidean metric $M=I$. Analytical Jacobians are checked against
automatic differentiation and finite differences. Experiments
were run on an NVIDIA A100 GPU with 40\,GB of memory.

\begin{figure*}[htp]
    \centering
    \includegraphics[width=\textwidth]{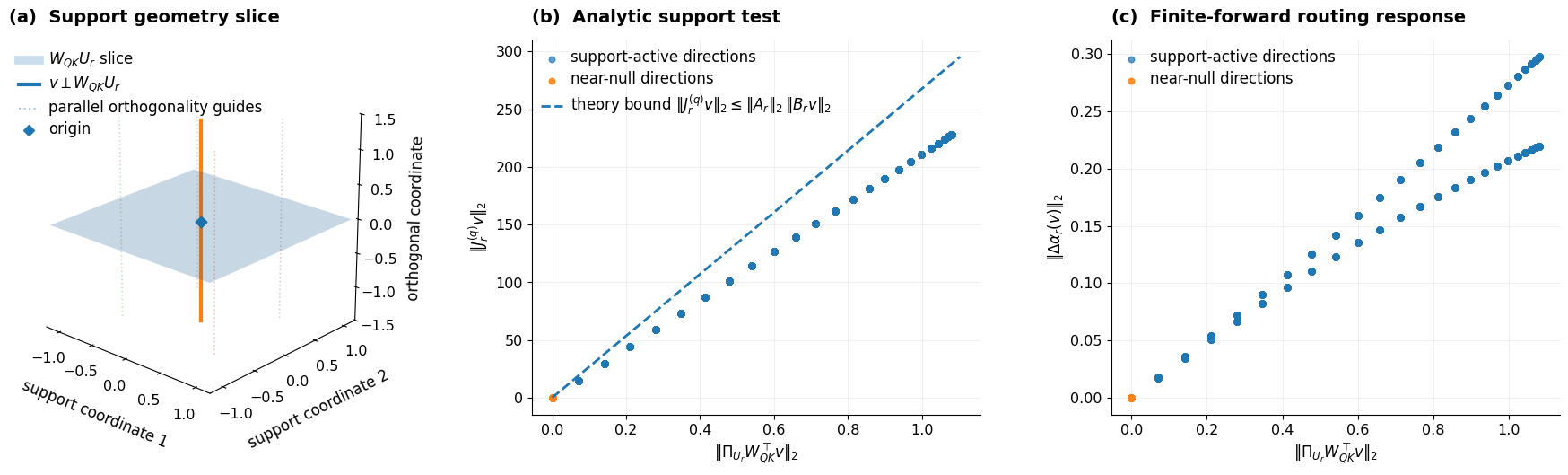}
    \caption{\textbf{Geometric support of query gating.}
    \textbf{(a)} Schematic three-dimensional slice illustrating directions
    orthogonal to the query-gating support.
    \textbf{(b)} Analytical response versus support overlap, including
    near-null directions and the theoretical upper bound.
    \textbf{(c)} Finite-forward changes in attention weights under
    fixed-key query perturbations.}
    \label{fig:exp_query_support}
\end{figure*}

\begin{figure*}[htp]
    \centering
    \begin{minipage}[c]{0.61\textwidth}
        \centering
        \includegraphics[width=\linewidth]
        {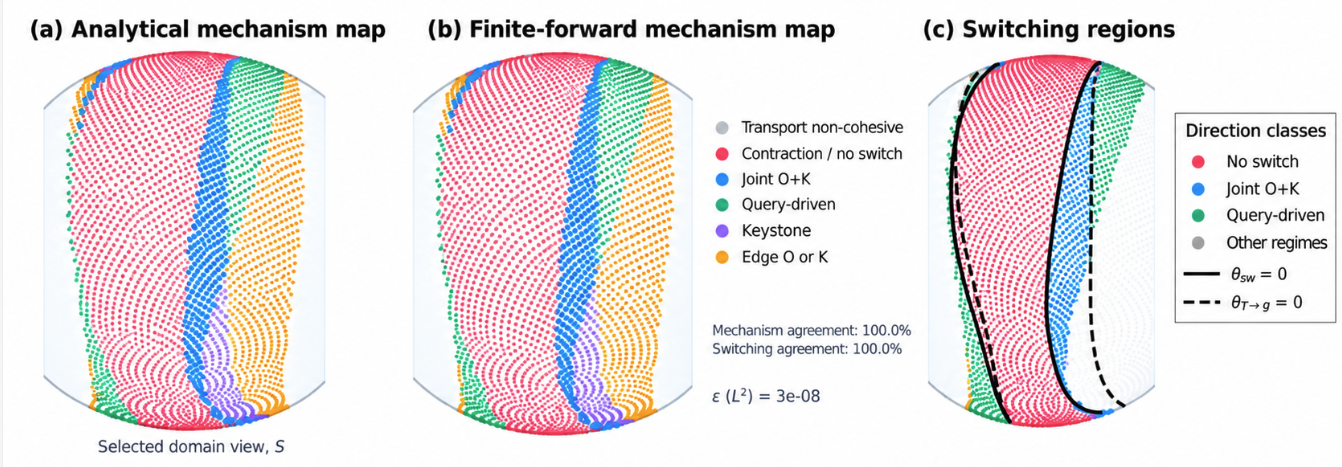}
    \end{minipage}\hfill
    \begin{minipage}[c]{0.37\textwidth}
        \centering
        \includegraphics[width=\linewidth]
        {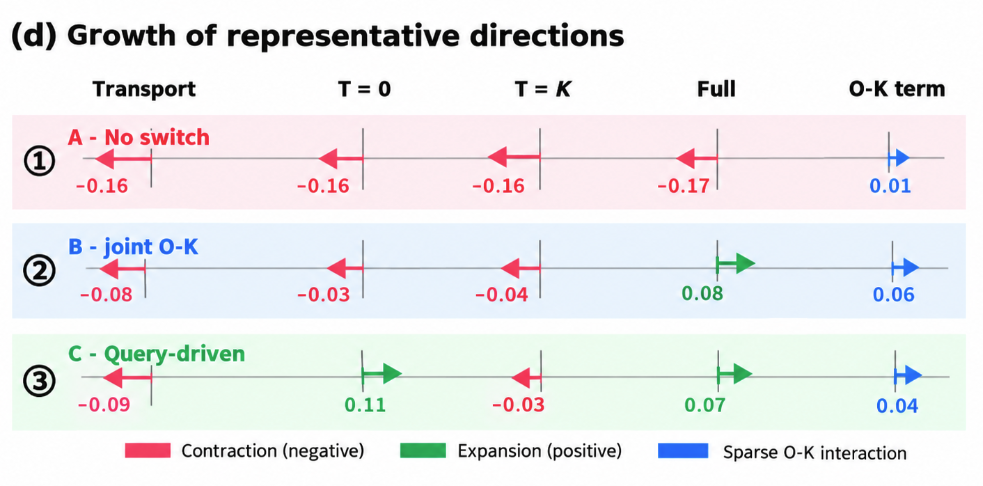}
    \end{minipage}
    \caption{\textbf{Theory-predicted regime geometry and gating-pathway ablations.}
    \textbf{(a,b)} Analytical and finite-forward regime maps on the same
    selected directional slice, with exact switching and mechanism-label
    agreement across the tested directions. \textbf{(c)} Growth-sign
    boundaries separating the observed regimes. \textbf{(d)} Directional
    growth under transport-only, single-gating-pathway, and full-attention
    counterfactuals for representative directions.}
    \label{fig:exp_regime_switching}
\end{figure*}

\subsection{Geometric Support of Query Gating}
\label{subsec:exp_query_geometry}

\paragraph{From attended geometry to query-gating support.}
Before asking whether attention gating can reverse transport contraction, we first
identify the directions along which query gating can act.
Proposition~\ref{prop:query_gating_support} predicts that the variation among
attended representations restricts the query-gating response: if
$W_{QK}^{\top}v\perp\mathcal U_r$, then $J_r^{(q)}v=0$.
Figure~\ref{fig:exp_query_support}(a) illustrates this geometric constraint
on a three-dimensional slice: perturbations orthogonal to
$W_{QK}\mathcal U_r$ cannot produce a first-order query-gating output
response. To test the prediction, we compare the support overlap
$s(v)=\|\Pi_{\mathcal U_r}W_{QK}^{\top}v\|_2$ with the analytical response
$\|J_r^{(q)}v\|_2$, and separately measure finite-forward changes in
attention weights after perturbing the query with keys held fixed.

\paragraph{From geometric selectivity to finite-forward behavior.}
Figure~\ref{fig:exp_query_support}(b) confirms that query gating is highly
selective: directions outside the attended-variation support produce
negligible analytical responses, whereas directions aligned with that
support can produce substantial changes. The observed Jacobian rank
attains the predicted geometric bound, reinforcing that the response
is constrained by variation among the attended representations rather
than by query perturbations alone. This directional selectivity also
appears in the finite-forward test in
Figure~\ref{fig:exp_query_support}(c), where near-null perturbations
produce negligible routing changes while support-active perturbations
can alter attention allocation. Importantly, the proposition constrains
the attention \emph{output}, not attention weights in general; the
finite-forward routing result is specific to the tested directions.
Thus, attended geometry rules out certain gating responses without
fully determining the magnitude of the remaining ones. Even when
gating is active, its effect on representation growth remains unknown.
We therefore next examine when query and key gating can overcome
value-transport contraction and induce a regime switch. Additional
geometric analyses appear in Appendix~\ref{subsec:app_exp_mechanisms}.

\subsection{From Gating Support to Regime Switching}
\label{subsec:exp_gating_regime}

\begin{figure*}[t]
    \centering
    \includegraphics[width=0.93\textwidth]{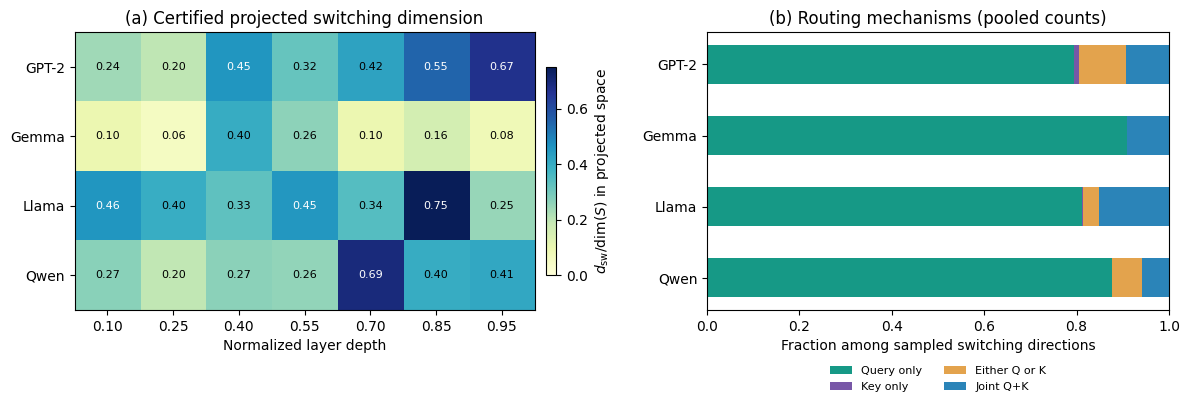}
    \caption{\textbf{Switching across architectures and depths.}
    \textbf{(a)} Median normalized effective switching dimension
    across seven normalized depths and four Transformer architectures.
    Normalization uses the projected transport-contractive subspace.
    \textbf{(b)} Mechanism composition among sampled switching
    directions. Measurements use two WikiText-2 validation prompts
    per model and 64-dimensional random orthonormal probe subspaces.}
    \label{fig:exp_cross_architecture}
\end{figure*}

\begin{figure*}[htp]
    \centering
    \includegraphics[width=0.97\textwidth]{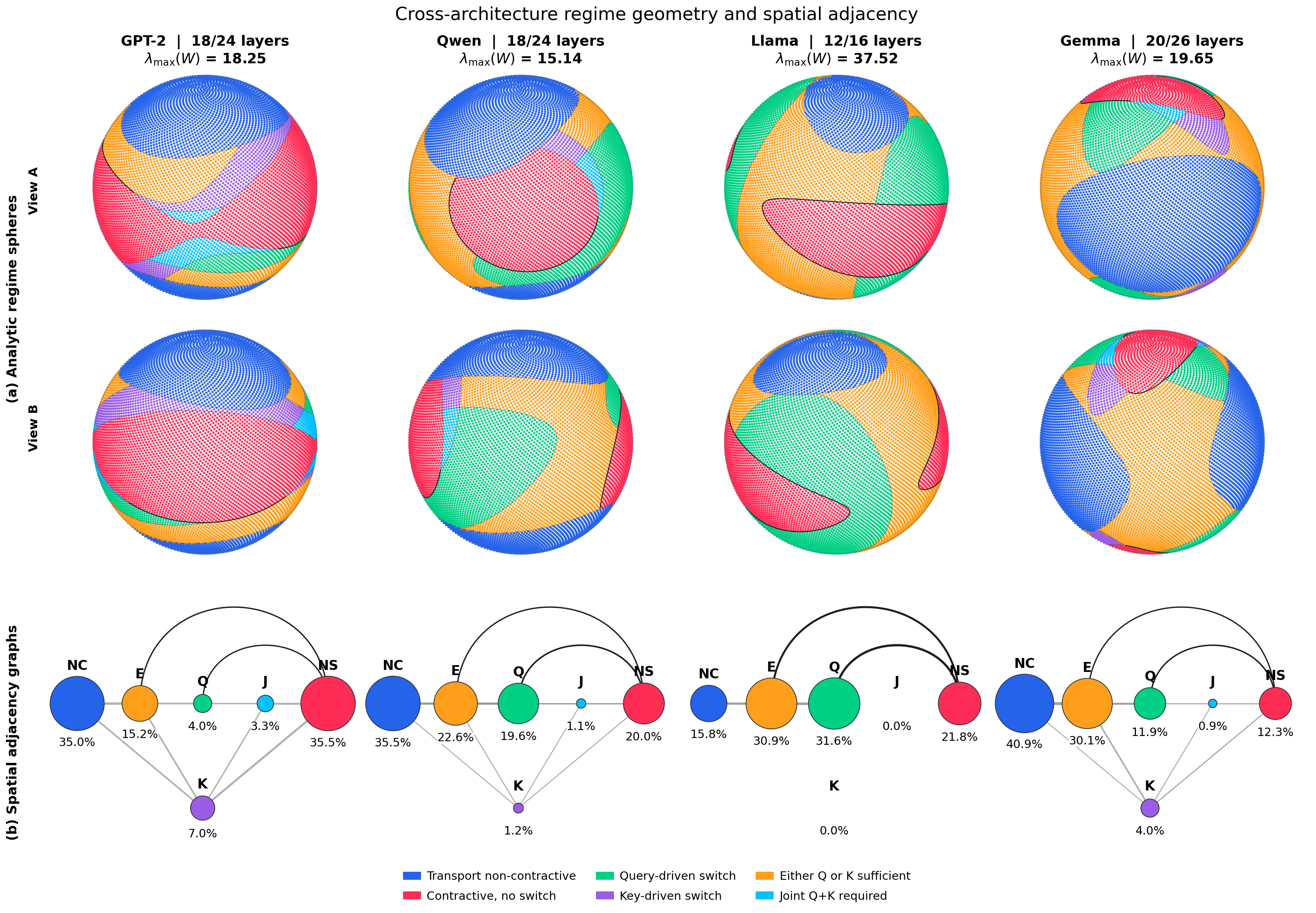}
    \caption{\textbf{Cross-architecture regime geometry and spatial adjacency.}
    \textbf{(a) Analytic regime spheres:} Two orthographic views
    (View A and B) of selected spherical slices across four Transformer
    architectures, colored by their predicted growth regimes.
    \textbf{(b) Spatial adjacency graphs:} Node sizes indicate regime
proportions, and edges represent shared regime boundaries.
NC: transport non-contractive; NS: contraction without switching;
Q/K: query-/key-driven switching; E: either pathway sufficient;
J: joint query--key required. Percentages are computed over all sampled sphere directions,
rather than only switching directions as in
Figure~\ref{fig:exp_cross_architecture}(b).
Results characterize the selected slices, not the full representation space.}
    \label{fig:exp_cross_arch_spatial_adjacency}
\end{figure*}

\paragraph{From geometric support to theory-predicted regimes.}
Attended-representation geometry constrains the directions along which query gating can act, but does not by itself determine whether gating reverses transport contraction. Proposition~\ref{prop:gating_growth_correction} establishes the mechanism underlying this reversal: a direction undergoes regime switching when the gating-induced growth correction exceeds the contraction imposed by fixed-weight value transport. To examine the resulting directional structure, we construct analytical regime maps on selected spherical slices of four Transformer architectures, as illustrated in Figure~\ref{fig:exp_cross_arch_spatial_adjacency}(a). The spheres distinguish transport non-contraction, contraction without switching, and query-driven, key-driven, either-pathway, and joint query--key switching. Although the relative sizes and spatial arrangements of these regions vary substantially across architectures, the same growth-sign criteria consistently explain their classification. At the selected operating point used for finite-forward validation, the analytical predictions in Figure~\ref{fig:exp_regime_switching}(a) agree with Figure~\ref{fig:exp_regime_switching}(b) on all tested directions, achieving $100\%$ agreement in both switching status and mechanism attribution.

\paragraph{Pathway interactions explain regime boundaries and spatial adjacency.}
The pathway ablations in Figure~\ref{fig:exp_regime_switching}(d)
distinguish directions that remain contracting from those rescued by
query gating, key gating, or both. The adjacency graphs in
Figure~\ref{fig:exp_cross_arch_spatial_adjacency}(b) reveal direct
boundaries between the no-switch (NS) regime and both the query-driven
(Q) and either-pathway-sufficient (E) regimes across all four architectures.
An NS--E boundary occurs where full directional growth vanishes while
transport contracts and either gating pathway alone produces expansion.
This requires a sufficiently negative query--key interaction to offset
positive single-pathway growth, showing that the combined gating effect
is not simply additive. These boundaries connect pathway interactions
to spatial adjacency, while differences in regime proportions reveal
architecture-dependent geometry. The results concern selected
three-dimensional slices, not universal topology or full-space prevalence.

\subsection{Spectral Certification and Reversal Severity}
\label{subsec:exp_spectral_severity}

\paragraph{From observed switching to spectral certification.}
Having established that attention gating can reverse transport
contraction along selected directions, a fundamental question remains:
does switching occur only along isolated directions, or does it exhibit
a richer geometric structure? To address this question, we examine
the spectral predictions of
Theorems~\ref{thm:spectral_regime_switching},
\ref{thm:effective_switching_dimension}, and
\ref{thm:absolute_reversal_severity}.
Figure~\ref{fig:exp_spectral_certification_severity}(a) shows that
the switching spectrum contains multiple eigenvalues above the
critical threshold, certifying a $47$-dimensional subspace in which
every nonzero direction contracts under transport but expands under
full attention. Thus, switching is not restricted to isolated
directions; it occupies a nontrivial linear subspace of the local
representation geometry.

\begin{wrapfigure}{r}{0.49\columnwidth}
    \vspace{-10pt}
    \centering
    \includegraphics[width=\linewidth]{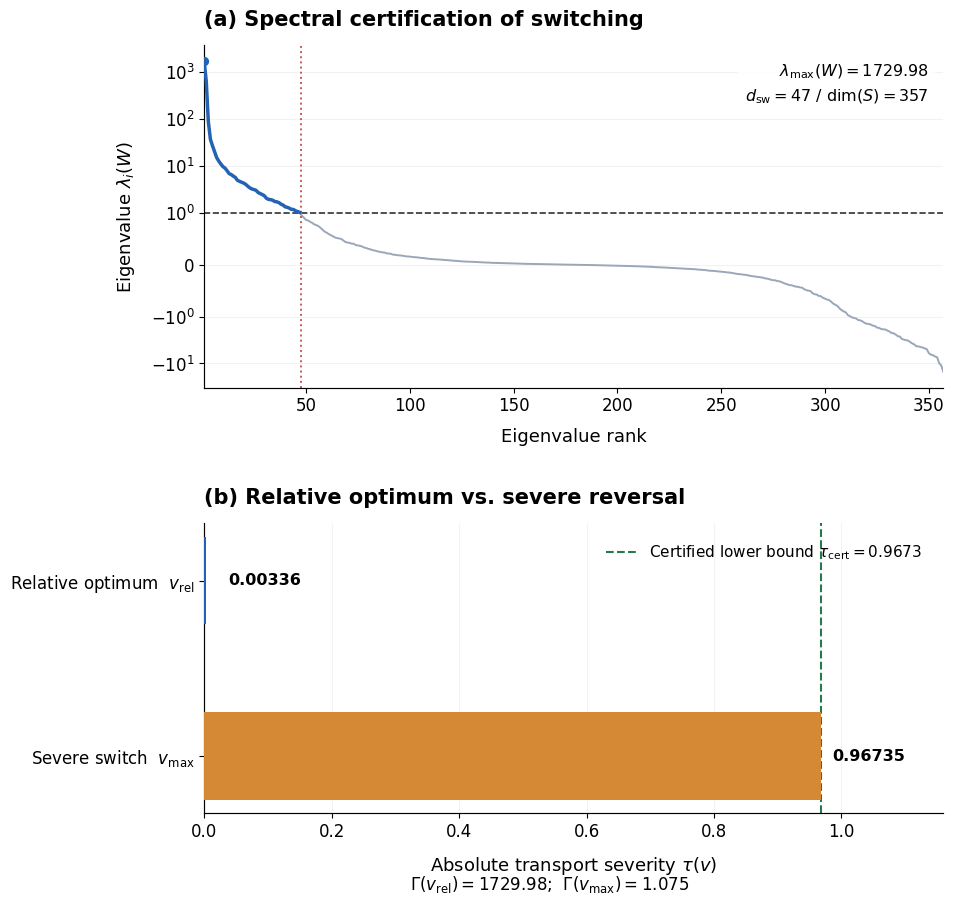}
    \vspace{-9pt}
    \caption{ \textbf{Spectral certification.}
    (a) Switching spectrum; (b) relative/absolute reversal severity.}
    \label{fig:exp_spectral_certification_severity}
    \vspace{-12pt}
\end{wrapfigure}

\paragraph{Relative switching strength versus reversal severity.}
While the spectrum certifies switching, its largest eigenvalue
may not identify the direction undergoing the strongest
absolute contraction reversal.
Figure~\ref{fig:exp_spectral_certification_severity}(b) reveals
that the direction maximizing relative switching strength experiences
weak transport contraction, whereas another direction
undergoes severe contraction before being rescued by gating.
The severity certificate closely matches this extreme while
providing a rigorous lower bound rather than an exact optimum.
These directions also involve different
mechanisms: the relative-optimal switch is query-driven, whereas
the severe switch requires both query and key gating
(Appendix~\ref{subsec:app_exp_strongest}).
Together, these findings distinguish the existence of switching
from its absolute severity and connect spectral structure to
the gating mechanisms. We next investigate whether
switching geometry persists across Transformer architectures
and depths.

\paragraph{Search quality and amortized computational cost.}
Spectral guidance provides a principled alternative to unguided
directional search by exploiting the geometry of transport
contraction and gating-induced growth.
Figure~\ref{fig:exp_search_cost}(a) shows that random probing
struggles to recover severe switching directions, even when
restricted to the transport-contractive subspace $S$.
While such restriction facilitates the discovery of ordinary
switches, it does not identify directions where gating overcomes
particularly strong contraction. In contrast, spectral analysis
directly characterizes switching directions and enables the
certification of severe reversals.
Figure~\ref{fig:exp_search_cost}(b) further illustrates the
computational benefit of this characterization: after a one-time
spectral preprocessing cost, directional properties can be
evaluated through inexpensive quadratic forms, making the
approach increasingly efficient as the number of queries grows.
Together, these results demonstrate that spectral theory provides
not only existence guarantees but also a structured and
computationally useful characterization of regime switching.
We next examine whether switching geometry persists across
Transformer architectures and depths.

\begin{figure*}[t]
    \centering
    \includegraphics[width=0.48\textwidth]{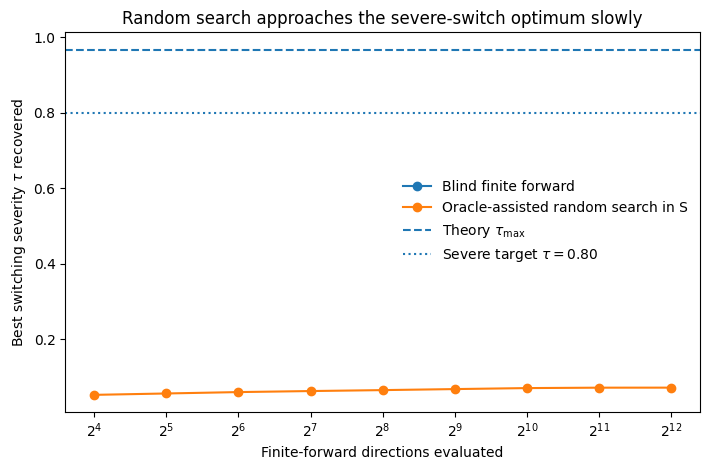}
    \hfill
    \includegraphics[width=0.48\textwidth]{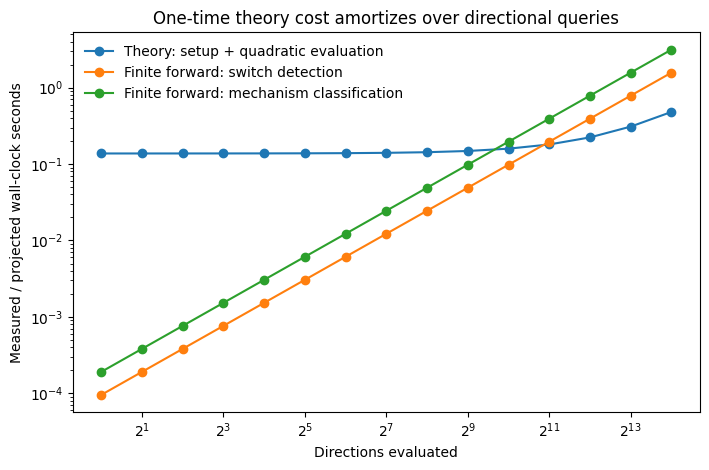}
    \caption{\textbf{Search quality and amortized cost of spectral guidance.}
    (a) Blind random search finds no switch within the tested budget,
    while oracle-assisted sampling in the transport-contractive subspace
    finds ordinary switches but misses severe reversals.
    (b) Spectral preprocessing incurs a one-time setup cost but becomes
    advantageous for sufficiently many directional evaluations.
    Total-cost curves are projected from measured setup and throughput timings.}
    \label{fig:exp_search_cost}
\end{figure*}

\subsection{Cross-Architecture Switching Geometry}
\label{subsec:exp_cross_architecture}

\paragraph{From local certification to cross-architecture geometry.}
Having established that spectral analysis can certify the existence,
dimensionality, and severity of switching at a selected operating point,
we next investigate whether this geometry persists across Transformer
architectures and depths. We examine GPT-2 Medium, Gemma-2-2B,
Llama-3.2-1B, and Qwen2.5-0.5B using projected perturbation subspaces.
Figure~\ref{fig:exp_cross_architecture}(a) reveals nontrivial switching
subspaces across all evaluated architectures and depths, although their
relative dimensions vary substantially without a universal depth trend.
Moreover, Figure~\ref{fig:exp_cross_architecture}(b) shows that the
mechanism-level distinctions also persist: query-driven switching
dominates among sampled switching directions, while joint query--key
mechanisms appear across all four architectures. These observations
suggest that switching is not specific to a single model configuration,
although its spectral geometry and pathway composition remain
architecture-dependent.

\paragraph{Spatial organization of switching regimes.}
While spectral dimensions and mechanism frequencies quantify where
switching occurs, they do not reveal how different regimes are
organized across perturbation directions.
Figure~\ref{fig:exp_cross_arch_spatial_adjacency} addresses this question
by visualizing the regime partitions and their spatial adjacency on
selected spherical slices. The same mechanism classes emerge across
architectures, yet their spatial distributions differ markedly:
query-driven regions occupy relatively small portions of the selected
GPT-2 and Qwen spheres, but substantially larger regions in Llama
and Gemma. The adjacency graphs further reveal how distinct regimes
meet within each slice, illustrating that switching is organized into
structured directional regions rather than appearing only as isolated
examples. Together, these findings suggest a recurring mechanism of
gating-induced switching with architecture-dependent geometric
organization. However, the visualizations do not establish a universal
topological structure, and the limited prompts and projected subspaces
preclude conclusions about full-space prevalence. Additional sensitivity
analyses are provided in Appendix~\ref{subsec:app_exp_cross_arch}.

\section{Conclusion}

We developed a mixture-gated theory of anisotropic representation dynamics in Transformers, revealing how interactions between value transport and attention gating can reverse representation contraction. Our analysis establishes geometric constraints on query gating and provides spectral criteria to characterize the existence, severity, and dimensionality of regime switching. Experiments across multiple pretrained Transformers support these theoretical predictions and demonstrate the distinct and joint contributions of query and key gating. Future work will extend this local analysis to representation dynamics across multiple layers and investigate its implications for downstream model behavior.

\bibliographystyle{plain}
\bibliography{references}

@inproceedings{dong2021attention,
  title     = {Attention Is Not All You Need: Pure Attention Loses Rank Doubly Exponentially with Depth},
  author    = {Dong, Yihe and Cordonnier, Jean-Baptiste and Loukas, Andreas},
  booktitle = {Proceedings of the 38th International Conference on Machine Learning},
  series    = {Proceedings of Machine Learning Research},
  volume    = {139},
  pages     = {2793--2803},
  year      = {2021},
  publisher = {PMLR},
  url       = {https://proceedings.mlr.press/v139/dong21a.html}
}

@inproceedings{kim2021lipschitz,
  title={The lipschitz constant of self-attention},
  author={Kim, Hyunjik and Papamakarios, George and Mnih, Andriy},
  booktitle={International Conference on Machine Learning},
  pages={5562--5571},
  year={2021},
  organization={PMLR}
}

@inproceedings{ormaniec2025does,
  title={What does it mean to be a transformer? insights from a theoretical hessian analysis},
  author={Ormaniec, Weronika and Dangel, Felix and Singh, Sidak Pal},
  booktitle={International Conference on Learning Representations},
  volume={2025},
  pages={17046--17081},
  year={2025}
}

@article{nguyen2026statistical,
  title={A Statistical Theory of Gated Attention through the Lens of Hierarchical Mixture of Experts},
  author={Nguyen, Viet and Pham, Tuan Minh and Cao, Thinh and Dinh, Tan and Nguyen, Huy and Ho, Nhat and Rinaldo, Alessandro},
  journal={arXiv preprint arXiv:2602.01468},
  year={2026}
}

@InProceedings{pmlrv235bao24b,
  title = 	 {Self-attention Networks Localize When {QK}-eigenspectrum Concentrates},
  author =       {Bao, Han and Hataya, Ryuichiro and Karakida, Ryo},
  booktitle = 	 {Proceedings of the 41st International Conference on Machine Learning},
  pages = 	 {2903--2922},
  year = 	 {2024},
  editor = 	 {Salakhutdinov, Ruslan and Kolter, Zico and Heller, Katherine and Weller, Adrian and Oliver, Nuria and Scarlett, Jonathan and Berkenkamp, Felix},
  volume = 	 {235},
  series = 	 {Proceedings of Machine Learning Research},
  month = 	 {21--27 Jul},
  publisher =    {PMLR},
  url = 	 {https://proceedings.mlr.press/v235/bao24b.html}
}

@inproceedings{noci2022signal,
  title     = {Signal Propagation in Transformers: Theoretical Perspectives and the Role of Rank Collapse},
  author    = {Noci, Lorenzo and Anagnostidis, Sotiris and Biggio, Luca and Orvieto, Antonio and Singh, Sidak Pal and Lucchi, Aurelien},
  booktitle = {Advances in Neural Information Processing Systems},
  volume    = {35},
  year      = {2022},
  doi       = {10.52202/068431-1972},
  url       = {https://proceedings.neurips.cc/paper_files/paper/2022/hash/ae0cba715b60c4052359b3d52a2cff7f-Abstract-Conference.html}
}

@inproceedings{kedia2024transformers,
  title     = {Transformers Get Stable: An End-to-End Signal Propagation Theory for Language Models},
  author    = {Kedia, Akhil and Zaidi, Mohd Abbas and Khyalia, Sushil and Jung, Jungho and Goka, Harshith and Lee, Haejun},
  booktitle = {Proceedings of the 41st International Conference on Machine Learning},
  series    = {Proceedings of Machine Learning Research},
  volume    = {235},
  pages     = {23449--23531},
  year      = {2024},
  publisher = {PMLR},
  url       = {https://proceedings.mlr.press/v235/kedia24a.html}
}

@inproceedings{pham2025dynamical,
  title     = {Dynamical Properties of Tokens in Self-Attention and Effects of Positional Encoding},
  author    = {Pham, Duy-Tung and Nguyen, An The and Tran, Viet-Hoang and Chung, Nhan-Phu and Tong, Xin and Nguyen, Tan M. and Vo, Thieu N.},
  booktitle = {Advances in Neural Information Processing Systems},
  volume    = {38},
  year      = {2025},
  url       = {https://proceedings.neurips.cc/paper_files/paper/2025/hash/ea7e6b3b58422f1388468a974af62391-Abstract-Conference.html}
}

@inproceedings{tomihari2025recurrent,
  title     = {Recurrent Self-Attention Dynamics: An Energy-Agnostic Perspective from Jacobians},
  author    = {Tomihari, Akiyoshi and Karakida, Ryo},
  booktitle = {Advances in Neural Information Processing Systems},
  volume    = {38},
  pages     = {155236--155269},
  year      = {2025},
  doi       = {10.52202/085713-4675},
  url       = {https://arxiv.org/abs/2505.19458}
}

@inproceedings{aubry2025transformer,
  title     = {Transformer Block Coupling and Its Correlation with Generalization in LLMs},
  author    = {Aubry, Murdock and Meng, Haoming and Sugolov, Anton and Papyan, Vardan},
  booktitle = {The Thirteenth International Conference on Learning Representations},
  year      = {2025},
  url       = {https://openreview.net/forum?id=kvLenbZZgg}
}

@inproceedings{saratchandran2025spectral,
  title     = {Spectral Conditioning of Attention Improves Transformer Performance},
  author    = {Saratchandran, Hemanth and Lucey, Simon},
  booktitle = {Advances in Neural Information Processing Systems},
  volume    = {38},
  year      = {2025},
  doi       = {10.52202/085713-1982},
  url       = {https://proceedings.neurips.cc/paper_files/paper/2025/hash/5551e21f0ffb5955befb1209534354c6-Abstract-Conference.html}
}

@inproceedings{jin2025moh,
  title     = {{MoH}: Multi-Head Attention as Mixture-of-Head Attention},
  author    = {Jin, Peng and Zhu, Bo and Yuan, Li and Yan, Shuicheng},
  booktitle = {Proceedings of the 42nd International Conference on Machine Learning},
  series    = {Proceedings of Machine Learning Research},
  volume    = {267},
  pages     = {28233--28255},
  year      = {2025},
  publisher = {PMLR},
  url       = {https://proceedings.mlr.press/v267/jin25l.html}
}

@misc{yan2025sigmoid,
  title         = {Sigmoid Self-Attention Has Lower Sample Complexity than Softmax Self-Attention: A Mixture-of-Experts Perspective},
  author        = {Yan, Fanqi and Nguyen, Huy and Akbarian, Pedram and Ho, Nhat and Rinaldo, Alessandro},
  year          = {2025},
  eprint        = {2502.00281},
  archivePrefix = {arXiv},
  primaryClass  = {cs.LG},
  doi           = {10.48550/arXiv.2502.00281},
  url           = {https://arxiv.org/abs/2502.00281},
  note          = {Preprint}
}

@inproceedings{ilyas2018blackbox,
  title     = {Black-box Adversarial Attacks with Limited Queries and Information},
  author    = {Ilyas, Andrew and Engstrom, Logan and Athalye, Anish and Lin, Jessy},
  booktitle = {Proceedings of the 35th International Conference on Machine Learning},
  series    = {Proceedings of Machine Learning Research},
  volume    = {80},
  pages     = {2137--2146},
  year      = {2018},
  publisher = {PMLR}
}

@inproceedings{meng2022rome,
  title     = {Locating and Editing Factual Associations in GPT},
  author    = {Meng, Kevin and Bau, David and Andonian, Alex and Belinkov, Yonatan},
  booktitle = {Advances in Neural Information Processing Systems},
  volume    = {35},
  year      = {2022}
}

@misc{turner2024activation,
  title         = {Activation Addition: Steering Language Models Without Optimization},
  author        = {Turner, Alexander Matt and Thiergart, Lisa and Leech, Gavin and Udell, David and Mini, Ulisse and MacDiarmid, Monte},
  year          = {2024},
  eprint        = {2308.10248},
  archivePrefix = {arXiv},
  primaryClass  = {cs.CL}
}

@article{liu2024lost,
  title   = {Lost in the Middle: How Language Models Use Long Contexts},
  author  = {Liu, Nelson F. and Lin, Kevin and Hewitt, John and Paranjape, Ashwin and Bevilacqua, Michele and Petroni, Fabio and Liang, Percy},
  journal = {Transactions of the Association for Computational Linguistics},
  volume  = {12},
  pages   = {157--173},
  year    = {2024},
  doi     = {10.1162/tacl_a_00638}
}

\newpage

\noindent\rule{\textwidth}{1.5pt}  

\begin{center}
    {\LARGE\bfseries Supplementary for A Mixture-Gated Theory of Anisotropic Representation Dynamics in Transformers}
\end{center}

\noindent\rule{\textwidth}{1.5pt}  

\addcontentsline{toc}{section}{Appendix}

\vspace{0.5em}
\addcontentsline{toc}{section}{Appendix}

\textbf{This supplementary material serves as the appendix to the main paper and provides
additional mathematical details, complete proofs of the theoretical results,
extensions of the proposed analysis, and full experimental settings and additional
results. It is organized as follows:}

\begin{itemize}[leftmargin=*, label=]

    \item \phantomsection
    \textbf{\ref{sec:notation}. \nameref{sec:notation}}
    \dotfill \pageref{sec:notation}

    \item \phantomsection
    \textbf{\ref{sec:app_theory}. \nameref{sec:app_theory}}
    \dotfill \pageref{sec:app_theory}

    \begin{itemize}
        \item[] \ref{subsec:app_notation}. \nameref{subsec:app_notation}
        \dotfill \pageref{subsec:app_notation}

        \item[] \ref{subsec:app_local_dynamics}. \nameref{subsec:app_local_dynamics}
        \dotfill \pageref{subsec:app_local_dynamics}

        \item[] \ref{subsec:app_transport_proof}. \nameref{subsec:app_transport_proof}
        \dotfill \pageref{subsec:app_transport_proof}

        \item[] \ref{app:attention_jacobian}. \nameref{app:attention_jacobian}
        \dotfill \pageref{app:attention_jacobian}

        \item[] \ref{subsec:app_gating_geometry}. \nameref{subsec:app_gating_geometry}
        \dotfill \pageref{subsec:app_gating_geometry}

        \item[] \ref{subsec:app_directional_alignment}. \nameref{subsec:app_directional_alignment}
        \dotfill \pageref{subsec:app_directional_alignment}

        \item[] \ref{subsec:app_modern_blocks}. \nameref{subsec:app_modern_blocks}
        \dotfill \pageref{subsec:app_modern_blocks}
    \end{itemize}

    \item \phantomsection
    \textbf{\ref{sec:app_regime_theory}. \nameref{sec:app_regime_theory}}
    \dotfill \pageref{sec:app_regime_theory}

    \begin{itemize}
        \item[] \ref{subsec:app_transport_full_growth}. \nameref{subsec:app_transport_full_growth}
        \dotfill \pageref{subsec:app_transport_full_growth}

        \item[] \ref{subsec:app_spectral_switch}. \nameref{subsec:app_spectral_switch}
        \dotfill \pageref{subsec:app_spectral_switch}

        \item[] \ref{subsec:app_switch_strength}. \nameref{subsec:app_switch_strength}
        \dotfill \pageref{subsec:app_switch_strength}

        \item[] \ref{subsec:app_no_switch}. \nameref{subsec:app_no_switch}
        \dotfill \pageref{subsec:app_no_switch}

        \item[] \ref{subsec:app_switch_robustness}. \nameref{subsec:app_switch_robustness}
        \dotfill \pageref{subsec:app_switch_robustness}
    \end{itemize}




    \item \phantomsection
    \textbf{\ref{sec:app_experiments}. \nameref{sec:app_experiments}}
    \dotfill \pageref{sec:app_experiments}

    \begin{itemize}
        \item[] \ref{subsec:app_exp_setup}. \nameref{subsec:app_exp_setup}
        \dotfill \pageref{subsec:app_exp_setup}

        \item[] \ref{subsec:app_exp_mechanisms}. \nameref{subsec:app_exp_mechanisms}
        \dotfill \pageref{subsec:app_exp_mechanisms}

        \item[] \ref{subsec:app_exp_switching}. \nameref{subsec:app_exp_switching}
        \dotfill \pageref{subsec:app_exp_switching}

        \item[] \ref{subsec:app_exp_strongest}. \nameref{subsec:app_exp_strongest}
        \dotfill \pageref{subsec:app_exp_strongest}

        \item[] \ref{subsec:app_exp_search_cost}. \nameref{subsec:app_exp_search_cost}
        \dotfill \pageref{subsec:app_exp_search_cost}



    \end{itemize}

\end{itemize}

\vspace{1em}
\hrule
\vspace{1em}

\newpage

\section{TABLE OF NOTATION}
\label{sec:notation}

Table~\ref{tab:notation} summarizes the notation central to the representation-dynamics, attention-gating, and spectral regime-switching theory.

{
\small
\setlength{\tabcolsep}{5pt}
\renewcommand{\arraystretch}{1.09}
\setlength{\LTleft}{0pt}
\setlength{\LTright}{0pt}

\newcommand{\notationgroup}[1]{%
  \addlinespace[1pt]
  \multicolumn{2}{@{}l}{\textbf{#1}}\\[-1pt]
}

\begin{longtable}{
    >{\raggedright\arraybackslash}p{0.275\textwidth}
    >{\raggedright\arraybackslash}p{0.66\textwidth}
}

\caption{Summary of the main theoretical notation.}
\label{tab:notation}\\
\toprule
\textbf{Notation} & \textbf{Meaning} \\
\midrule
\endfirsthead

\multicolumn{2}{l}{\small\textit{Table~\ref{tab:notation} continued.}}\\
\toprule
\textbf{Notation} & \textbf{Meaning} \\
\midrule
\endhead

\midrule
\multicolumn{2}{r}{\small\textit{Continued on next page}}\\
\endfoot

\bottomrule
\endlastfoot


\notationgroup{Local representation dynamics}

$\ell,\ z_\ell,\ \delta z_\ell,\ T_\ell$
& Layer index, incoming representation, local perturbation, and Transformer block map. \\

$J_\ell,\ v,\ M$
& Block Jacobian, nonzero perturbation direction, and symmetric positive-definite representation metric. \\

$G_\ell^M$
& Local growth operator $J_\ell^\top M J_\ell-M$, governing directional expansion and contraction. \\


\midrule
\notationgroup{Self-attention and query-gating geometry}

$x_i,\ r,\ i,\ m$
& Token representation; query-token, attended-token, and perturbed-token indices. \\

$W_{QK},\ W_{VO}$
& Combined query--key score operator and value--output transformation. \\

$s_{ri},\ \alpha_{ri},\ F_r$
& Attention score, softmax weight, and attention output at query token $r$. \\

$\mu_r,\ \Sigma_r$
& Attention-weighted mean and covariance of attended token representations. \\

$\mathcal U_r,\ J_r^{(q)}$
& Attended-variation subspace $\operatorname{Range}(\Sigma_r)$ and query-gating Jacobian $W_{VO}\Sigma_r W_{QK}^\top$. \\


\midrule
\notationgroup{Transport--gating decomposition and growth}

$J_T,\ J_Q,\ J_K,\ J_G$
& Transport, query-gating, key-gating, and combined gating Jacobians ($J_G=J_Q+J_K$). \\

$J_{\mathrm{full}}$
& Full-attention Jacobian $J_T+J_Q+J_K$. \\

$G_T^M,\ G_{\mathrm{full}}^M$
& Transport-only and full-attention growth operators $J_T^\top M J_T-M$ and $J_{\mathrm{full}}^\top M J_{\mathrm{full}}-M$. \\

$\Delta G_G^M$
& Gating-induced growth correction $G_{\mathrm{full}}^M-G_T^M$, including transport--gating interactions. \\

$\mathcal I_{\mathrm{gate}}(v),\ D_v,\ S_v,\ \rho_v$
& Signed directional gating contribution and its diversity--sensitivity--alignment factors, $\mathcal I_{\mathrm{gate}}(v)=\rho_v D_vS_v$. \\

$g_T(v),\ g_{\mathrm{full}}(v)$
& Scalar directional growths $v^\top G_T^M v$ and $v^\top G_{\mathrm{full}}^M v$, with the metric understood. \\


\midrule
\notationgroup{Spectral regime-switching theory}

$S,\ P,\ u$
& Transport-contractive subspace, an orthonormal basis matrix, and coordinates $v=Pu$. \\

$C,\ R$
& Projected contraction $-P^\top G_T^M P$ and projected gating correction $P^\top\Delta G_G^M P$. \\

$W=C^{-1/2}RC^{-1/2}$
& Symmetric, transport-whitened switching operator. \\

$\lambda_{\max}(W),\ w_{\max}$
& Largest switching eigenvalue and an associated unit eigenvector. \\

$\Gamma(v),\ v_{\mathrm{rel}}$
& Relative switching strength $v^\top\Delta G_G^M v/(-v^\top G_T^M v)$ and its maximizing direction. \\

$d_{\mathrm{sw}},\ V_{\mathrm{sw}}$
& Effective switching dimension and a largest-dimensional subspace in which every nonzero direction switches. \\


\midrule
\notationgroup{Absolute reversal severity and certification}

$\tau(v),\ \tau_{\max}$
& Fraction of incoming squared $M$-magnitude removed by transport and supremal severity among switching directions. \\

$B,\ (\eta_i,q_i)$
& Projected metric $B=P^\top MP$ and generalized severity eigenpairs $Cq_i=\eta_i Bq_i$. \\

$\tau_0,\ S_{\tau_0},\ P_{\tau_0}$
& Severity threshold, uniformly contractive subspace, and its orthonormal basis matrix. \\

$C_{\tau_0},\ R_{\tau_0},\ W_{\tau_0}$
& Contraction, gating-correction, and whitened switching operators restricted to $S_{\tau_0}$. \\

$\tau_{\mathrm{cert}}$
& Strongest certified severity threshold, satisfying $\tau_{\mathrm{cert}}\leq\tau_{\max}$. \\

\end{longtable}
}
\newpage

\section{MATHEMATICAL DETAILS OF REPRESENTATION DYNAMICS}
\label{sec:app_theory}

\subsection{Notation and Technical Setup}
\label{subsec:app_notation}

We first fix the notation used throughout the supplementary material. A Transformer is viewed as a composition of blocks \(T_\ell\), where \(\ell\) always denotes the \emph{layer index}. The representation entering layer \(\ell\) is \(z_\ell\in\mathbb{R}^d\), and the block output is \(z_{\ell+1}=T_\ell(z_\ell)\). We study two nearby representations \(z_\ell\) and \(z_\ell+\delta z_\ell\), where \(\delta z_\ell\in\mathbb{R}^d\) denotes their local difference. The Jacobian of the block at \(z_\ell\) is \(J_\ell:=\partial T_\ell/\partial z_\ell\).

Throughout the local analysis, we assume that each block \(T_\ell\) is twice continuously differentiable in a neighborhood of the representation under consideration. This ensures that, for sufficiently small \(\delta z_\ell\), the block admits the first-order expansion
\begin{equation}
T_\ell(z_\ell+\delta z_\ell)-T_\ell(z_\ell)
=
J_\ell\delta z_\ell
+
\mathcal{O}(\|\delta z_\ell\|_2^2).
\label{eq:app_local_taylor_setup}
\end{equation}
The term \(\mathcal{O}(\|\delta z_\ell\|_2^2)\) collects the higher-order contribution that vanishes quadratically as the perturbation magnitude approaches zero. Unless stated otherwise, all directional statements in the main theory refer to this local linearization.

Representation magnitudes are measured using a fixed symmetric positive-definite matrix \(M=M^\top\succ0\). For \(u,v\in\mathbb{R}^d\), we write \(\langle u,v\rangle_M:=u^\top Mv\) and \(\|v\|_M^2:=v^\top Mv\). The Euclidean norm is denoted by \(\|\cdot\|_2\). For a matrix \(A\), \(\|A\|_2\) denotes its spectral norm, \(\lambda_{\max}(A)\) its largest eigenvalue when \(A\) is symmetric, \(\operatorname{Range}(A)\) its column space, and \(\operatorname{rank}(A)\) its rank.

For the self-attention analysis, we fix one Transformer layer and one attention head and suppress their layer and head superscripts whenever no ambiguity arises. Let \(n\) be the number of tokens and let \(x_i\in\mathbb{R}^d\) denote the representation of token \(i\) at this fixed layer. Token indices are kept distinct from the layer index: \(r\) denotes the query token, \(i\) denotes a token attended to by that query, and \(m\) denotes the token whose representation is perturbed. Thus, \(\ell\) is reserved exclusively for Transformer depth, whereas \(r,i,m\) index tokens.

For the single-head attention formulas, let \(W_Q,W_K\in\mathbb{R}^{d\times d_k}\) denote the query and key projection matrices, where \(d_k\) is the query/key dimension. We use the combined query--key operator \(W_{QK}:=W_QW_K^\top/\sqrt{d_k}\in\mathbb{R}^{d\times d}\), so that the score between query token \(r\) and token \(i\) is \(s_{ri}=x_r^\top W_{QK}x_i\). The corresponding attention weight is \(\alpha_{ri}=\exp(s_{ri})/\sum_j\exp(s_{rj})\). We denote by \(W_{VO}\in\mathbb{R}^{d\times d}\) the combined value and output transformation, so the attention output associated with query token \(r\) is
\begin{equation}
F_r
=
\sum_{i=1}^{n}\alpha_{ri}W_{VO}x_i.
\label{eq:app_attention_setup}
\end{equation}

We use \(\mu_r:=\sum_i\alpha_{ri}x_i\) for the attention-weighted mean of the token representations and \(\Sigma_r:=\sum_i\alpha_{ri}(x_i-\mu_r)(x_i-\mu_r)^\top\) for their attention-weighted covariance. The Jacobian \(J_{rm}:=\partial F_r/\partial x_m\) measures how perturbing token \(m\) changes the attention output of query token \(r\). Later, \(J_{\mathrm{transport}}\) and \(J_{\mathrm{gating}}\) denote the components of an attention Jacobian arising respectively from changes in the transported information and from changes in the state-dependent routing weights.

For a subspace \(\mathcal{S}\subseteq\mathbb{R}^d\), \(P\in\mathbb{R}^{d\times q}\) denotes a matrix whose \(q\) columns form a basis of \(\mathcal{S}\). The notation \(A\preceq B\) means that \(B-A\) is positive semidefinite, while \(A\prec B\) means that \(B-A\) is positive definite. These conventions will be used in the spectral characterization of gating-induced regime switching.

\subsection{Local Representation Dynamics and Growth Operator}
\label{subsec:app_local_dynamics}

We formalize how a small representation difference evolves through a Transformer block and derive an operator that determines whether this difference grows or shrinks along a given direction. Let \(z_{\ell+1}=T_\ell(z_\ell)\), where \(T_\ell:\mathbb{R}^d\rightarrow\mathbb{R}^d\) denotes the transformation implemented by layer \(\ell\). Consider two nearby representations \(z_\ell\) and \(z_\ell+\delta z_\ell\). Their difference after the block is \(\delta z_{\ell+1}:=T_\ell(z_\ell+\delta z_\ell)-T_\ell(z_\ell)\). Assuming that \(T_\ell\) is twice continuously differentiable in a neighborhood of \(z_\ell\), Taylor expansion gives
\begin{equation}
\delta z_{\ell+1}
=
J_\ell\delta z_\ell+R_\ell(\delta z_\ell),
\qquad
\|R_\ell(\delta z_\ell)\|_2
=
\mathcal{O}(\|\delta z_\ell\|_2^2),
\label{eq:app_local_dynamics}
\end{equation}
where \(J_\ell:=\partial T_\ell/\partial z_\ell\) is the Jacobian evaluated at \(z_\ell\), and \(R_\ell(\delta z_\ell)\) collects the higher-order terms. Thus, for sufficiently small perturbations, the block acts locally as the linear transformation \(J_\ell\).

\paragraph{From representation magnitude to the growth operator.}
To determine whether a representation difference grows or shrinks, we must compare its magnitude before and after the transformation. Let \(M=M^\top\succ0\) be a fixed symmetric positive-definite matrix defining the squared representation magnitude \(\|v\|_M^2:=v^\top Mv\). The choice \(M=I\) gives the usual Euclidean norm, while a general \(M\) allows different representation directions to be measured with different weights. Importantly, \(M\) is a choice of measurement geometry, not an additional transformation implemented by the Transformer. Unless otherwise specified, the following derivation applies to any fixed \(M\succ0\).

For a perturbation \(\delta z_\ell\), we are interested in the change in its squared magnitude,
\[
\Delta_\ell^M(\delta z_\ell)
:=
\|\delta z_{\ell+1}\|_M^2
-
\|\delta z_\ell\|_M^2.
\]
A positive value means that the representation difference has grown after the block, whereas a negative value means that it has shrunk. Substituting the local approximation \(\delta z_{\ell+1}\approx J_\ell\delta z_\ell\), we obtain
\begin{equation}
\begin{aligned}
\Delta_\ell^M(\delta z_\ell)
&\approx
\|J_\ell\delta z_\ell\|_M^2
-
\|\delta z_\ell\|_M^2\\
&=
(J_\ell\delta z_\ell)^\top
M(J_\ell\delta z_\ell)
-
\delta z_\ell^\top M\delta z_\ell\\
&=
\delta z_\ell^\top
\bigl(J_\ell^\top M J_\ell-M\bigr)
\delta z_\ell.
\end{aligned}
\label{eq:app_growth_derivation}
\end{equation}

This calculation shows that the change in squared representation magnitude is governed by the matrix \(J_\ell^\top M J_\ell-M\). The first term, \(J_\ell^\top M J_\ell\), measures the squared magnitude of a perturbation after applying the local transformation, while the second term, \(M\), measures its original squared magnitude. Their difference therefore determines whether the block increases or decreases the magnitude of a perturbation. This motivates the following definition of the \emph{representation growth operator}:
\begin{equation}
\boxed{
\mathcal{G}_\ell^M:=J_\ell^\top M J_\ell-M.
}
\label{eq:app_growth_operator}
\end{equation}
The growth operator is not an additional model parameter or a separate component of the Transformer. It is derived entirely from the block Jacobian \(J_\ell\) and the chosen metric \(M\), and summarizes how the local transformation changes squared representation magnitudes across different directions.

The approximation in Eq.~\eqref{eq:app_growth_derivation} can also be expressed with an explicit remainder. Substituting Eq.~\eqref{eq:app_local_dynamics} into the exact squared-magnitude difference gives
\begin{equation}
\Delta_\ell^M(\delta z_\ell)
=
\delta z_\ell^\top\mathcal{G}_\ell^M\delta z_\ell
+
\mathcal{O}(\|\delta z_\ell\|_2^3).
\label{eq:app_growth_remainder}
\end{equation}
Indeed, expanding \(\|J_\ell\delta z_\ell+R_\ell(\delta z_\ell)\|_M^2\) produces a cross term of order \(\mathcal{O}(\|\delta z_\ell\|_2^3)\) and a squared-remainder term of order \(\mathcal{O}(\|\delta z_\ell\|_2^4)\). Consequently, the quadratic form induced by \(\mathcal{G}_\ell^M\) gives the leading-order change in squared representation magnitude. It characterizes the local dynamics without asserting that the same growth behavior necessarily persists under arbitrarily large perturbations.

\paragraph{Directional amplification and contraction.}
For a nonzero perturbation direction \(v\), define its local amplification factor as \(a_\ell^M(v):=\|J_\ell v\|_M/\|v\|_M\). By construction, the growth operator satisfies the exact identity
\begin{equation}
v^\top\mathcal{G}_\ell^Mv
=
\|J_\ell v\|_M^2-\|v\|_M^2
=
\bigl((a_\ell^M(v))^2-1\bigr)\|v\|_M^2.
\label{eq:app_growth_identity}
\end{equation}
Therefore, \(v^\top\mathcal{G}_\ell^Mv>0\) indicates local amplification, \(v^\top\mathcal{G}_\ell^Mv<0\) indicates local contraction, and equality indicates preservation of squared magnitude at the linearized level. Because \(M\) is symmetric, the growth operator is also symmetric, satisfying \((\mathcal{G}_\ell^M)^\top=\mathcal{G}_\ell^M\). Its eigenvalues are therefore real, and the signs of its quadratic form provide a direct characterization of directional growth.

In particular, if \(\mathcal{G}_\ell^M\) has both positive and negative eigenvalues, there exist nonzero directions \(v_+\) and \(v_-\) such that \(v_+^\top\mathcal{G}_\ell^Mv_+>0\) and \(v_-^\top\mathcal{G}_\ell^Mv_-<0\). The same block then locally amplifies some representation differences while suppressing others. We refer to this coexistence of expanding and contracting directions as \emph{anisotropic representation dynamics}.

\paragraph{Connection to singular values.}
When \(M=I\), the growth operator becomes \(\mathcal{G}_\ell^I=J_\ell^\top J_\ell-I\). The matrix \(J_\ell^\top J_\ell\) is symmetric positive semidefinite, and its eigenvectors are the right singular vectors of \(J_\ell\), with eigenvalues equal to the squared singular values. Hence, if \(v_k\) is a right singular vector associated with singular value \(\sigma_k(J_\ell)\), then
\begin{equation}
v_k^\top\mathcal{G}_\ell^Iv_k
=
\bigl(\sigma_k(J_\ell)^2-1\bigr)\|v_k\|_2^2.
\label{eq:app_growth_singular_value}
\end{equation}
A singular value greater than one therefore identifies a locally expanding direction, while a singular value smaller than one identifies a contracting direction. The growth operator expresses these behaviors through a single symmetric matrix, with zero separating expansion from contraction.

\paragraph{A simple geometric example.}
Consider a two-dimensional local transformation with \(M=I\) and \(J_\ell=\operatorname{diag}(2,1/2)\). The transformation doubles perturbations along the first coordinate direction while halving those along the second. Its growth operator is
\[
\mathcal{G}_\ell^I
=
J_\ell^\top J_\ell-I
=
\begin{pmatrix}
3&0\\
0&-\frac34
\end{pmatrix}.
\]
For \(v_1=(1,0)^\top\), the squared magnitude changes from \(1\) to \(4\), giving \(v_1^\top\mathcal{G}_\ell^Iv_1=3>0\). For \(v_2=(0,1)^\top\), it changes from \(1\) to \(1/4\), giving \(v_2^\top\mathcal{G}_\ell^Iv_2=-3/4<0\). Thus, a single local transformation can simultaneously amplify and contract different directions, and the growth operator captures both behaviors through the signs of its quadratic form.

The growth operator provides a general description of these directional dynamics, but does not yet explain which components of self-attention produce them. In the following sections, we decompose the attention Jacobian into value transport, query gating, and key gating, and examine how their interaction determines the sign of directional growth.

\subsection{Self-Attention and the Transport--Gating Decomposition}
\label{subsec:app_transport_proof}

We begin with a general state-dependent mixture, which will later be specialized to self-attention. Let \(x\in\mathbb{R}^{p}\) denote the state on which the mixture depends, and consider
\(F(x)=\sum_{i=1}^{n}\alpha_i(x)f_i(x)\), where \(f_i(x)\in\mathbb{R}^{q}\) denotes the information contributed by candidate \(i\). The mixture weight \(\alpha_i(x)\) is obtained from a differentiable scalar score \(s_i(x)\in\mathbb{R}\) through the softmax
\[
\alpha_i(x)
=
\frac{\exp(s_i(x))}
{\sum_{j=1}^{n}\exp(s_j(x))}.
\]
A perturbation of \(x\) can therefore affect \(F(x)\) in two different ways: it can change the candidate information \(f_i(x)\), or it can change the weights \(\alpha_i(x)\) used to combine the candidates. Let \(J_F:=\partial F/\partial x\in\mathbb{R}^{q\times p}\) and \(J_{f_i}:=\partial f_i/\partial x\in\mathbb{R}^{q\times p}\). Differentiating the mixture output term by term gives
\begin{equation}
\begin{aligned}
dF
&=
d\left(
\sum_{i=1}^{n}
\alpha_i f_i
\right) \\
&=
\sum_{i=1}^{n}
d(\alpha_i f_i) \\
&=
\sum_{i=1}^{n}
\left(
\alpha_i\,df_i
+
f_i\,d\alpha_i
\right) \\
&=
\left(
\sum_{i=1}^{n}
\alpha_i J_{f_i}
\right)dx
+
\sum_{i=1}^{n}
f_i(\nabla\alpha_i)^\top dx .
\end{aligned}
\label{eq:app_mixture_differential}
\end{equation}
Since \(dF=J_Fdx\), Eq.~\eqref{eq:app_mixture_differential} implies
\begin{equation}
J_F
=
\sum_{i=1}^{n}
\alpha_iJ_{f_i}
+
\sum_{i=1}^{n}
f_i(\nabla\alpha_i)^\top .
\label{eq:app_mixture_product_rule}
\end{equation}

The first term in Eq.~\eqref{eq:app_mixture_product_rule} describes how the candidate information changes while the current mixture weights are held fixed. We therefore define the \emph{transport Jacobian} as
\begin{equation}
J_{\mathrm{transport}}
:=
\sum_{i=1}^{n}
\alpha_iJ_{f_i}.
\label{eq:app_transport_jacobian}
\end{equation}

It remains to simplify the second term, which is caused entirely by the state dependence of the mixture weights. To differentiate the softmax, define \(Z(x):=\sum_{j=1}^{n}\exp(s_j(x))\), so that \(\alpha_i=\exp(s_i)/Z\). Applying the quotient rule gives
\begin{equation}
\begin{aligned}
\nabla\alpha_i
&=
\nabla
\left(
\frac{\exp(s_i)}{Z}
\right) \\
&=
\frac{
Z\,\exp(s_i)\nabla s_i
-
\exp(s_i)\nabla Z
}{
Z^2
}.
\end{aligned}
\label{eq:app_softmax_quotient}
\end{equation}
The gradient of the normalizing factor is
\begin{equation}
\begin{aligned}
\nabla Z
&=
\nabla
\left(
\sum_{j=1}^{n}\exp(s_j)
\right) \\
&=
\sum_{j=1}^{n}
\exp(s_j)\nabla s_j.
\end{aligned}
\label{eq:app_softmax_partition_gradient}
\end{equation}
Substituting Eq.~\eqref{eq:app_softmax_partition_gradient} into Eq.~\eqref{eq:app_softmax_quotient} and using \(\alpha_i=\exp(s_i)/Z\) yields
\begin{equation}
\begin{aligned}
\nabla\alpha_i
&=
\frac{\exp(s_i)}{Z}
\nabla s_i
-
\frac{\exp(s_i)}{Z}
\sum_{j=1}^{n}
\frac{\exp(s_j)}{Z}
\nabla s_j \\
&=
\alpha_i\nabla s_i
-
\alpha_i
\sum_{j=1}^{n}
\alpha_j\nabla s_j \\
&=
\alpha_i
\left(
\nabla s_i-\bar g
\right),
\end{aligned}
\label{eq:app_softmax_gradient}
\end{equation}
where we define
\[
\bar g
:=
\sum_{j=1}^{n}\alpha_j\nabla s_j
\]
as the mixture-weighted mean score gradient. Thus, the change of a softmax weight is determined not by the absolute score response \(\nabla s_i\), but by its response relative to the weighted average response \(\bar g\).

Substituting Eq.~\eqref{eq:app_softmax_gradient} into the weight-dependent term of Eq.~\eqref{eq:app_mixture_product_rule} gives
\begin{equation}
\begin{aligned}
\sum_{i=1}^{n}
f_i(\nabla\alpha_i)^\top
&=
\sum_{i=1}^{n}
\alpha_i
f_i
(\nabla s_i-\bar g)^\top \\
&=
\sum_{i=1}^{n}
\alpha_if_i(\nabla s_i)^\top
-
\sum_{i=1}^{n}
\alpha_if_i\bar g^\top \\
&=
\sum_{i=1}^{n}
\alpha_if_i(\nabla s_i)^\top
-
\left(
\sum_{i=1}^{n}\alpha_if_i
\right)\bar g^\top \\
&=
\sum_{i=1}^{n}
\alpha_if_i(\nabla s_i)^\top
-
F\bar g^\top.
\end{aligned}
\label{eq:app_gating_expansion}
\end{equation}
Using \(\bar g=\sum_i\alpha_i\nabla s_i\), the last term can itself be written as
\begin{equation}
\begin{aligned}
F\bar g^\top
&=
F
\left(
\sum_{i=1}^{n}
\alpha_i\nabla s_i
\right)^\top \\
&=
\sum_{i=1}^{n}
\alpha_i
F(\nabla s_i)^\top.
\end{aligned}
\label{eq:app_mean_score_expansion}
\end{equation}
Combining Eqs.~\eqref{eq:app_gating_expansion} and \eqref{eq:app_mean_score_expansion} therefore gives
\begin{equation}
\begin{aligned}
\sum_{i=1}^{n}
f_i(\nabla\alpha_i)^\top
&=
\sum_{i=1}^{n}
\alpha_i
f_i(\nabla s_i)^\top
-
\sum_{i=1}^{n}
\alpha_i
F(\nabla s_i)^\top \\
&=
\sum_{i=1}^{n}
\alpha_i
(f_i-F)
(\nabla s_i)^\top.
\end{aligned}
\label{eq:app_centered_candidate_gating}
\end{equation}

The same quantity can also be centered in score-gradient space. Since the mixture output satisfies \(F=\sum_i\alpha_if_i\), we have
\begin{equation}
\begin{aligned}
\sum_{i=1}^{n}
\alpha_i(f_i-F)
&=
\sum_{i=1}^{n}\alpha_if_i
-
F\sum_{i=1}^{n}\alpha_i \\
&=
F-F \\
&=
0.
\end{aligned}
\label{eq:app_centered_candidate_zero}
\end{equation}
Consequently,
\begin{equation}
\begin{aligned}
\sum_{i=1}^{n}
\alpha_i
(f_i-F)
(\nabla s_i)^\top
&=
\sum_{i=1}^{n}
\alpha_i
(f_i-F)
(\nabla s_i)^\top \\
&\quad-
\left[
\sum_{i=1}^{n}
\alpha_i(f_i-F)
\right]
\bar g^\top \\
&=
\sum_{i=1}^{n}
\alpha_i
(f_i-F)
(\nabla s_i-\bar g)^\top.
\end{aligned}
\label{eq:app_centered_score_gating}
\end{equation}

We therefore define the \emph{gating Jacobian} as
\begin{equation}
J_{\mathrm{gating}}
:=
\sum_{i=1}^{n}
\alpha_i
(f_i-F)
(\nabla s_i-\bar g)^\top.
\label{eq:app_gating_jacobian}
\end{equation}
Substituting Eqs.~\eqref{eq:app_transport_jacobian} and \eqref{eq:app_gating_jacobian} into Eq.~\eqref{eq:app_mixture_product_rule} gives the transport--gating decomposition
\begin{equation}
\boxed{
J_{\mathrm{full}}
=
J_{\mathrm{transport}}
+
J_{\mathrm{gating}}.
}
\label{eq:app_transport_gating_decomposition}
\end{equation}

The two Jacobians correspond to different computational effects. \(J_{\mathrm{transport}}\) measures how the information carried by the candidates changes while their current routing weights are fixed. In contrast, \(J_{\mathrm{gating}}\) measures how the output changes because the state perturbation redistributes those routing weights. Its centered form is particularly informative: \(f_i-F\) measures how candidate \(i\) differs from the current mixture output, whereas \(\nabla s_i-\bar g\) measures how the routing preference for that candidate changes relative to the average routing response. The gating effect therefore depends jointly on \emph{where different candidate information lies} and \emph{how the perturbation changes routing among those candidates}.

Self-attention is a special case of this state-dependent mixture. For a fixed query token \(r\), recall from Eq.~\eqref{eq:app_attention_setup} that \(F_r=\sum_i\alpha_{ri}W_{VO}x_i\). Hence the candidate information is \(f_i=W_{VO}x_i\), while the weights \(\alpha_{ri}\) are determined by query--key scores. The generic decomposition above therefore separates representation changes propagated through the value pathway from changes caused by attention reallocation. The next subsection further resolves this gating term into the exact query and key pathways of self-attention.

\subsection{Exact Query, Key, and Value Jacobian Decomposition}
\label{app:attention_jacobian}

We now specialize the transport--gating decomposition to dot-product self-attention and derive the exact Jacobian with respect to an individual token representation. Recall that we fix one Transformer layer and one attention head, suppressing the corresponding layer and head indices. Let \(r\) denote the query token, \(i\) a candidate token, and \(m\) the token whose representation is perturbed. The attention score and weight are
\(s_{ri}=x_r^\top W_{QK}x_i\) and
\(\alpha_{ri}=\exp(s_{ri})/\sum_j\exp(s_{rj})\), respectively, and the attention output is \(F_r=\sum_i\alpha_{ri}W_{VO}x_i\).

Our goal is to compute
\(J_{rm}:=\partial F_r/\partial x_m\), which measures how a perturbation of token \(m\) changes the attention output associated with query token \(r\). Consider an infinitesimal perturbation \(dx_m\). Differentiating the attention output gives
\begin{equation}
\begin{aligned}
dF_r
&=
d\left(
\sum_{i=1}^{n}
\alpha_{ri}W_{VO}x_i
\right) \\
&=
\sum_{i=1}^{n}
d\alpha_{ri}\,W_{VO}x_i
+
\sum_{i=1}^{n}
\alpha_{ri}W_{VO}\,dx_i.
\end{aligned}
\label{eq:app_attention_output_differential}
\end{equation}
Because only \(x_m\) is perturbed, \(dx_i=0\) for \(i\neq m\) and \(dx_m\) is the only nonzero token differential. Hence the second term becomes
\begin{equation}
\begin{aligned}
\sum_{i=1}^{n}
\alpha_{ri}W_{VO}\,dx_i
&=
\alpha_{rm}W_{VO}\,dx_m.
\end{aligned}
\label{eq:app_value_transport_differential}
\end{equation}
This is the direct \emph{value transport} contribution. The remaining term arises because perturbing \(x_m\) may also change the attention weights. To compute this routing contribution, we first differentiate the score
\(s_{ri}=x_r^\top W_{QK}x_i\). The token \(x_m\) can affect this score in two different roles: it may be the query when \(m=r\), and it may be the key associated with candidate \(i\) when \(m=i\). To make these two pathways explicit, recall that
\(W_{QK}=W_QW_K^\top/\sqrt{d_k}\), so that
\[
s_{ri}
=
x_r^\top W_{QK}x_i
=
\frac{Q_r^\top k_i}{\sqrt{d_k}},
\]
where \(Q_r:=W_Q^\top x_r\) is the query vector produced by token \(r\), and \(k_i:=W_K^\top x_i\) is the key vector produced by token \(i\). Hence \(x_r\) affects the score through the query \(Q_r\), while \(x_i\) affects it through the key \(k_i\). Differentiating the score therefore gives
\begin{equation}
\begin{aligned}
ds_{ri}
&=
d\left(
x_r^\top W_{QK}x_i
\right) \\
&=
(dx_r)^\top W_{QK}x_i
+
x_r^\top W_{QK}\,dx_i.
\end{aligned}
\label{eq:app_attention_score_two_paths}
\end{equation}
The first term in Eq.~\eqref{eq:app_attention_score_two_paths} is produced by perturbing the representation \(x_r\) that generates the query. Indeed,
\begin{equation}
\begin{aligned}
\frac{1}{\sqrt{d_k}}(dQ_r)^\top k_i
&=
\frac{1}{\sqrt{d_k}}
(W_Q^\top dx_r)^\top W_K^\top x_i \\
&=
(dx_r)^\top
\frac{W_QW_K^\top}{\sqrt{d_k}}
x_i \\
&=
(dx_r)^\top W_{QK}x_i.
\end{aligned}
\label{eq:app_query_score_path}
\end{equation}
We therefore refer to this term as the \emph{query pathway}. A perturbation of the query token changes \(Q_r\), and consequently can change its compatibility with every candidate key \(k_i\). The second term is produced by perturbing the representation \(x_i\) that generates the key:
\begin{equation}
\begin{aligned}
\frac{1}{\sqrt{d_k}}Q_r^\top dk_i
&=
\frac{1}{\sqrt{d_k}}
(W_Q^\top x_r)^\top W_K^\top dx_i \\
&=
x_r^\top
\frac{W_QW_K^\top}{\sqrt{d_k}}
dx_i \\
&=
x_r^\top W_{QK}\,dx_i.
\end{aligned}
\label{eq:app_key_score_path}
\end{equation}
We refer to this term as the \emph{key pathway}. In this case, the query representation is fixed, while the perturbation changes the key of the candidate token and therefore changes how strongly that candidate matches the query.

Because our Jacobian \(J_{rm}\) considers a perturbation only to token \(m\), we have
\(dx_r=\mathbf{1}_{\{m=r\}}dx_m\) and
\(dx_i=\mathbf{1}_{\{i=m\}}dx_m\). Substituting these relations into Eq.~\eqref{eq:app_attention_score_two_paths} gives
\begin{equation}
\begin{aligned}
ds_{ri}
&=
\mathbf{1}_{\{m=r\}}
(dx_m)^\top W_{QK}x_i
+
\mathbf{1}_{\{i=m\}}
x_r^\top W_{QK}\,dx_m \\
&=
\left[
\mathbf{1}_{\{m=r\}}
W_{QK}x_i
+
\mathbf{1}_{\{i=m\}}
W_{QK}^\top x_r
\right]^\top
dx_m.
\end{aligned}
\label{eq:app_attention_score_differential}
\end{equation}
Here, \(\mathbf{1}_{\{\cdot\}}\) denotes the indicator function. The factor
\(\mathbf{1}_{\{m=r\}}\) activates the query pathway only when the perturbed token \(m\) is the query token itself, whereas
\(\mathbf{1}_{\{i=m\}}\) activates the key pathway only when the perturbed token is the candidate whose score is being evaluated. These two cases are not mutually exclusive: if \(m=r=i\), the same token generates both the query and the key, and both contributions to \(ds_{ri}\) are present.

The softmax differential derived in Section~\ref{subsec:app_transport_proof} gives
\begin{equation}
\begin{aligned}
d\alpha_{ri}
&=
\alpha_{ri}
\left(
ds_{ri}
-
\sum_{j=1}^{n}
\alpha_{rj}ds_{rj}
\right).
\end{aligned}
\label{eq:app_attention_weight_differential}
\end{equation}
We next evaluate the query and key contributions separately.

\paragraph{Query contribution.}
Suppose first that the perturbation acts through the query role. From Eq.~\eqref{eq:app_attention_score_differential}, the query-induced score differential is
\begin{equation}
\begin{aligned}
ds_{ri}^{(q)}
&=
\mathbf{1}_{\{m=r\}}
(dx_m)^\top W_{QK}x_i \\
&=
\mathbf{1}_{\{m=r\}}
x_i^\top W_{QK}^\top dx_m.
\end{aligned}
\label{eq:app_query_score_differential}
\end{equation}
Its attention-weighted average is
\begin{equation}
\begin{aligned}
\sum_{j=1}^{n}
\alpha_{rj}ds_{rj}^{(q)}
&=
\mathbf{1}_{\{m=r\}}
\sum_{j=1}^{n}
\alpha_{rj}
x_j^\top W_{QK}^\top dx_m \\
&=
\mathbf{1}_{\{m=r\}}
\left(
\sum_{j=1}^{n}
\alpha_{rj}x_j
\right)^\top
W_{QK}^\top dx_m \\
&=
\mathbf{1}_{\{m=r\}}
\mu_r^\top W_{QK}^\top dx_m,
\end{aligned}
\label{eq:app_query_mean_score_differential}
\end{equation}
where, recall, \(\mu_r:=\sum_j\alpha_{rj}x_j\) is the attention-weighted mean representation. Substituting Eqs.~\eqref{eq:app_query_score_differential} and \eqref{eq:app_query_mean_score_differential} into the softmax differential gives
\begin{equation}
\begin{aligned}
d\alpha_{ri}^{(q)}
&=
\alpha_{ri}
\left(
ds_{ri}^{(q)}
-
\sum_{j=1}^{n}
\alpha_{rj}ds_{rj}^{(q)}
\right) \\
&=
\mathbf{1}_{\{m=r\}}
\alpha_{ri}
\left(
x_i-\mu_r
\right)^\top
W_{QK}^\top dx_m.
\end{aligned}
\label{eq:app_query_attention_weight_differential}
\end{equation}

The corresponding contribution to the attention output is therefore
\begin{equation}
\begin{aligned}
dF_r^{(q)}
&=
\sum_{i=1}^{n}
W_{VO}x_i\,
d\alpha_{ri}^{(q)} \\
&=
\mathbf{1}_{\{m=r\}}
\sum_{i=1}^{n}
\alpha_{ri}
W_{VO}x_i
(x_i-\mu_r)^\top
W_{QK}^\top dx_m \\
&=
\mathbf{1}_{\{m=r\}}
W_{VO}
\left[
\sum_{i=1}^{n}
\alpha_{ri}
x_i(x_i-\mu_r)^\top
\right]
W_{QK}^\top dx_m.
\end{aligned}
\label{eq:app_query_gating_before_covariance}
\end{equation}

The matrix inside brackets is exactly the attention-weighted covariance. To see this, write \(x_i=(x_i-\mu_r)+\mu_r\). Then
\begin{equation}
\begin{aligned}
\sum_{i=1}^{n}
\alpha_{ri}
x_i(x_i-\mu_r)^\top
&=
\sum_{i=1}^{n}
\alpha_{ri}
\left[
(x_i-\mu_r)+\mu_r
\right]
(x_i-\mu_r)^\top \\
&=
\sum_{i=1}^{n}
\alpha_{ri}
(x_i-\mu_r)(x_i-\mu_r)^\top \\
&\quad+
\mu_r
\left[
\sum_{i=1}^{n}
\alpha_{ri}(x_i-\mu_r)
\right]^\top.
\end{aligned}
\label{eq:app_query_covariance_expansion}
\end{equation}
The second term vanishes because
\begin{equation}
\begin{aligned}
\sum_{i=1}^{n}
\alpha_{ri}(x_i-\mu_r)
&=
\sum_{i=1}^{n}
\alpha_{ri}x_i
-
\mu_r
\sum_{i=1}^{n}
\alpha_{ri} \\
&=
\mu_r-\mu_r \\
&=
0.
\end{aligned}
\label{eq:app_centered_attention_mean_zero}
\end{equation}
Hence, recalling the definition
\(\Sigma_r:=\sum_i\alpha_{ri}(x_i-\mu_r)(x_i-\mu_r)^\top\), Eq.~\eqref{eq:app_query_gating_before_covariance} reduces to
\begin{equation}
dF_r^{(q)}
=
\mathbf{1}_{\{m=r\}}
W_{VO}\Sigma_rW_{QK}^\top dx_m.
\label{eq:app_query_gating_differential}
\end{equation}
The exact query-gating Jacobian is therefore
\begin{equation}
\boxed{
J_{rm}^{(q)}
=
\mathbf{1}_{\{m=r\}}
W_{VO}\Sigma_rW_{QK}^\top.
}
\label{eq:app_query_gating_jacobian}
\end{equation}

\paragraph{Key contribution.}
We now isolate the effect of perturbing token \(m\) through its role as a key. From Eq.~\eqref{eq:app_attention_score_differential},
\begin{equation}
\begin{aligned}
ds_{ri}^{(k)}
&=
\mathbf{1}_{\{i=m\}}
x_r^\top W_{QK}\,dx_m \\
&=
\mathbf{1}_{\{i=m\}}
(W_{QK}^\top x_r)^\top dx_m.
\end{aligned}
\label{eq:app_key_score_differential}
\end{equation}
Taking the attention-weighted average over candidate tokens gives
\begin{equation}
\begin{aligned}
\sum_{j=1}^{n}
\alpha_{rj}ds_{rj}^{(k)}
&=
\sum_{j=1}^{n}
\alpha_{rj}
\mathbf{1}_{\{j=m\}}
(W_{QK}^\top x_r)^\top dx_m \\
&=
\alpha_{rm}
(W_{QK}^\top x_r)^\top dx_m.
\end{aligned}
\label{eq:app_key_mean_score_differential}
\end{equation}
Substituting into the softmax differential yields
\begin{equation}
\begin{aligned}
d\alpha_{ri}^{(k)}
&=
\alpha_{ri}
\left(
ds_{ri}^{(k)}
-
\sum_{j=1}^{n}
\alpha_{rj}ds_{rj}^{(k)}
\right) \\
&=
\alpha_{ri}
\left(
\mathbf{1}_{\{i=m\}}-\alpha_{rm}
\right)
(W_{QK}^\top x_r)^\top dx_m.
\end{aligned}
\label{eq:app_key_attention_weight_differential}
\end{equation}

The resulting change in the attention output is
\begin{equation}
\begin{aligned}
dF_r^{(k)}
&=
\sum_{i=1}^{n}
W_{VO}x_i\,
d\alpha_{ri}^{(k)} \\
&=
\sum_{i=1}^{n}
\alpha_{ri}
W_{VO}x_i
\left(
\mathbf{1}_{\{i=m\}}-\alpha_{rm}
\right)
(W_{QK}^\top x_r)^\top dx_m \\
&=
\left[
\sum_{i=1}^{n}
\alpha_{ri}
W_{VO}x_i
\mathbf{1}_{\{i=m\}}
-
\alpha_{rm}
\sum_{i=1}^{n}
\alpha_{ri}W_{VO}x_i
\right]
(W_{QK}^\top x_r)^\top dx_m \\
&=
\left[
\alpha_{rm}W_{VO}x_m
-
\alpha_{rm}F_r
\right]
(W_{QK}^\top x_r)^\top dx_m.
\end{aligned}
\label{eq:app_key_gating_expansion}
\end{equation}
Since \(F_r=\sum_i\alpha_{ri}W_{VO}x_i=W_{VO}\mu_r\), we obtain
\begin{equation}
\begin{aligned}
dF_r^{(k)}
&=
\alpha_{rm}
W_{VO}(x_m-\mu_r)
(W_{QK}^\top x_r)^\top dx_m.
\end{aligned}
\label{eq:app_key_gating_differential}
\end{equation}
Hence the exact key-gating Jacobian is
\begin{equation}
\boxed{
J_{rm}^{(k)}
=
\alpha_{rm}
W_{VO}(x_m-\mu_r)
(W_{QK}^\top x_r)^\top.
}
\label{eq:app_key_gating_jacobian}
\end{equation}

\paragraph{Complete decomposition.}
Combining the value contribution in Eq.~\eqref{eq:app_value_transport_differential}, the query-gating contribution in Eq.~\eqref{eq:app_query_gating_differential}, and the key-gating contribution in Eq.~\eqref{eq:app_key_gating_differential}, we have
\begin{equation}
\begin{aligned}
dF_r
&=
dF_r^{(v)}
+
dF_r^{(q)}
+
dF_r^{(k)} \\
&=
\Big[
\alpha_{rm}W_{VO}
+
\mathbf{1}_{\{m=r\}}
W_{VO}\Sigma_rW_{QK}^\top \\
&\qquad\qquad
+
\alpha_{rm}
W_{VO}(x_m-\mu_r)
(W_{QK}^\top x_r)^\top
\Big]dx_m.
\end{aligned}
\label{eq:app_full_attention_differential}
\end{equation}
Therefore,
\begin{equation}
\boxed{
J_{rm}
=
\underbrace{\alpha_{rm}W_{VO}}_{\text{value transport}}
+
\underbrace{\mathbf{1}_{\{m=r\}}W_{VO}\Sigma_rW_{QK}^{\top}}_{\text{query gating}}
+
\underbrace{
\alpha_{rm}W_{VO}(x_m-\mu_r)(W_{QK}^{\top}x_r)^\top
}_{\text{key gating}}.
}
\label{eq:app_full_qkv_jacobian}
\end{equation}

This decomposition is exact for the single-head dot-product attention map considered here. The first term propagates a perturbation directly through the value pathway. The second arises because perturbing the query changes its relative preference over all candidate tokens, and its dependence on \(\Sigma_r\) shows that this effect is mediated by variation among the currently attended representations. The third arises because perturbing token \(m\) changes its compatibility with the fixed query and therefore changes how strongly that token is selected as a key. When \(m=r\), the same representation participates simultaneously as value, query, and key, and all three terms in Eq.~\eqref{eq:app_full_qkv_jacobian} contribute.

Equation~\eqref{eq:app_full_qkv_jacobian} refines the generic decomposition of Section~\ref{subsec:app_transport_proof}: the value term is the transport contribution, whereas the query and key terms together form the state-dependent gating contribution. In the next subsection, we study the geometry of the query-gating operator \(W_{VO}\Sigma_rW_{QK}^\top\), which makes explicit the representation subspace in which changes to attention routing can act.

\subsection{Geometry and Support of Query Gating}
\label{subsec:app_gating_geometry}

We now examine how the geometry of attended representations constrains the query-gating response and provide the complete proof of Proposition~\ref{prop:query_gating_support}. For a fixed query token \(r\), recall that the query-gating component of the attention Jacobian is
\begin{equation}
J_r^{(q)}
:=
W_{VO}\Sigma_rW_{QK}^{\top},
\label{eq:app_query_gating_operator}
\end{equation}
where the indicator \(\mathbf{1}_{\{m=r\}}\) is omitted because we consider perturbations of the query token itself. Here, \(\Sigma_r\) denotes the attention-weighted covariance of the token representations,
\begin{equation}
\Sigma_r
=
\sum_{i=1}^{n}
\alpha_{ri}(x_i-\mu_r)(x_i-\mu_r)^\top,
\qquad
\mu_r=\sum_{i=1}^{n}\alpha_{ri}x_i.
\label{eq:app_query_covariance}
\end{equation}
The attention weights satisfy \(\alpha_{ri}\geq0\) and \(\sum_i\alpha_{ri}=1\).

The factorization in Eq.~\eqref{eq:app_query_gating_operator} suggests that query gating is governed by more than the query-key transformation alone. Although \(W_{QK}^{\top}\) determines how a query perturbation changes attention scores, the resulting change in the output also depends on whether the attended representations differ from one another. The covariance \(\Sigma_r\) captures precisely this variation, while \(W_{VO}\) maps the resulting redistribution into the output representation space. Our objective is therefore to identify the subspace on which \(\Sigma_r\) can act and determine how this subspace constrains the query-gating Jacobian.

\paragraph{Geometric support of query gating.}
Let \(\mathcal A_r:=\{i:\alpha_{ri}>0\}\) denote the active attention set, which also accommodates attention masks that assign zero weight to unavailable tokens. Define the \emph{attended variation subspace} as
\begin{equation}
\mathcal U_r
:=
\operatorname{span}
\left\{
x_i-\mu_r:i\in\mathcal A_r
\right\}
\subseteq\mathbb R^d.
\label{eq:app_attended_variation_subspace}
\end{equation}
This subspace contains all linear combinations of deviations of the attended representations from their attention-weighted mean. We show that the covariance acts only within this subspace and that its exact support determines the possible output directions, rank, and null directions of query gating.

\medskip
\noindent\textbf{Proposition~\ref{prop:query_gating_support} (restated: Geometric Support of Query Gating).}
Let \(\mathcal U_r\) be the attended variation subspace and \(J_r^{(q)}=W_{VO}\Sigma_rW_{QK}^{\top}\) the query-gating Jacobian. Then
\begin{equation}
\begin{aligned}
\operatorname{Range}(J_r^{(q)})
&\subseteq W_{VO}\mathcal U_r,\\
\operatorname{rank}(J_r^{(q)})
&\leq\dim(\mathcal U_r),\\
W_{QK}^{\top}v\perp\mathcal U_r
&\Longrightarrow J_r^{(q)}v=0.
\end{aligned}
\label{eq:app_query_gating_support}
\end{equation}

\begin{proof}
The central observation is that the range of the attention-weighted covariance is exactly the attended variation subspace. Once this relationship is established, the three claims follow directly from the factorization of the query-gating Jacobian.

To see why the covariance is restricted to \(\mathcal U_r\), consider its action on an arbitrary vector \(v\in\mathbb R^d\). Expanding Eq.~\eqref{eq:app_query_covariance} gives
\begin{equation}
\begin{aligned}
\Sigma_rv
&=
\left[
\sum_{i\in\mathcal A_r}
\alpha_{ri}(x_i-\mu_r)(x_i-\mu_r)^\top
\right]v\\
&=
\sum_{i\in\mathcal A_r}
\alpha_{ri}(x_i-\mu_r)
\left[(x_i-\mu_r)^\top v\right]\\
&=
\sum_{i\in\mathcal A_r}
\underbrace{
\alpha_{ri}(x_i-\mu_r)^\top v
}_{\text{scalar coefficient}}
(x_i-\mu_r).
\end{aligned}
\label{eq:app_covariance_action}
\end{equation}
The final expression is a linear combination of the centered attended representations \(x_i-\mu_r\). By the definition of their span, this implies
\begin{equation}
\operatorname{Range}(\Sigma_r)
\subseteq\mathcal U_r.
\label{eq:app_covariance_range_inclusion}
\end{equation}

The inclusion shows that the covariance cannot produce components outside \(\mathcal U_r\), but it does not yet establish whether every direction in \(\mathcal U_r\) belongs to its range. To prove equality, we characterize the null space of \(\Sigma_r\). Since \(\Sigma_r\) is a sum of symmetric rank-one matrices with nonnegative weights, it is symmetric positive semidefinite. Its quadratic form satisfies
\begin{equation}
\begin{aligned}
v^\top\Sigma_rv
&=
v^\top
\left[
\sum_{i\in\mathcal A_r}
\alpha_{ri}(x_i-\mu_r)(x_i-\mu_r)^\top
\right]v\\
&=
\sum_{i\in\mathcal A_r}
\alpha_{ri}
v^\top(x_i-\mu_r)(x_i-\mu_r)^\top v\\
&=
\sum_{i\in\mathcal A_r}
\alpha_{ri}
\left[v^\top(x_i-\mu_r)\right]
\left[(x_i-\mu_r)^\top v\right]\\
&=
\sum_{i\in\mathcal A_r}
\alpha_{ri}
\left[v^\top(x_i-\mu_r)\right]^2.
\end{aligned}
\label{eq:app_covariance_quadratic_form}
\end{equation}

Every summand is nonnegative, and \(\alpha_{ri}>0\) for all \(i\in\mathcal A_r\). Consequently, the quadratic form vanishes if and only if every inner product \(v^\top(x_i-\mu_r)\) vanishes. Furthermore, because \(\Sigma_r\) is positive semidefinite, \(v^\top\Sigma_rv=0\) holds if and only if \(\Sigma_rv=0\). We therefore obtain
\begin{equation}
\begin{aligned}
v\in\operatorname{Null}(\Sigma_r)
&\iff \Sigma_rv=0\\
&\iff v^\top\Sigma_rv=0\\
&\iff v^\top(x_i-\mu_r)=0,
\quad\forall i\in\mathcal A_r\\
&\iff v\perp
\operatorname{span}
\{x_i-\mu_r:i\in\mathcal A_r\}\\
&\iff v\in\mathcal U_r^\perp.
\end{aligned}
\label{eq:app_covariance_nullspace}
\end{equation}
Thus, \(\operatorname{Null}(\Sigma_r)=\mathcal U_r^\perp\). Since \(\Sigma_r\) is symmetric, its range is the orthogonal complement of its null space. Taking orthogonal complements yields
\begin{equation}
\begin{aligned}
\operatorname{Range}(\Sigma_r)
&=\operatorname{Null}(\Sigma_r)^\perp\\
&=(\mathcal U_r^\perp)^\perp\\
&=\mathcal U_r.
\end{aligned}
\label{eq:app_covariance_exact_support}
\end{equation}
The last equality follows because \(\mathcal U_r\) is a finite-dimensional linear subspace and is therefore equal to its double orthogonal complement. This establishes the exact geometric support of the covariance: its range consists precisely of the directions spanned by variation among the attended representations.

We now return to the query-gating Jacobian. Since \(J_r^{(q)}=W_{VO}\Sigma_rW_{QK}^{\top}\), any output response first passes through \(\Sigma_r\) and then through \(W_{VO}\). In particular, for an arbitrary perturbation \(v\),
\begin{equation}
\begin{aligned}
J_r^{(q)}v
&=
W_{VO}\Sigma_rW_{QK}^{\top}v\\
&=
W_{VO}
\sum_{i\in\mathcal A_r}
\alpha_{ri}(x_i-\mu_r)
\left[(x_i-\mu_r)^\top W_{QK}^{\top}v\right]\\
&=
\sum_{i\in\mathcal A_r}
\alpha_{ri}
\left[(x_i-\mu_r)^\top W_{QK}^{\top}v\right]
W_{VO}(x_i-\mu_r).
\end{aligned}
\label{eq:app_query_gating_direction}
\end{equation}
The last expression shows that every query-gating response is a linear combination of the transformed centered representations \(W_{VO}(x_i-\mu_r)\). We can express the same observation through the range of the matrix product:
\begin{equation}
\begin{aligned}
\operatorname{Range}(J_r^{(q)})
&=
\operatorname{Range}
\left(W_{VO}\Sigma_rW_{QK}^{\top}\right)\\
&\subseteq
\operatorname{Range}
\left(W_{VO}\Sigma_r\right)\\
&=
W_{VO}\operatorname{Range}(\Sigma_r)\\
&=
W_{VO}\mathcal U_r.
\end{aligned}
\label{eq:app_query_gating_range}
\end{equation}
Here, the inclusion follows from the fact that right multiplication cannot enlarge the column space of a matrix, while the final equality follows from Eq.~\eqref{eq:app_covariance_exact_support}. Therefore, query gating cannot generate an arbitrary output perturbation: its response must remain within \(W_{VO}\mathcal U_r\).

This output-support restriction also limits the dimensionality of query gating. Since the rank of a matrix is the dimension of its range, and multiplication cannot increase rank beyond that of either factor, we have
\begin{equation}
\begin{aligned}
\operatorname{rank}(J_r^{(q)})
&=
\operatorname{rank}
\left(W_{VO}\Sigma_rW_{QK}^{\top}\right)\\
&\leq
\operatorname{rank}
\left(\Sigma_rW_{QK}^{\top}\right)\\
&\leq
\operatorname{rank}(\Sigma_r)\\
&=
\dim\left(\operatorname{Range}(\Sigma_r)\right)\\
&=
\dim(\mathcal U_r).
\end{aligned}
\label{eq:app_query_gating_rank}
\end{equation}
Thus, even if the ambient representation space has dimension \(d\), the rank of the query-gating Jacobian cannot exceed the dimension of the attended variation subspace.

It remains to identify perturbation directions that produce no query-gating response. The null-space characterization in Eq.~\eqref{eq:app_covariance_nullspace} provides this condition directly. Suppose that \(W_{QK}^{\top}v\perp\mathcal U_r\). Then
\begin{equation}
\begin{aligned}
W_{QK}^{\top}v\perp\mathcal U_r
&\iff
W_{QK}^{\top}v
\in\mathcal U_r^\perp\\
&\iff
W_{QK}^{\top}v
\in\operatorname{Null}(\Sigma_r)\\
&\Longrightarrow
\Sigma_rW_{QK}^{\top}v=0.
\end{aligned}
\end{equation}
Substituting this relation into the query-gating Jacobian yields
\begin{equation}
\begin{aligned}
J_r^{(q)}v
&=
W_{VO}\Sigma_rW_{QK}^{\top}v\\
&=
W_{VO}\left(\Sigma_rW_{QK}^{\top}v\right)\\
&=
W_{VO}0\\
&=0.
\end{aligned}
\label{eq:app_query_gating_null_direction}
\end{equation}
Equivalently, for any \(u\in\mathcal U_r\), the identity \(v^\top W_{QK}u=(W_{QK}^{\top}v)^\top u\) shows that \(W_{QK}^{\top}v\perp\mathcal U_r\) is equivalent to \(v\perp W_{QK}\mathcal U_r\). Hence perturbations orthogonal to the transformed variation subspace cannot activate query gating. This establishes the final claim and completes the proof.
\end{proof}

The proposition gives both an output-side and an input-side characterization of query gating. On the output side, the response is confined to \(W_{VO}\mathcal U_r\), with rank at most \(\dim(\mathcal U_r)\). On the input side, a perturbation cannot induce query gating when its transformation through \(W_{QK}^{\top}\) is orthogonal to the attended variation subspace. The latter condition is sufficient but not necessary: the response may also vanish when \(W_{VO}\) annihilates a nonzero covariance response. These distinctions are important because the exact covariance support does not imply that the full query-gating Jacobian has the same range or rank.

\paragraph{Pairwise form of the covariance.}
The preceding proof shows that the covariance is supported precisely on variation among attended representations. We now express this variation directly through pairwise differences, rather than deviations from the attention-weighted mean. Besides providing a geometric interpretation, this alternative representation will be useful for bounding the magnitude of query gating.

Starting from the covariance definition, expanding the centered outer product gives
\begin{equation}
\begin{aligned}
\Sigma_r
&=
\sum_i\alpha_{ri}
(x_i-\mu_r)(x_i-\mu_r)^\top\\
&=
\sum_i\alpha_{ri}
\left(
x_ix_i^\top-x_i\mu_r^\top
-\mu_rx_i^\top+\mu_r\mu_r^\top
\right)\\
&=
\sum_i\alpha_{ri}x_ix_i^\top
-\left(\sum_i\alpha_{ri}x_i\right)\mu_r^\top\\
&\quad
-\mu_r\left(\sum_i\alpha_{ri}x_i\right)^\top
+\left(\sum_i\alpha_{ri}\right)\mu_r\mu_r^\top\\
&=
\sum_i\alpha_{ri}x_ix_i^\top
-\mu_r\mu_r^\top-\mu_r\mu_r^\top
+\mu_r\mu_r^\top\\
&=
\sum_i\alpha_{ri}x_ix_i^\top-\mu_r\mu_r^\top.
\end{aligned}
\label{eq:app_covariance_second_moment}
\end{equation}
The fourth equality uses \(\sum_i\alpha_{ri}x_i=\mu_r\) and \(\sum_i\alpha_{ri}=1\). Thus, the covariance equals the attention-weighted second moment minus the outer product of the weighted mean.

To recover this expression from pairwise differences, consider
\[
S_r:=
\frac12\sum_{i,j}
\alpha_{ri}\alpha_{rj}
(x_i-x_j)(x_i-x_j)^\top.
\]
Expanding the outer product gives
\begin{equation}
\begin{aligned}
S_r
&=
\frac12\sum_{i,j}
\alpha_{ri}\alpha_{rj}
\left(
x_ix_i^\top-x_ix_j^\top
-x_jx_i^\top+x_jx_j^\top
\right)\\
&=
\frac12\sum_{i,j}
\alpha_{ri}\alpha_{rj}x_ix_i^\top
+\frac12\sum_{i,j}
\alpha_{ri}\alpha_{rj}x_jx_j^\top\\
&\quad
-\frac12\sum_{i,j}
\alpha_{ri}\alpha_{rj}x_ix_j^\top
-\frac12\sum_{i,j}
\alpha_{ri}\alpha_{rj}x_jx_i^\top.
\end{aligned}
\label{eq:app_pairwise_covariance_expansion}
\end{equation}

The first two sums contain terms depending on only one of the indices. Since the attention weights sum to one, they reduce to
\begin{equation}
\begin{aligned}
\sum_{i,j}\alpha_{ri}\alpha_{rj}x_ix_i^\top
&=
\sum_i\alpha_{ri}x_ix_i^\top
\left(\sum_j\alpha_{rj}\right)\\
&=
\sum_i\alpha_{ri}x_ix_i^\top,\\
\sum_{i,j}\alpha_{ri}\alpha_{rj}x_jx_j^\top
&=
\left(\sum_i\alpha_{ri}\right)
\sum_j\alpha_{rj}x_jx_j^\top\\
&=
\sum_j\alpha_{rj}x_jx_j^\top.
\end{aligned}
\label{eq:app_pairwise_second_moment_terms}
\end{equation}
The two cross terms instead factor into products of the attention-weighted means:
\begin{equation}
\begin{aligned}
\sum_{i,j}\alpha_{ri}\alpha_{rj}x_ix_j^\top
&=
\left(\sum_i\alpha_{ri}x_i\right)
\left(\sum_j\alpha_{rj}x_j\right)^\top\\
&=\mu_r\mu_r^\top,\\
\sum_{i,j}\alpha_{ri}\alpha_{rj}x_jx_i^\top
&=
\left(\sum_j\alpha_{rj}x_j\right)
\left(\sum_i\alpha_{ri}x_i\right)^\top\\
&=\mu_r\mu_r^\top.
\end{aligned}
\label{eq:app_pairwise_terms}
\end{equation}
Substituting these identities into Eq.~\eqref{eq:app_pairwise_covariance_expansion}, we obtain
\begin{equation}
\begin{aligned}
S_r
&=
\frac12
\left[
2\sum_i\alpha_{ri}x_ix_i^\top
-2\mu_r\mu_r^\top
\right]\\
&=
\sum_i\alpha_{ri}x_ix_i^\top
-\mu_r\mu_r^\top\\
&=\Sigma_r,
\end{aligned}
\label{eq:app_pairwise_covariance_derivation}
\end{equation}
where the last equality follows from Eq.~\eqref{eq:app_covariance_second_moment}. We therefore obtain the pairwise covariance identity
\begin{equation}
\boxed{
\Sigma_r
=
\frac12\sum_{i,j}
\alpha_{ri}\alpha_{rj}
(x_i-x_j)(x_i-x_j)^\top.
}
\label{eq:app_pairwise_covariance}
\end{equation}

Substituting this identity into the query-gating Jacobian gives
\begin{equation}
\begin{aligned}
J_r^{(q)}
&=
W_{VO}
\left[
\frac12\sum_{i,j}
\alpha_{ri}\alpha_{rj}
(x_i-x_j)(x_i-x_j)^\top
\right]W_{QK}^{\top}\\
&=
\frac12\sum_{i,j}
\alpha_{ri}\alpha_{rj}
W_{VO}(x_i-x_j)(x_i-x_j)^\top
W_{QK}^{\top}.
\end{aligned}
\label{eq:app_pairwise_query_gating}
\end{equation}

This representation makes the role of attended-representation diversity explicit. Each contribution depends on a pairwise difference \(x_i-x_j\), scaled by the corresponding attention weights. If all active attended representations coincide, every difference vanishes, and hence \(\Sigma_r=0\) and \(J_r^{(q)}=0\). The routing distribution may still be sensitive to the query in principle, but redistributing attention among identical representations cannot change the query-gating output. More generally, small pairwise differences constrain the strength of the response, as quantified next.

\paragraph{A geometric upper bound.}
The pairwise covariance identity allows us to relate the magnitude of query gating to two properties of the attended representations: their geometric diversity and the concentration of the attention distribution. Define
\[
\operatorname{diam}_r
:=
\max_{i,j\in\mathcal A_r}\|x_i-x_j\|_2,
\qquad
\kappa_r:=\sum_i\alpha_{ri}^2.
\]
The first quantity is the diameter of the active attended representations. The second measures attention concentration, with \(\kappa_r=1\) for one-hot attention and smaller values when attention is spread across multiple tokens.

To derive the bound, we start from Eq.~\eqref{eq:app_pairwise_covariance} and apply the triangle inequality for the spectral norm. Since every \(\alpha_{ri}\alpha_{rj}\) is nonnegative,
\begin{equation}
\begin{aligned}
\|\Sigma_r\|_2
&=
\left\|
\frac12\sum_{i,j}
\alpha_{ri}\alpha_{rj}
(x_i-x_j)(x_i-x_j)^\top
\right\|_2\\
&\leq
\frac12\sum_{i,j}
\alpha_{ri}\alpha_{rj}
\left\|
(x_i-x_j)(x_i-x_j)^\top
\right\|_2\\
&=
\frac12\sum_{i,j}
\alpha_{ri}\alpha_{rj}\|x_i-x_j\|_2^2.
\end{aligned}
\label{eq:app_covariance_norm_bound_step1}
\end{equation}
The last equality uses the rank-one identity \(\|aa^\top\|_2=\|a\|_2^2\), which follows because \(aa^\top\) has a single nonzero eigenvalue equal to \(a^\top a\).

The terms with \(i=j\) vanish because \(x_i-x_i=0\). For the remaining terms, the squared distance is bounded by \(\operatorname{diam}_r^2\). Therefore,
\begin{equation}
\begin{aligned}
\|\Sigma_r\|_2
&\leq
\frac12\sum_{i\neq j}
\alpha_{ri}\alpha_{rj}\|x_i-x_j\|_2^2\\
&\leq
\frac{\operatorname{diam}_r^2}{2}
\sum_{i\neq j}\alpha_{ri}\alpha_{rj}.
\end{aligned}
\label{eq:app_covariance_diameter_bound}
\end{equation}

The remaining sum measures the total probability mass assigned to pairs of distinct tokens. We can express this quantity through the attention concentration:
\begin{equation}
\begin{aligned}
\sum_{i\neq j}\alpha_{ri}\alpha_{rj}
&=
\sum_{i,j}\alpha_{ri}\alpha_{rj}
-\sum_i\alpha_{ri}^2\\
&=
\left(\sum_i\alpha_{ri}\right)
\left(\sum_j\alpha_{rj}\right)
-\sum_i\alpha_{ri}^2\\
&=
1-\kappa_r.
\end{aligned}
\label{eq:app_attention_concentration_identity}
\end{equation}
Combining the last two equations gives
\begin{equation}
\boxed{
\|\Sigma_r\|_2
\leq
\frac{\operatorname{diam}_r^2}{2}(1-\kappa_r).
}
\label{eq:app_covariance_norm_bound}
\end{equation}

It remains to transfer this covariance bound to the query-gating Jacobian. By submultiplicativity of the spectral norm,
\begin{equation}
\begin{aligned}
\|J_r^{(q)}\|_2
&=
\|W_{VO}\Sigma_rW_{QK}^{\top}\|_2\\
&\leq
\|W_{VO}\|_2
\|\Sigma_rW_{QK}^{\top}\|_2\\
&\leq
\|W_{VO}\|_2
\|\Sigma_r\|_2
\|W_{QK}^{\top}\|_2\\
&=
\|W_{VO}\|_2
\|\Sigma_r\|_2
\|W_{QK}\|_2.
\end{aligned}
\label{eq:app_query_gating_norm_bound}
\end{equation}
The final equality follows from the invariance of the spectral norm under transposition. Substituting Eq.~\eqref{eq:app_covariance_norm_bound} yields
\begin{equation}
\begin{aligned}
\|J_r^{(q)}\|_2
&\leq
\|W_{VO}\|_2\|W_{QK}\|_2
\frac{\operatorname{diam}_r^2}{2}(1-\kappa_r)\\
&=
\frac12\|W_{VO}\|_2\|W_{QK}\|_2
\operatorname{diam}_r^2(1-\kappa_r).
\end{aligned}
\end{equation}
Hence,
\begin{equation}
\boxed{
\|J_r^{(q)}\|_2
\leq
\frac12
\|W_{VO}\|_2
\|W_{QK}\|_2
\operatorname{diam}_r^2(1-\kappa_r).
}
\label{eq:app_query_gating_geometric_bound}
\end{equation}

The bound identifies two limiting regimes in which query gating becomes weak. When the attended representations approach one another, their diameter decreases, suppressing the covariance through the squared-distance factor. When attention becomes nearly one-hot, \(1-\kappa_r\) approaches zero, reducing the effect of reallocating probability mass among different representations. These conclusions concern the corresponding limits with bounded representations and transformation norms. Conversely, a large diameter and diffuse attention do not guarantee a large response, because the query-key and value-output transformations may eliminate or attenuate the relevant directions.

\paragraph{Geometric interpretation.}
The preceding results provide a geometric description of how query gating responds to representation perturbations. For a perturbation \(v\), the operator \(W_{QK}^{\top}\) determines the direction in which the query attempts to change its routing preferences. The covariance \(\Sigma_r\) then restricts this response to variation that actually exists among the attended representations. Finally, \(W_{VO}\) maps the surviving variation into the output space.

This interpretation explains why the range restriction is generally an inclusion rather than an equality. Although \(\operatorname{Range}(\Sigma_r)=\mathcal U_r\), the surrounding transformations can remove additional directions. The operator \(W_{QK}^{\top}\) may fail to reach some directions in \(\mathcal U_r\), while \(W_{VO}\) may map some nonzero covariance responses to zero. Consequently, query gating may operate in a subspace of substantially smaller dimension than the attended variation subspace.

A simple two-dimensional example illustrates this geometry. Consider two equally attended representations,
\[
x_1=
\begin{pmatrix}
-1\\0
\end{pmatrix},
\qquad
x_2=
\begin{pmatrix}
1\\0
\end{pmatrix},
\qquad
\alpha_{r1}=\alpha_{r2}=\frac12.
\]
Their attention-weighted mean is
\[
\begin{aligned}
\mu_r
&=
\alpha_{r1}x_1+\alpha_{r2}x_2\\
&=
\frac12
\begin{pmatrix}
-1\\0
\end{pmatrix}
+
\frac12
\begin{pmatrix}
1\\0
\end{pmatrix}\\
&=
\begin{pmatrix}
0\\0
\end{pmatrix}.
\end{aligned}
\]
Since the weighted mean is zero, the covariance becomes
\begin{equation}
\begin{aligned}
\Sigma_r
&=
\frac12x_1x_1^\top+\frac12x_2x_2^\top\\
&=
\frac12
\begin{pmatrix}
1&0\\
0&0
\end{pmatrix}
+
\frac12
\begin{pmatrix}
1&0\\
0&0
\end{pmatrix}\\
&=
\begin{pmatrix}
1&0\\
0&0
\end{pmatrix}.
\end{aligned}
\label{eq:app_query_gating_2d_example}
\end{equation}
Thus, the covariance has rank one, and its range is
\[
\mathcal U_r
=
\operatorname{span}
\left\{
\begin{pmatrix}
1\\0
\end{pmatrix}
\right\}.
\]
The attended representations vary only along the horizontal axis. To isolate the covariance effect, suppose \(W_{QK}=W_{VO}=I\), giving \(J_r^{(q)}=\Sigma_r\).

For a horizontal perturbation \(v_x=(1,0)^\top\), the query-gating response is
\[
\begin{aligned}
J_r^{(q)}v_x
&=
\begin{pmatrix}
1&0\\
0&0
\end{pmatrix}
\begin{pmatrix}
1\\0
\end{pmatrix}\\
&=
\begin{pmatrix}
1\\0
\end{pmatrix}.
\end{aligned}
\]
The response is nonzero because the perturbation acts along the direction in which the two attended representations differ. In contrast, for a vertical perturbation \(v_y=(0,1)^\top\),
\[
\begin{aligned}
J_r^{(q)}v_y
&=
\begin{pmatrix}
1&0\\
0&0
\end{pmatrix}
\begin{pmatrix}
0\\1
\end{pmatrix}\\
&=
\begin{pmatrix}
0\\0
\end{pmatrix}.
\end{aligned}
\]
Although the query representation changes, the two attended representations have no variation along the vertical direction. The covariance therefore removes the perturbation, producing no query-gating response.

The geometric norm bound is also exact in this example. The representation diameter and attention concentration are
\[
\begin{aligned}
\operatorname{diam}_r
&=\|x_1-x_2\|_2=2,\\
\kappa_r
&=\alpha_{r1}^2+\alpha_{r2}^2
=\frac14+\frac14=\frac12.
\end{aligned}
\]
Since both transformation matrices are identities, Eq.~\eqref{eq:app_query_gating_geometric_bound} gives
\[
\begin{aligned}
\|J_r^{(q)}\|_2
&\leq
\frac12(1)(1)(2^2)
\left(1-\frac12\right)\\
&=1.
\end{aligned}
\]
Directly from Eq.~\eqref{eq:app_query_gating_2d_example}, we have \(\|J_r^{(q)}\|_2=1\), so the bound is attained. The example therefore illustrates both the exact directional restriction imposed by the covariance and the dependence of query-gating magnitude on representation separation and attention concentration.

Taken together, these results establish that query gating is governed by the geometry of the representations currently receiving attention. The attended variation subspace determines the covariance support, the surrounding query-key and value-output transformations constrain which of those directions remain accessible, and the pairwise covariance identity quantifies how representation diversity and attention concentration bound the response magnitude. These results characterize where query gating can act and how strongly it can act. They do not, however, determine whether the resulting perturbation amplifies or suppresses a representation difference. That signed effect depends on how query gating interacts with value transport and key gating, motivating the subsequent analysis of directional growth and regime switching.

\subsection{Directional Routing Alignment}
\label{subsec:app_directional_alignment}

We now ask a complementary question: once a direction \(v\) is fixed, does gating reinforce or oppose motion along that direction? This signed directional effect is important because the same routing mechanism can strengthen one representation difference while suppressing another. We return to the state-dependent mixture formulation from Section~\ref{subsec:app_transport_proof}. From this point onward, we consider representation maps from \(\mathbb{R}^d\) to \(\mathbb{R}^d\), so that input and output perturbations lie in the same representation space. Recall that
\[
F(x)
=
\sum_{i=1}^{n}
\alpha_i(x)f_i(x),
\]
with gating Jacobian
\begin{equation}
J_{\mathrm{gating}}
=
\sum_{i=1}^{n}
\alpha_i
(f_i-F)
(\nabla s_i-\bar g)^\top,
\qquad
\bar g
:=
\sum_{j=1}^{n}\alpha_j\nabla s_j.
\label{eq:app_directional_gating_jacobian}
\end{equation}
When specialized to self-attention, \(J_{\mathrm{gating}}\) collects the state-dependent routing contributions, including query and key gating. Let \(M=M^\top\succ0\) be the representation metric and let \(v\in\mathbb{R}^d\setminus\{0\}\) be a representation direction. We define the signed directional gating contribution as
\begin{equation}
\mathcal{I}_{\mathrm{gate}}(v)
:=
v^\top M J_{\mathrm{gating}}v.
\label{eq:app_directional_gating_scalar}
\end{equation}
The sign of \(\mathcal{I}_{\mathrm{gate}}(v)\) indicates whether the first-order gating response reinforces or opposes displacement along \(v\): \(\mathcal{I}_{\mathrm{gate}}(v)>0\) is reinforcing, \(\mathcal{I}_{\mathrm{gate}}(v)<0\) is opposing, and \(\mathcal{I}_{\mathrm{gate}}(v)=0\) gives no signed first-order contribution along that direction. Substituting Eq.~\eqref{eq:app_directional_gating_jacobian} into Eq.~\eqref{eq:app_directional_gating_scalar} yields
\begin{equation}
\begin{aligned}
\mathcal{I}_{\mathrm{gate}}(v)
&=
v^\top M
\left[
\sum_{i=1}^{n}
\alpha_i
(f_i-F)
(\nabla s_i-\bar g)^\top
\right]
v \\
&=
\sum_{i=1}^{n}
\alpha_i
\left[
v^\top M(f_i-F)
\right]
\left[
(\nabla s_i-\bar g)^\top v
\right].
\end{aligned}
\label{eq:app_directional_gating_expansion}
\end{equation}

For each candidate \(i\), we define
\begin{equation}
a_i(v)
:=
v^\top M(f_i-F),
\qquad
b_i(v)
:=
(\nabla s_i-\bar g)^\top v.
\label{eq:app_directional_ab}
\end{equation}
The quantity \(a_i(v)\) measures where candidate \(f_i\) lies relative to the current mixture output \(F\) along direction \(v\). In contrast, \(b_i(v)\) measures how a perturbation along \(v\) changes the relative routing preference for candidate \(i\). With this notation,
\begin{equation}
\mathcal{I}_{\mathrm{gate}}(v)
=
\sum_{i=1}^{n}
\alpha_i a_i(v)b_i(v).
\label{eq:app_directional_gating_ab}
\end{equation}

Both directional quantities are centered under the routing distribution
\(\{\alpha_i\}_{i=1}^{n}\). Since the softmax weights satisfy
\(\alpha_i\geq0\) and \(\sum_i\alpha_i=1\), they define a discrete probability
distribution over the candidates. For any candidate-dependent scalar quantity
\(c_i\), we write its expectation under this routing distribution as
\[
\mathbb{E}_{\alpha}[c]
:=
\sum_{i=1}^{n}\alpha_i c_i.
\]

For the directional candidate displacement \(a_i(v)\), we have
\begin{equation}
\begin{aligned}
\mathbb{E}_{\alpha}[a(v)]
&=
\sum_{i=1}^{n}\alpha_i a_i(v) \\
&=
\sum_{i=1}^{n}
\alpha_i
v^\top M(f_i-F) \\
&=
v^\top M
\left[
\sum_{i=1}^{n}\alpha_if_i
-
F\sum_{i=1}^{n}\alpha_i
\right] \\
&=
v^\top M(F-F) \\
&=
0.
\end{aligned}
\label{eq:app_directional_a_centered}
\end{equation}
The equality follows because \(F=\sum_i\alpha_if_i\) is the routing-weighted
mean of the candidate representations. Thus \(f_i-F\) is already a centered
candidate displacement, and its weighted average is zero. Similarly, for the directional routing response \(b_i(v)\),
\begin{equation}
\begin{aligned}
\mathbb{E}_{\alpha}[b(v)]
&=
\sum_{i=1}^{n}\alpha_i b_i(v) \\
&=
\sum_{i=1}^{n}
\alpha_i
(\nabla s_i-\bar g)^\top v \\
&=
\left[
\sum_{i=1}^{n}\alpha_i\nabla s_i
-
\bar g\sum_{i=1}^{n}\alpha_i
\right]^\top v \\
&=
(\bar g-\bar g)^\top v \\
&=
0,
\end{aligned}
\label{eq:app_directional_b_centered}
\end{equation}
where \(\bar g=\sum_i\alpha_i\nabla s_i\) is the routing-weighted mean score
gradient. Hence \(\nabla s_i-\bar g\) is also centered under the same routing
distribution. The weighted covariance between \(a_i(v)\) and \(b_i(v)\) is defined as
\begin{equation}
\operatorname{Cov}_{\alpha}
\left(
a(v),b(v)
\right)
:=
\sum_{i=1}^{n}
\alpha_i
\left(
a_i(v)-\mathbb{E}_{\alpha}[a(v)]
\right)
\left(
b_i(v)-\mathbb{E}_{\alpha}[b(v)]
\right).
\label{eq:app_weighted_covariance_definition}
\end{equation}
Because both weighted expectations are zero, this simplifies to
\begin{equation}
\begin{aligned}
\operatorname{Cov}_{\alpha}
\left(
a(v),b(v)
\right)
&=
\sum_{i=1}^{n}
\alpha_i
\left(
a_i(v)-0
\right)
\left(
b_i(v)-0
\right) \\
&=
\sum_{i=1}^{n}
\alpha_i
a_i(v)b_i(v).
\end{aligned}
\label{eq:app_weighted_covariance_reduction}
\end{equation}
Comparing Eq.~\eqref{eq:app_weighted_covariance_reduction} with
Eq.~\eqref{eq:app_directional_gating_ab}, we obtain
\begin{equation}
\boxed{
\mathcal{I}_{\mathrm{gate}}(v)
=
\operatorname{Cov}_{\alpha}
\left(
a(v),b(v)
\right).
}
\label{eq:app_directional_gating_covariance}
\end{equation}

This covariance identity gives a direct interpretation of the sign of the
directional gating effect. Recall that \(a_i(v)\) measures where candidate
\(f_i\) lies relative to the current mixture output along direction \(v\),
whereas \(b_i(v)\) measures whether a perturbation along \(v\) increases or
decreases the relative routing preference for that candidate. If candidates
with \(a_i(v)>0\) tend to have \(b_i(v)>0\), while candidates with
\(a_i(v)<0\) tend to have \(b_i(v)<0\), then the two quantities vary together,
their weighted covariance is positive, and gating reinforces motion along
\(v\). Conversely, if routing increasingly favors candidates lying on the
opposite side of \(v\), then \(a_i(v)\) and \(b_i(v)\) tend to have opposite
signs, the covariance becomes negative, and gating opposes motion along
\(v\). To separate magnitude from sign, define the \emph{directional candidate diversity}
\begin{equation}
D_v
:=
\left(
\sum_{i=1}^{n}
\alpha_i a_i(v)^2
\right)^{1/2},
\label{eq:app_directional_diversity}
\end{equation}
the \emph{directional routing sensitivity}
\begin{equation}
S_v
:=
\left(
\sum_{i=1}^{n}
\alpha_i b_i(v)^2
\right)^{1/2},
\label{eq:app_directional_sensitivity}
\end{equation}
and, whenever \(D_vS_v>0\), the \emph{directional routing alignment}
\begin{equation}
\rho_v
:=
\frac{
\sum_{i=1}^{n}
\alpha_i a_i(v)b_i(v)
}{
D_vS_v
}.
\label{eq:app_directional_alignment}
\end{equation}

To characterize the range of \(\rho_v\), apply the weighted
Cauchy--Schwarz inequality to the two weighted sequences
\(\{\sqrt{\alpha_i}a_i(v)\}_{i=1}^{n}\) and
\(\{\sqrt{\alpha_i}b_i(v)\}_{i=1}^{n}\):
\begin{equation}
\begin{aligned}
\left|
\sum_{i=1}^{n}
\alpha_i a_i(v)b_i(v)
\right|^2
&=
\left|
\sum_{i=1}^{n}
\left(\sqrt{\alpha_i}a_i(v)\right)
\left(\sqrt{\alpha_i}b_i(v)\right)
\right|^2 \\
&\leq
\left(
\sum_{i=1}^{n}
\alpha_i a_i(v)^2
\right)
\left(
\sum_{i=1}^{n}
\alpha_i b_i(v)^2
\right) \\
&=
D_v^2S_v^2.
\end{aligned}
\label{eq:app_directional_cauchy_schwarz}
\end{equation}

Taking the nonnegative square root of both sides gives
\begin{equation}
\left|
\sum_{i=1}^{n}
\alpha_i a_i(v)b_i(v)
\right|
\leq
D_vS_v.
\label{eq:app_directional_cauchy_schwarz_root}
\end{equation}
Since \(D_vS_v>0\), dividing both sides by \(D_vS_v\) yields
\begin{equation}
\begin{aligned}
|\rho_v|
&=
\left|
\frac{
\sum_{i=1}^{n}\alpha_i a_i(v)b_i(v)
}{
D_vS_v
}
\right| \\
&=
\frac{
\left|
\sum_{i=1}^{n}\alpha_i a_i(v)b_i(v)
\right|
}{
D_vS_v
} \\
&\leq
1.
\end{aligned}
\label{eq:app_directional_alignment_bound_derivation}
\end{equation}
Therefore,
\begin{equation}
-1
\leq
\rho_v
\leq
1.
\label{eq:app_directional_alignment_bound}
\end{equation}

Equivalently, \(\rho_v\) is the cosine similarity between the two weighted
vectors
\[
\widetilde a(v)
:=
\left(
\sqrt{\alpha_1}a_1(v),\ldots,
\sqrt{\alpha_n}a_n(v)
\right)^\top
\]
and
\[
\widetilde b(v)
:=
\left(
\sqrt{\alpha_1}b_1(v),\ldots,
\sqrt{\alpha_n}b_n(v)
\right)^\top.
\]
Indeed,
\begin{equation}
\begin{aligned}
\rho_v
&=
\frac{
\widetilde a(v)^\top\widetilde b(v)
}{
\|\widetilde a(v)\|_2
\|\widetilde b(v)\|_2
}.
\end{aligned}
\label{eq:app_directional_alignment_cosine}
\end{equation}
Hence \(\rho_v=1\) corresponds to perfect directional alignment,
\(\rho_v=-1\) to perfect anti-alignment, and \(\rho_v=0\) to zero weighted
linear alignment between candidate displacement and routing response. We now recover the directional gating contribution. From
Eq.~\eqref{eq:app_directional_gating_ab},
\begin{equation}
\mathcal{I}_{\mathrm{gate}}(v)
=
\sum_{i=1}^{n}
\alpha_i a_i(v)b_i(v).
\label{eq:app_directional_gating_recall}
\end{equation}
When \(D_vS_v>0\), rearranging the definition of \(\rho_v\) gives
\begin{equation}
\begin{aligned}
\rho_v
&=
\frac{
\sum_{i=1}^{n}
\alpha_i a_i(v)b_i(v)
}{
D_vS_v
} \\
\Longrightarrow\qquad
\sum_{i=1}^{n}
\alpha_i a_i(v)b_i(v)
&=
\rho_vD_vS_v.
\end{aligned}
\label{eq:app_directional_alignment_rearrangement}
\end{equation}
Combining this identity with
\(\mathcal{I}_{\mathrm{gate}}(v)=v^\top M J_{\mathrm{gating}}v\) therefore yields
\begin{equation}
\boxed{
v^\top M J_{\mathrm{gating}}v
=
\mathcal{I}_{\mathrm{gate}}(v)
=
\rho_vD_vS_v.
}
\label{eq:app_directional_gating_law}
\end{equation}

The same identity also remains valid when \(D_vS_v=0\). In that case,
either \(D_v=0\) or \(S_v=0\). If \(D_v=0\), then
\(\sum_i\alpha_i a_i(v)^2=0\), which implies
\(a_i(v)=0\) for every candidate with \(\alpha_i>0\). Similarly, if
\(S_v=0\), then \(b_i(v)=0\) for every candidate with positive routing
weight. In either case,
\[
\sum_i\alpha_i a_i(v)b_i(v)=0,
\]
and because we define \(\rho_v=0\), the right-hand side
\(\rho_vD_vS_v\) is also zero. Thus
Eq.~\eqref{eq:app_directional_gating_law} holds in all cases. Equation~\eqref{eq:app_directional_gating_law} separates three distinct requirements for a strong directional gating effect. \(D_v\) measures whether the available candidate representations contain meaningful variation along \(v\); \(S_v\) measures whether routing is locally responsive to perturbations along \(v\); and \(\rho_v\) determines whether this routing response favors candidates that reinforce or oppose the direction.

\begin{corollary}[Three Mechanisms of Weak Gating]
\label{cor:weak_gating}
If $D_v = 0$, $S_v = 0$, or $\rho_v = 0$, then
\[
I_{\mathrm{gate}}(v) = 0.
\]
\end{corollary}
These cases have different interpretations. \(D_v=0\) means that the candidates are indistinguishable along \(v\), so changing their routing weights cannot move the representation in that direction. \(S_v=0\) means that routing itself is locally insensitive to \(v\). Finally, \(\rho_v=0\) means that candidate displacement and routing response are individually present but cancel because they are directionally unaligned. This distinction is useful diagnostically: an observed lack of representation change may arise because the relevant information is absent, because routing does not respond to it, or because routing responds in a directionally neutral manner.

\begin{corollary}[Directional Upper Bound]
\label{cor:directional_upper_bound}
For any direction $v \neq 0$,
\[
|I_{\mathrm{gate}}(v)| \leq D_v S_v.
\]
\end{corollary}

Thus diversity and routing sensitivity determine the maximum possible magnitude of the local gating effect, while \(\rho_v\) determines how much of this potential effect is realized with positive or negative sign.

\paragraph{Example: reinforcing versus suppressing routing.}
Consider a one-dimensional directional view with two equally weighted candidates, \(\alpha_1=\alpha_2=1/2\). Suppose their projected displacements along \(v\) are
\[
a_1(v)=1,
\qquad
a_2(v)=-1.
\]
The two candidates therefore lie symmetrically on opposite sides of the current mixture output along \(v\). The directional diversity is
\begin{equation}
\begin{aligned}
D_v
&=
\left[
\frac{1}{2}(1)^2
+
\frac{1}{2}(-1)^2
\right]^{1/2} \\
&=
1.
\end{aligned}
\label{eq:app_directional_example_diversity}
\end{equation}

Now suppose that a perturbation along \(v\) increases the relative score of the positive-side candidate and decreases that of the negative-side candidate:
\[
b_1(v)=s,
\qquad
b_2(v)=-s,
\qquad
s>0.
\]
Then
\begin{equation}
\begin{aligned}
S_v
&=
\left[
\frac{1}{2}s^2
+
\frac{1}{2}(-s)^2
\right]^{1/2} \\
&=
s,
\end{aligned}
\label{eq:app_directional_example_sensitivity}
\end{equation}
and
\begin{equation}
\begin{aligned}
\rho_v
&=
\frac{
\frac{1}{2}(1)(s)
+
\frac{1}{2}(-1)(-s)
}{
(1)(s)
} \\
&=
1.
\end{aligned}
\label{eq:app_directional_example_positive_alignment}
\end{equation}
Therefore,
\begin{equation}
\mathcal{I}_{\mathrm{gate}}(v)
=
\rho_vD_vS_v
=
s
>
0.
\label{eq:app_directional_example_positive}
\end{equation}
Routing shifts mass toward the candidate lying in the positive \(v\)-direction and away from the candidate lying in the negative direction, so gating reinforces the perturbation. If instead the routing response is reversed,
\[
b_1(v)=-s,
\qquad
b_2(v)=s,
\]
then
\begin{equation}
\begin{aligned}
\rho_v
&=
\frac{
\frac{1}{2}(1)(-s)
+
\frac{1}{2}(-1)(s)
}{
(1)(s)
} \\
&=
-1,
\end{aligned}
\label{eq:app_directional_example_negative_alignment}
\end{equation}
and hence
\begin{equation}
\mathcal{I}_{\mathrm{gate}}(v)
=
-s
<
0.
\label{eq:app_directional_example_negative}
\end{equation}
The candidate geometry and routing sensitivity are unchanged: \(D_v=1\) and \(S_v=s\) in both cases. Only the alignment changes. Nevertheless, the qualitative behavior reverses from reinforcement to suppression. This example illustrates why large candidate diversity and large routing sensitivity alone are insufficient to determine directional growth; their alignment is essential.

\paragraph{Implications for downstream interventions.}
The decomposition above provides a local diagnostic for several downstream settings in which an intervention is useful only if its representation effect is not immediately suppressed. Consider, for example, a candidate steering or editing direction \(v\). A positive value of \(\mathcal{I}_{\mathrm{gate}}(v)\) indicates that state-dependent routing locally reinforces this intervention, whereas a negative value indicates that routing locally opposes it. The quantities \(D_v\), \(S_v\), and \(\rho_v\) further identify why: the relevant feature may be absent from the available candidate representations, routing may be insensitive to it, or routing may respond but redirect information in the unfavorable direction.

The same distinction is useful for adversarial perturbations. A perturbation direction with large \(S_v\) is not necessarily effective if the candidate representations exhibit little diversity along that direction or if \(\rho_v<0\). Directions with large positive \(\rho_vD_vS_v\) are therefore natural candidates for further downstream evaluation, whereas directions with negative alignment are locally resisted by routing.

For pruning and head-level diagnosis, the decomposition separates routing activity from functional routing effect. A head may exhibit substantial changes in its attention weights, corresponding to non-negligible routing sensitivity, while still having small \(D_v\) because the representations between which it redistributes mass are nearly identical. Such a head can appear dynamically active in attention space while producing little directional representation change. Conversely, a head with substantial \(D_vS_v\) and positive alignment may have a larger functional effect despite less visually dramatic attention changes.

For robustness, the same factorization identifies distinct ways of reducing sensitivity to an undesirable direction: decreasing the available candidate diversity \(D_v\), decreasing routing sensitivity \(S_v\), or reducing the directional alignment \(|\rho_v|\). These mechanisms are not equivalent, and distinguishing them can help separate representational causes of instability from routing-induced causes.

These interpretations remain local. A positive gating contribution at one layer does not by itself guarantee that the corresponding representation difference will survive to the model output, because subsequent layers may rotate or suppress the direction. Likewise, a large local value of \(\mathcal{I}_{\mathrm{gate}}(v)\) does not by itself guarantee a label change or successful steering outcome. Downstream behavior additionally depends on later Jacobian composition
and output alignment, which lie beyond the scope of the present
local analysis.

Most importantly, Eq.~\eqref{eq:app_directional_gating_law} evaluates gating after a direction \(v\) has already been specified. Our central theoretical question is stronger: whether there exists any direction for which gating is sufficiently strong to overturn contraction induced by transport, and which direction maximizes this effect. The next section converts the directional analysis above into a spectral characterization of such gating-induced regime switches.

\subsection{Extensions to Modern Transformer Blocks}
\label{subsec:app_modern_blocks}

The preceding analysis used a single-head dot-product attention map to expose the transport--gating mechanism in its simplest form. We now show how the same decomposition extends to the components commonly surrounding self-attention in a Transformer block. The central transport--gating distinction is preserved under multi-head aggregation and differentiable composition, although some of the more specific factorizations, such as \(W_{VO}\Sigma_rW_{QK}^{\top}\), must be modified when the score operator itself depends on token position.

\paragraph{Multi-head self-attention.}
Consider \(H\) attention heads. For head \(h\in\{1,\ldots,H\}\), let
\(\alpha_{ri}^{(h)}\) denote the attention weight from query token \(r\) to candidate token \(i\), and let \(W_{VO}^{(h)}\) denote the corresponding value transformation together with the portion of the output projection associated with that head. The multi-head output can then be written as
\begin{equation}
F_r^{\mathrm{MHA}}
=
\sum_{h=1}^{H}
\sum_{i=1}^{n}
\alpha_{ri}^{(h)}
W_{VO}^{(h)}x_i.
\label{eq:app_multihead_output}
\end{equation}
This representation is equivalent to the usual concatenation followed by an output projection, after partitioning the output projection into the blocks associated with individual heads.

Because differentiation is linear, the Jacobian with respect to a perturbation of token \(m\) is the sum of the head-wise Jacobians:
\begin{equation}
\begin{aligned}
J_{rm}^{\mathrm{MHA}}
&:=
\frac{\partial F_r^{\mathrm{MHA}}}{\partial x_m} \\
&=
\sum_{h=1}^{H}
\frac{\partial F_r^{(h)}}{\partial x_m} \\
&=
\sum_{h=1}^{H}
J_{rm}^{(h)}.
\end{aligned}
\label{eq:app_multihead_jacobian}
\end{equation}
For ordinary dot-product attention without position-dependent score transformations, each head satisfies the decomposition derived in Section~\ref{app:attention_jacobian}:
\begin{equation}
\begin{aligned}
J_{rm}^{(h)}
&=
\underbrace{
\alpha_{rm}^{(h)}W_{VO}^{(h)}
}_{\text{value transport}} \\
&\quad+
\underbrace{
\mathbf{1}_{\{m=r\}}
W_{VO}^{(h)}
\Sigma_r^{(h)}
\bigl(W_{QK}^{(h)}\bigr)^\top
}_{\text{query gating}} \\
&\quad+
\underbrace{
\alpha_{rm}^{(h)}
W_{VO}^{(h)}
\bigl(x_m-\mu_r^{(h)}\bigr)
\left[
\bigl(W_{QK}^{(h)}\bigr)^\top x_r
\right]^\top
}_{\text{key gating}},
\end{aligned}
\label{eq:app_multihead_individual_jacobian}
\end{equation}
where
\(\mu_r^{(h)}:=\sum_i\alpha_{ri}^{(h)}x_i\) and
\(\Sigma_r^{(h)}:=\sum_i\alpha_{ri}^{(h)}
(x_i-\mu_r^{(h)})(x_i-\mu_r^{(h)})^\top\).

Thus the complete multi-head Jacobian retains the additive structure
\begin{equation}
\boxed{
J_{rm}^{\mathrm{MHA}}
=
J_{rm}^{\mathrm{transport}}
+
J_{rm}^{\mathrm{gating}},
}
\label{eq:app_multihead_transport_gating}
\end{equation}
with the transport and gating terms obtained by summing the corresponding contributions across heads. Different heads may therefore reinforce, cancel, or act on different representation subspaces.

\paragraph{Pre-normalization and residual connections.}
We next incorporate token-wise normalization. Let
\(\mathcal{N}:\mathbb{R}^d\rightarrow\mathbb{R}^d\) denote either LayerNorm, RMSNorm, or another differentiable normalization map, and define
\[
\widetilde{x}_i
:=
\mathcal{N}(x_i),
\qquad
D_i
:=
\frac{\partial \mathcal{N}(x_i)}{\partial x_i}.
\]
Suppose attention is evaluated on the normalized representations,
\(F_r(\widetilde{x}_1,\ldots,\widetilde{x}_n)\). If
\(\widetilde J_{rm}:=\partial F_r/\partial\widetilde{x}_m\) denotes the attention Jacobian with respect to the normalized token, then the chain rule gives
\begin{equation}
\begin{aligned}
\frac{\partial F_r}{\partial x_m}
&=
\frac{\partial F_r}{\partial\widetilde{x}_m}
\frac{\partial\widetilde{x}_m}{\partial x_m} \\
&=
\widetilde J_{rm}D_m.
\end{aligned}
\label{eq:app_prenorm_attention_jacobian}
\end{equation}
Hence normalization does not remove the transport--gating decomposition; it transforms the input directions through the local normalization Jacobian \(D_m\).

For a pre-normalized residual attention sublayer,
\[
y_r
=
x_r
+
F_r\bigl(\mathcal{N}(x_1),\ldots,\mathcal{N}(x_n)\bigr),
\]
the token-to-token Jacobian becomes
\begin{equation}
\boxed{
\frac{\partial y_r}{\partial x_m}
=
\mathbf{1}_{\{m=r\}}I
+
\widetilde J_{rm}D_m.
}
\label{eq:app_residual_attention_jacobian}
\end{equation}
The identity term propagates the residual stream directly, while
\(\widetilde J_{rm}D_m\) contains the attention transport and gating effects.

For completeness, consider LayerNorm
\[
\mathcal{N}_{\mathrm{LN}}(x)
=
\operatorname{Diag}(\gamma)
\frac{x-\mu(x)\mathbf{1}}{\sigma(x)}
+
\beta,
\]
where
\(\mu(x):=d^{-1}\mathbf{1}^\top x\),
\(\Pi_{\mathrm{ctr}}:=I-d^{-1}\mathbf{1}\mathbf{1}^\top\),
\(c:=\Pi_{\mathrm{ctr}}x\), and
\(\sigma(x):=\sqrt{d^{-1}\|c\|_2^2+\varepsilon}\).
Its Jacobian is
\begin{equation}
D_{\mathrm{LN}}(x)
=
\operatorname{Diag}(\gamma)
\left[
\frac{\Pi_{\mathrm{ctr}}}{\sigma(x)}
-
\frac{cc^\top}{d\,\sigma(x)^3}
\right].
\label{eq:app_layernorm_jacobian}
\end{equation}
For RMSNorm,
\[
\mathcal{N}_{\mathrm{RMS}}(x)
=
\operatorname{Diag}(\gamma)\frac{x}{\eta(x)},
\qquad
\eta(x)
:=
\sqrt{\frac{\|x\|_2^2}{d}+\varepsilon},
\]
and
\begin{equation}
D_{\mathrm{RMS}}(x)
=
\operatorname{Diag}(\gamma)
\left[
\frac{I}{\eta(x)}
-
\frac{xx^\top}{d\,\eta(x)^3}
\right].
\label{eq:app_rmsnorm_jacobian}
\end{equation}
These expressions show explicitly that normalization can rotate and rescale perturbation directions before they enter the attention mechanism. The attention decomposition itself is applied after this local transformation.

\paragraph{Feed-forward and gated feed-forward sublayers.}
The feed-forward network is token-wise and therefore does not introduce additional cross-token routing. Nevertheless, it can transform an attention-induced representation difference before that difference reaches later layers.

For a standard feed-forward map
\[
H(x)
=
W_2\phi(W_1x+b_1)+b_2,
\]
its Jacobian is
\begin{equation}
J_H(x)
=
W_2
\operatorname{Diag}
\left(
\phi'(W_1x+b_1)
\right)
W_1.
\label{eq:app_ffn_jacobian}
\end{equation}

For a gated feed-forward map of the form
\[
H_{\mathrm{gate}}(x)
=
W_2
\left[
\phi(W_gx)
\odot
(W_ux)
\right],
\]
which includes SwiGLU-type structures for the appropriate choice of \(\phi\), define
\(g:=W_gx\) and \(u:=W_ux\). Differentiating the element-wise product gives
\begin{equation}
\begin{aligned}
dH_{\mathrm{gate}}
&=
W_2\,
d\left[
\phi(g)\odot u
\right] \\
&=
W_2
\left[
\operatorname{Diag}(u)
\operatorname{Diag}\bigl(\phi'(g)\bigr)\,dg
+
\operatorname{Diag}\bigl(\phi(g)\bigr)\,du
\right] \\
&=
W_2
\left[
\operatorname{Diag}(u)
\operatorname{Diag}\bigl(\phi'(g)\bigr)W_g
+
\operatorname{Diag}\bigl(\phi(g)\bigr)W_u
\right]dx.
\end{aligned}
\label{eq:app_gated_ffn_differential}
\end{equation}
Hence
\begin{equation}
J_{H_{\mathrm{gate}}}(x)
=
W_2
\left[
\operatorname{Diag}(W_ux)
\operatorname{Diag}\bigl(\phi'(W_gx)\bigr)W_g
+
\operatorname{Diag}\bigl(\phi(W_gx)\bigr)W_u
\right].
\label{eq:app_gated_ffn_jacobian}
\end{equation}
The feed-forward sublayer therefore contributes an additional local transformation of the directions produced by attention, but it does not alter the distinction between value transport and state-dependent attention routing inside the attention sublayer.

\paragraph{Full pre-normalized Transformer block.}
The preceding components can be combined through the chain rule. Let \(X\) denote the vectorized sequence representation, let \(\mathcal{A}\) denote multi-head attention, and let \(\mathcal{H}\) denote the token-wise feed-forward map. A pre-normalized block has the form
\[
Y
=
X+\mathcal{A}(\mathcal{N}_1(X)),
\qquad
Z
=
Y+\mathcal{H}(\mathcal{N}_2(Y)).
\]
Let \(D_1\) and \(D_2\) denote the block-diagonal Jacobians of the two token-wise normalization maps, and let \(J_{\mathcal A}\) and \(J_{\mathcal H}\) denote the Jacobians of attention and the feed-forward map with respect to their normalized inputs. Then
\begin{equation}
\begin{aligned}
J_{Y\leftarrow X}
&=
I+J_{\mathcal A}D_1, \\
J_{Z\leftarrow Y}
&=
I+J_{\mathcal H}D_2,
\end{aligned}
\label{eq:app_block_sub_jacobians}
\end{equation}
and therefore
\begin{equation}
\boxed{
J_{\mathrm{block}}
=
\left(
I+J_{\mathcal H}D_2
\right)
\left(
I+J_{\mathcal A}D_1
\right).
}
\label{eq:app_full_prenorm_block_jacobian}
\end{equation}
If
\(J_{\mathcal A}=J_{\mathrm{transport}}+J_{\mathrm{gating}}\),
then
\begin{equation}
\begin{aligned}
J_{\mathrm{block}}
&=
\left(
I+J_{\mathcal H}D_2
\right)
\left[
I+
\left(
J_{\mathrm{transport}}
+
J_{\mathrm{gating}}
\right)D_1
\right].
\end{aligned}
\label{eq:app_full_block_transport_gating}
\end{equation}
Thus the attention routing effect remains explicitly identifiable inside the full block, while normalization, residual connections, and the feed-forward sublayer determine how that effect is subsequently transformed.

\paragraph{Rotary positional embeddings.}
Rotary positional embeddings require slightly more care because the query--key interaction becomes position dependent. Let \(R_r,R_i\in\mathbb{R}^{d_k\times d_k}\) denote the rotations applied at query position \(r\) and key position \(i\), respectively. Define
\[
Q_r
=
R_rW_Q^\top x_r,
\qquad
k_i
=
R_iW_K^\top x_i.
\]
The attention score becomes
\begin{equation}
\begin{aligned}
s_{ri}
&=
\frac{Q_r^\top k_i}{\sqrt{d_k}} \\
&=
x_r^\top
\left(
\frac{
W_QR_r^\top R_iW_K^\top
}{
\sqrt{d_k}
}
\right)
x_i.
\end{aligned}
\label{eq:app_rope_score}
\end{equation}
Define the position-dependent score operator
\[
A_{ri}
:=
\frac{
W_QR_r^\top R_iW_K^\top
}{
\sqrt{d_k}
}.
\]
Then \(s_{ri}=x_r^\top A_{ri}x_i\), and the query-score gradient is
\begin{equation}
\nabla_{x_r}s_{ri}
=
A_{ri}x_i.
\label{eq:app_rope_query_gradient}
\end{equation}

Because \(A_{ri}\) now depends on the candidate position \(i\), a single common matrix \(W_{QK}\) can no longer, in general, be factored outside the attention-weighted sum. The query-gating Jacobian instead follows directly from the generic gating formula. Define
\[
g_{ri}
:=
A_{ri}x_i,
\qquad
\bar g_r
:=
\sum_i\alpha_{ri}g_{ri}.
\]
Then
\begin{equation}
\begin{aligned}
J_r^{(q)}
&=
\sum_{i=1}^{n}
\alpha_{ri}
W_{VO}(x_i-\mu_r)
(g_{ri}-\bar g_r)^\top \\
&=
W_{VO}
\left[
\sum_{i=1}^{n}
\alpha_{ri}
(x_i-\mu_r)
(g_{ri}-\bar g_r)^\top
\right].
\end{aligned}
\label{eq:app_rope_query_gating}
\end{equation}
The matrix in brackets is an attention-weighted \emph{cross-covariance} between candidate representation variation and query-score-gradient variation. Define
\begin{equation}
C_{xg,r}
:=
\sum_{i=1}^{n}
\alpha_{ri}
(x_i-\mu_r)
(g_{ri}-\bar g_r)^\top.
\label{eq:app_rope_cross_covariance}
\end{equation}
Then
\begin{equation}
\boxed{
J_r^{(q)}
=
W_{VO}C_{xg,r}.
}
\label{eq:app_rope_query_gating_compact}
\end{equation}

The covariance factorization from Section~\ref{subsec:app_gating_geometry} is recovered as a special case. Without position-dependent score rotations, \(A_{ri}=W_{QK}\) for every \(i\), so
\begin{equation}
\begin{aligned}
g_{ri}-\bar g_r
&=
W_{QK}x_i
-
W_{QK}\mu_r \\
&=
W_{QK}(x_i-\mu_r).
\end{aligned}
\label{eq:app_no_rope_gradient_centering}
\end{equation}
Consequently,
\begin{equation}
\begin{aligned}
C_{xg,r}
&=
\sum_i
\alpha_{ri}
(x_i-\mu_r)
\left[
W_{QK}(x_i-\mu_r)
\right]^\top \\
&=
\sum_i
\alpha_{ri}
(x_i-\mu_r)(x_i-\mu_r)^\top
W_{QK}^\top \\
&=
\Sigma_rW_{QK}^\top,
\end{aligned}
\label{eq:app_rope_reduction_to_covariance}
\end{equation}
and therefore
\(J_r^{(q)}=W_{VO}\Sigma_rW_{QK}^\top\), as derived previously. Thus RoPE does not remove the transport--gating mechanism; it replaces the simple covariance factorization by a position-dependent representation--routing cross-covariance.

\paragraph{Grouped-query and multi-query attention.}
Grouped-query attention and multi-query attention share key and/or value projections across multiple query heads. Let \(g(h)\) denote the key--value group associated with query head \(h\). A head-specific score may then be written as
\[
s_{ri}^{(h)}
=
\frac{
\left(Q_r^{(h)}\right)^\top
k_i^{(g(h))}
}{
\sqrt{d_k}
}.
\]
The corresponding output remains a state-dependent weighted mixture, so the head-wise transport--gating decomposition continues to hold. Sharing changes the parameterization of the score and value maps but not the product-rule argument that separates changes in candidate information from changes in routing weights. The full Jacobian is again obtained by summing the head-wise contributions, with the shared key and value projections appearing in the corresponding terms.

\paragraph{Attention masks.}
Causal, padding, or structured masks restrict the candidate set available to a query. If
\(\mathcal{A}_r\) denotes the set of unmasked candidates, then all sums above can be restricted to \(i\in\mathcal{A}_r\). In particular,
\[
\sum_{i\in\mathcal{A}_r}\alpha_{ri}=1,
\]
so the centering, covariance, and transport--gating derivations remain unchanged on the active candidate set. The geometric support of gating is consequently determined by the variation among the representations that are actually available to the query.

\paragraph{Summary.}
The extensions above separate architecture-specific transformations from the underlying routing mechanism. Multi-head aggregation adds head-wise transport and gating contributions; normalization preconditions the perturbation directions; residual connections add a direct identity path; feed-forward sublayers post-process the resulting representation differences; and positional transformations such as RoPE modify the precise score geometry. In all cases, however, the local attention response retains the same fundamental distinction: one component changes the information being transported, while another arises because the state perturbation changes how information is routed. The regime-switch analysis in the next section is formulated at the operator level and therefore applies once these architecture-specific contributions have been incorporated into the corresponding local transport and gating Jacobians.


\section{THEORY OF GATING-INDUCED REGIME SWITCHING}
\label{sec:app_regime_theory}

The previous section showed how transport and gating contribute to the local
attention dynamics. We now ask whether gating can turn a representation direction that contracts under transport alone into one that expands under the full dynamics. To answer this,
we compare two dynamics for the same direction. In the first, the attention
routing is held fixed, so the direction evolves only through transport. In the
second, routing is allowed to change with the representation, so both transport
and gating are present. We call it a \emph{gating-induced regime switch} when a
direction contracts in the transport-only dynamics but expands in the full
dynamics.

The main problem is therefore to determine whether such a direction exists, and
if so, which direction is affected most strongly by gating. We formulate this
comparison using the growth operators of the transport-only and full dynamics.
Their difference measures how much adding gating changes directional growth.
This leads to a spectral criterion that is necessary and sufficient for the
existence of a switching direction, and also identifies the direction for which
gating most strongly overcomes transport-induced contraction. We then derive
conditions under which switching cannot occur, study how many independent
switching directions can exist, and analyze how robust the criterion is when the
relevant operators are estimated approximately.

\subsection{Transport-Only and Full Representation Dynamics}
\label{subsec:app_transport_full_growth}

Recall that the local attention Jacobian separates as
\(J_{\mathrm{full}}
=
J_{\mathrm{transport}}
+
J_{\mathrm{gating}}\).
We compare the growth of the same representation direction under two
counterfactual dynamics: one in which the routing weights are held fixed and
only transport remains, and one in which state-dependent gating is restored.
Using the growth operator from
Section~\ref{subsec:app_local_dynamics}, these two cases are
\begin{equation}
\mathcal{G}_{\mathrm{transport}}^M
=
J_{\mathrm{transport}}^\top
M
J_{\mathrm{transport}}
-M,
\qquad
\mathcal{G}_{\mathrm{full}}^M
=
J_{\mathrm{full}}^\top
M
J_{\mathrm{full}}
-M.
\label{eq:app_transport_full_growth_operators}
\end{equation}

\paragraph{Definition~\ref{def:gating_regime_switch} (restated): Gating-Induced Regime Switch.}
A nonzero direction \(v\) undergoes a
\emph{gating-induced regime switch} if
\begin{equation}
v^\top\mathcal{G}_{\mathrm{transport}}^M v<0,
\qquad
v^\top\mathcal{G}_{\mathrm{full}}^M v>0.
\label{eq:app_regime_switch_definition}
\end{equation}
Thus, the direction contracts when routing is fixed but expands once
state-dependent gating is restored.

To quantify how restoring gating changes directional growth, define
\begin{equation}
\Delta\mathcal{G}_{\mathrm{gating}}^M
:=
\mathcal{G}_{\mathrm{full}}^M
-
\mathcal{G}_{\mathrm{transport}}^M.
\label{eq:app_gating_growth_correction}
\end{equation}

\medskip

\noindent\textbf{Proposition~\ref{prop:gating_growth_correction}.}
\label{prop:app_exact_gating_growth_correction}
The change in the growth operator induced by gating is
\begin{equation}
\begin{aligned}
\Delta\mathcal{G}_{\mathrm{gating}}^M
&=
J_{\mathrm{transport}}^\top
M
J_{\mathrm{gating}}
+
J_{\mathrm{gating}}^\top
M
J_{\mathrm{transport}} \\
&\qquad
+
J_{\mathrm{gating}}^\top
M
J_{\mathrm{gating}}.
\end{aligned}
\label{eq:app_exact_gating_growth_correction}
\end{equation}

\begin{proof}
Since
\(J_{\mathrm{full}}
=
J_{\mathrm{transport}}
+
J_{\mathrm{gating}}\),
we have
\begin{equation}
\begin{aligned}
\mathcal{G}_{\mathrm{full}}^M
&=
J_{\mathrm{full}}^\top
M
J_{\mathrm{full}}
-M \\
&=
\left(
J_{\mathrm{transport}}
+
J_{\mathrm{gating}}
\right)^\top
M
\left(
J_{\mathrm{transport}}
+
J_{\mathrm{gating}}
\right)
-M.
\end{aligned}
\end{equation}
Expanding the product gives
\begin{equation}
\begin{aligned}
\mathcal{G}_{\mathrm{full}}^M
&=
J_{\mathrm{transport}}^\top
M
J_{\mathrm{transport}}
-M \\
&\quad+
J_{\mathrm{transport}}^\top
M
J_{\mathrm{gating}}
+
J_{\mathrm{gating}}^\top
M
J_{\mathrm{transport}} \\
&\quad+
J_{\mathrm{gating}}^\top
M
J_{\mathrm{gating}}.
\end{aligned}
\end{equation}
The first line is exactly
\(\mathcal{G}_{\mathrm{transport}}^M\).
Therefore,
\begin{equation}
\begin{aligned}
\mathcal{G}_{\mathrm{full}}^M
-
\mathcal{G}_{\mathrm{transport}}^M
&=
J_{\mathrm{transport}}^\top
M
J_{\mathrm{gating}}
+
J_{\mathrm{gating}}^\top
M
J_{\mathrm{transport}} \\
&\qquad
+
J_{\mathrm{gating}}^\top
M
J_{\mathrm{gating}},
\end{aligned}
\end{equation}
which proves the result.
\end{proof}

Equation~\eqref{eq:app_exact_gating_growth_correction} shows that gating
changes representation growth through two distinct effects. The first two
terms describe the interaction between the transport and gating responses,
and can either increase or decrease growth depending on their alignment.
The final term is the growth contributed by gating itself and is
nonnegative in every direction because \(M\succ0\). Thus, the effect of
gating cannot be characterized by its magnitude alone; its interaction
with transport also matters.

For any direction \(v\),
\begin{equation}
v^\top\mathcal{G}_{\mathrm{full}}^M v
=
v^\top\mathcal{G}_{\mathrm{transport}}^M v
+
v^\top\Delta\mathcal{G}_{\mathrm{gating}}^M v.
\label{eq:app_full_growth_relative_to_transport}
\end{equation}
If transport contracts \(v\), so that
\(v^\top\mathcal{G}_{\mathrm{transport}}^M v<0\),
then gating reverses this contraction exactly when
\begin{equation}
v^\top\Delta\mathcal{G}_{\mathrm{gating}}^M v
>
-
v^\top\mathcal{G}_{\mathrm{transport}}^M v.
\label{eq:app_exact_switch_condition}
\end{equation}
The left-hand side measures how much restoring gating changes growth along
\(v\), whereas the right-hand side is the amount of transport-induced
contraction that must be overcome. The next subsection asks whether such a
direction exists without searching over directions individually.

\subsection{Spectral Criterion for Regime Switching}
\label{subsec:app_spectral_switch}

Equation~\eqref{eq:app_exact_switch_condition} determines whether a fixed
direction undergoes a regime switch. We now ask whether such a direction
exists within an entire region in which transport alone is contracting. Let \(\mathcal S\subseteq\mathbb R^d\) be an \(r\)-dimensional subspace, and
let \(P\in\mathbb R^{d\times r}\) have orthonormal columns
\(p_1,\ldots,p_r\) spanning \(\mathcal S\). Thus, every direction
\(v\in\mathcal S\) can be written as
\(v=Pu\) for some coordinate vector \(u\in\mathbb R^r\). We call
\(\mathcal S\) transport-contractive when transport decreases the
\(M\)-energy of every nonzero direction in this subspace, i.e.,
\begin{equation}
    v^\top
\mathcal G_{\mathrm{transport}}^M
v
<0
\qquad
\text{for all }v\in\mathcal S\setminus\{0\}.
\end{equation}
Such a subspace exists whenever $\mathcal{G}_{\text {transport }}^M$ has at least one negative eigenvalue. More generally, if it has $r$ negative eigenvalues, their associated eigenvectors span an $r$-dimensional transport-contractive subspace. Substituting \(v=Pu\) gives
\begin{equation}
    v^\top
\mathcal G_{\mathrm{transport}}^M
v
=
u^\top
P^\top
\mathcal G_{\mathrm{transport}}^M
P
u.
\end{equation}
Therefore, requiring every nonzero direction in \(\mathcal S\) to be
transport-contractive is equivalent to
\begin{equation}
P^\top
\mathcal G_{\mathrm{transport}}^M
P
\prec 0.
\label{eq:app_transport_contractive_subspace}
\end{equation}
In other words, the projected operator
\(P^\top\mathcal G_{\mathrm{transport}}^M P\) describes transport growth
using coordinates inside \(\mathcal S\), and its negative definiteness
guarantees contraction along every nonzero direction in that subspace.
Because this projected operator is negative definite, its negation is
positive definite. We denote $C
:=
-
P^\top
\mathcal G_{\mathrm{transport}}^M
P
\succ0,$ so that \(u^\top C u\) measures the transport contraction experienced by
the direction \(v=Pu\). On the same subspace, the change in growth produced
by restoring gating is represented by $R
:=
P^\top
\Delta\mathcal G_{\mathrm{gating}}^M
P.$ Hence the full growth operator restricted to \(\mathcal S\) takes the simple
matrix form
\begin{equation}
P^\top
\mathcal G_{\mathrm{full}}^M
P
=
-C+R.
\label{eq:app_projected_full_growth}
\end{equation}

A regime switch therefore exists whenever the gating-induced growth \(R\)
can overcome the transport-contraction operator \(C\) along at least one
direction. The following result gives an exact criterion.

\paragraph{Theorem 1 (restated): Spectral Criterion for Gating-Induced Regime Switching.}
Let \(\mathcal S\) be a transport-contractive subspace satisfying
Eq.~\eqref{eq:app_transport_contractive_subspace}.
Define \(\mathcal W:=C^{-1/2}RC^{-1/2}\).
Then there exists a nonzero direction \(v\in\mathcal S\) that undergoes
a gating-induced regime switch if and only if
\begin{equation}
\lambda_{\max}(\mathcal W)>1.
\label{eq:app_spectral_switch_criterion}
\end{equation}
Moreover, whenever this condition holds, a switching direction is
\[
v_{\mathrm{sw}}=PC^{-1/2}w_{\max},
\]
where \(w_{\max}\) is a unit eigenvector of \(\mathcal W\)
associated with \(\lambda_{\max}(\mathcal W)\).

\begin{proof}
Because the columns of \(P\) form an orthonormal basis of
\(\mathcal S\), every nonzero direction \(v\in\mathcal S\) can be written
uniquely as
\[
v=Pu,
\qquad
u\in\mathbb R^r\setminus\{0\}.
\]
Since \(\mathcal S\) is transport-contractive,
\(P^\top\mathcal G_{\mathrm{transport}}^M P\prec0\). Therefore, for every
\(u\neq0\),
\begin{equation}
\begin{aligned}
v^\top
\mathcal G_{\mathrm{transport}}^M
v
&=
(Pu)^\top
\mathcal G_{\mathrm{transport}}^M
(Pu) \\
&=
u^\top
P^\top
\mathcal G_{\mathrm{transport}}^M
P
u \\
&=
-\,u^\top C u \\
&<
0,
\end{aligned}
\label{eq:app_transport_contraction_on_subspace}
\end{equation}
where
\(C=-P^\top\mathcal G_{\mathrm{transport}}^M P\succ0\).
Hence the transport-contraction part of
Definition~\ref{def:gating_regime_switch} is automatically satisfied
throughout \(\mathcal S\). It remains to determine whether the full dynamics expand at least one
direction in this subspace. Using
\(
\mathcal G_{\mathrm{full}}^M
=
\mathcal G_{\mathrm{transport}}^M
+
\Delta\mathcal G_{\mathrm{gating}}^M
\),
we obtain
\begin{equation}
\begin{aligned}
v^\top
\mathcal G_{\mathrm{full}}^M
v
&=
(Pu)^\top
\mathcal G_{\mathrm{full}}^M
(Pu) \\
&=
u^\top
P^\top
\mathcal G_{\mathrm{full}}^M
P
u \\
&=
u^\top
P^\top
\left(
\mathcal G_{\mathrm{transport}}^M
+
\Delta\mathcal G_{\mathrm{gating}}^M
\right)
P
u \\
&=
u^\top
P^\top
\mathcal G_{\mathrm{transport}}^M
P
u
+
u^\top
P^\top
\Delta\mathcal G_{\mathrm{gating}}^M
P
u \\
&=
-\,u^\top C u
+
u^\top R u \\
&=
u^\top(R-C)u,
\end{aligned}
\label{eq:app_full_growth_projected}
\end{equation}
where
\(
R=P^\top\Delta\mathcal G_{\mathrm{gating}}^M P
\). Thus, a gating-induced regime switch exists in \(\mathcal S\) if and only if
there exists \(u\neq0\) such that
\begin{equation}
\begin{aligned}
u^\top(R-C)u
&>
0 \\
\Longleftrightarrow\qquad
u^\top R u-u^\top C u
&>
0 \\
\Longleftrightarrow\qquad
u^\top R u
&>
u^\top C u.
\end{aligned}
\label{eq:app_switch_quadratic_inequality}
\end{equation}

Because \(C\succ0\), we have \(u^\top C u>0\) for every nonzero \(u\).
Therefore, dividing both sides of
Eq.~\eqref{eq:app_switch_quadratic_inequality} by \(u^\top C u\) preserves
the inequality and gives
\begin{equation}
\frac{u^\top R u}{u^\top C u}>1.
\label{eq:app_generalized_rayleigh_switch}
\end{equation}
The regime-switch problem is therefore equivalent to asking whether the
generalized Rayleigh quotient associated with the pair \((R,C)\) can exceed
one. We next transform this generalized quotient into an ordinary Rayleigh
quotient. Since \(C\succ0\), the symmetric matrices \(C^{1/2}\) and
\(C^{-1/2}\) exist and are nonsingular. For simplicity, we introduce the change of coordinates
\begin{equation}
w=C^{1/2}u,
\qquad
u=C^{-1/2}w.
\label{eq:app_whitening_change_coordinates}
\end{equation}
Because \(C^{1/2}\) is invertible, \(u\neq0\) if and only if \(w\neq0\). Under this change of coordinates, the denominator in
Eq.~\eqref{eq:app_generalized_rayleigh_switch} becomes
\begin{equation}
\begin{aligned}
u^\top C u
&=
\left(C^{-1/2}w\right)^\top
C
\left(C^{-1/2}w\right) \\
&=
w^\top
\left(C^{-1/2}\right)^\top
C
C^{-1/2}
w \\
&=
w^\top
C^{-1/2}
C
C^{-1/2}
w \\
&=
w^\top I w \\
&=
w^\top w,
\end{aligned}
\label{eq:app_whitened_denominator}
\end{equation}
where we used the symmetry of \(C^{-1/2}\). Similarly, the numerator becomes
\begin{equation}
\begin{aligned}
u^\top R u
&=
\left(C^{-1/2}w\right)^\top
R
\left(C^{-1/2}w\right) \\
&=
w^\top
\left(C^{-1/2}\right)^\top
R
C^{-1/2}
w \\
&=
w^\top
C^{-1/2}
R
C^{-1/2}
w.
\end{aligned}
\label{eq:app_whitened_numerator}
\end{equation}

Substituting
Eqs.~\eqref{eq:app_whitened_denominator} and
\eqref{eq:app_whitened_numerator} into
Eq.~\eqref{eq:app_generalized_rayleigh_switch} yields
\begin{equation}
\begin{aligned}
\frac{u^\top R u}{u^\top C u}
&=
\frac{
w^\top C^{-1/2}RC^{-1/2}w
}{
w^\top w
} \\
&>1.
\end{aligned}
\label{eq:app_whitened_rayleigh_quotient}
\end{equation}

Let
\begin{equation}
\mathcal W
:=
C^{-1/2}RC^{-1/2}.
\label{eq:app_switching_operator}
\end{equation}
Since both
\(\mathcal G_{\mathrm{full}}^M\) and
\(\mathcal G_{\mathrm{transport}}^M\) are symmetric,
\(\Delta\mathcal G_{\mathrm{gating}}^M\) is symmetric, and therefore
\(R=P^\top\Delta\mathcal G_{\mathrm{gating}}^M P\) is symmetric. Hence
\(\mathcal W\) is also symmetric. Equation~\eqref{eq:app_whitened_rayleigh_quotient} therefore shows that a
regime switch exists if and only if
\begin{equation}
\exists\,w\neq0
\quad\text{such that}\quad
\frac{w^\top\mathcal W w}{w^\top w}>1.
\label{eq:app_exist_whitened_direction}
\end{equation}
By the Rayleigh--Ritz characterization for symmetric matrices,
\begin{equation}
\begin{aligned}
\max_{w\neq0}
\frac{w^\top\mathcal W w}{w^\top w}
&=
\lambda_{\max}(\mathcal W) \\
&=
\lambda_{\max}
\left(
C^{-1/2}RC^{-1/2}
\right).
\end{aligned}
\label{eq:app_rayleigh_ritz_switch}
\end{equation}
Consequently, we have
\begin{equation}
\begin{aligned}
\exists\,v\in\mathcal S\setminus\{0\}
\ \text{undergoing a regime switch}
&\Longleftrightarrow
\exists\,w\neq0:
\frac{w^\top\mathcal W w}{w^\top w}>1 \\
&\Longleftrightarrow
\max_{w\neq0}
\frac{w^\top\mathcal W w}{w^\top w}>1 \\
&\Longleftrightarrow
\lambda_{\max}(\mathcal W)>1 \\
&\Longleftrightarrow
\lambda_{\max}
\left(
C^{-1/2}RC^{-1/2}
\right)>1.
\end{aligned}
\end{equation}
This proves the necessary and sufficient condition.
\end{proof}

It remains to verify the explicit construction of a switching
direction. Let \(w_{\max}\) be a unit eigenvector of \(\mathcal W\)
associated with \(\lambda_{\max}(\mathcal W)>1\), and define
\[
v_{\mathrm{sw}}:=PC^{-1/2}w_{\max}.
\]
Since \(C\succ0\) and \(P\) has orthonormal columns,
\(v_{\mathrm{sw}}\neq0\). Its transport-only growth satisfies
\begin{align}
v_{\mathrm{sw}}^\top G_{\mathrm{transport}}^M v_{\mathrm{sw}}
&=-w_{\max}^\top w_{\max}\\
&=-1<0.
\end{align}
Meanwhile, its full growth is
\begin{align}
v_{\mathrm{sw}}^\top G_{\mathrm{full}}^M v_{\mathrm{sw}}
&=w_{\max}^\top(\mathcal W-I)w_{\max}\\
&=\lambda_{\max}(\mathcal W)-1>0.
\end{align}
Therefore, \(v_{\mathrm{sw}}\) satisfies both inequalities
in the definition of a gating-induced regime switch,
which completes the proof of the constructive claim.

The whitening by \(C^{-1/2}\) has a direct interpretation. Transport may
contract different directions in \(\mathcal S\) by very different amounts,
so comparing gating growth in its original coordinates would not account
for how difficult each direction is to reverse. The matrix
\begin{equation}
\mathcal W
:=
C^{-1/2}RC^{-1/2}
\label{eq:app_switching_operator}
\end{equation}
measures the gating-induced change after normalizing by this transport
contraction. Its threshold therefore has a fixed meaning: eigenvalues below
one correspond to directions in which gating is insufficient to overcome
transport, whereas an eigenvalue above one certifies a contraction-to-expansion
switch.

\subsection{Strength of Regime Switching}
\label{subsec:app_switch_strength}

The existence criterion in the previous subsection is binary: it determines
whether gating can reverse transport-induced contraction somewhere inside
\(\mathcal S\), but does not quantify how strong that reversal is. Two
complementary notions are useful for this purpose. The first measures how
strongly gating overcomes transport relative to the contraction it faces,
whereas the second measures how severe the reversed transport contraction is
in absolute terms. These notions distinguish a direction that is easy to flip
because transport is nearly neutral from one that remains switchable despite
substantial transport contraction.

\paragraph{Relative switching strength.}
For a nonzero direction \(v=Pu\in\mathcal S\), the gating-induced growth
correction is \(u^\top Ru\), while the magnitude of the transport contraction
is \(u^\top Cu\). Their ratio
\begin{equation}
\Gamma(v)
=
\frac{
v^\top\Delta\mathcal G_{\mathrm{gating}}^M v
}{
-
v^\top\mathcal G_{\mathrm{transport}}^M v
}
=
\frac{u^\top Ru}{u^\top Cu}
\label{eq:app_relative_switch_strength}
\end{equation}
measures how much gating-induced growth is obtained per unit of contraction
that must be overcome. Because every direction in \(\mathcal S\) is
transport-contractive, a direction \(v\in\mathcal S\) switches precisely when
\(\Gamma(v)>1\).

\paragraph{Corollary~\ref{cor:max_relative_switching} (restated): Maximum Relative Switching Strength.}
Let \(\mathcal W=C^{-1/2}RC^{-1/2}\), and let \(w_\star\) be a unit
eigenvector associated with \(\lambda_{\max}(\mathcal W)\). Then
\(v_{\mathrm{rel}}=PC^{-1/2}w_\star\) maximizes \(\Gamma(v)\) over all
nonzero \(v\in\mathcal S\), up to arbitrary nonzero scaling, and
\(\Gamma(v_{\mathrm{rel}})=\lambda_{\max}(\mathcal W)\). Consequently,
if \(\lambda_{\max}(\mathcal W)>1\), then \(v_{\mathrm{rel}}\) undergoes
a gating-induced regime switch.

\begin{proof}
Because the columns of \(P\) form an orthonormal basis of
\(\mathcal S\), every nonzero direction \(v\in\mathcal S\) can be
written uniquely as \(v=Pu\) for some \(u\neq0\). By the definitions
\(C=-P^\top\mathcal G_{\mathrm{transport}}^M P\) and
\(R=P^\top\Delta\mathcal G_{\mathrm{gating}}^M P\), the transport
contraction along \(v\) satisfies
\begin{equation}
\begin{aligned}
-v^\top\mathcal G_{\mathrm{transport}}^M v
&=
-(Pu)^\top
\mathcal G_{\mathrm{transport}}^M
(Pu) \\
&=
-u^\top
P^\top\mathcal G_{\mathrm{transport}}^M P
u \\
&=
u^\top C u,
\end{aligned}
\label{eq:app_relative_transport}
\end{equation}
while the gating-induced growth correction satisfies
\begin{equation}
\begin{aligned}
v^\top\Delta\mathcal G_{\mathrm{gating}}^M v
&=
(Pu)^\top
\Delta\mathcal G_{\mathrm{gating}}^M
(Pu) \\
&=
u^\top
P^\top\Delta\mathcal G_{\mathrm{gating}}^M P
u \\
&=
u^\top R u.
\end{aligned}
\label{eq:app_relative_gating}
\end{equation}
Therefore
\begin{equation}
\Gamma(v)
=
\frac{u^\top Ru}{u^\top Cu}.
\label{eq:app_relative_generalized_rayleigh}
\end{equation}

Since \(C\succ0\), introduce the whitened coordinate
\begin{equation}
w=C^{1/2}u,
\qquad
u=C^{-1/2}w.
\label{eq:app_relative_whitening}
\end{equation}
The denominator becomes
\begin{equation}
\begin{aligned}
u^\top C u
&=
(C^{-1/2}w)^\top C(C^{-1/2}w) \\
&=
w^\top C^{-1/2}CC^{-1/2}w \\
&=
w^\top w,
\end{aligned}
\label{eq:app_relative_whitened_denominator}
\end{equation}
and similarly
\begin{equation}
\begin{aligned}
u^\top R u
&=
(C^{-1/2}w)^\top R(C^{-1/2}w) \\
&=
w^\top C^{-1/2}RC^{-1/2}w \\
&=
w^\top\mathcal W w.
\end{aligned}
\label{eq:app_relative_whitened_numerator}
\end{equation}
Hence
\begin{equation}
\Gamma(v)
=
\frac{
w^\top\mathcal W w
}{
w^\top w
}.
\label{eq:app_relative_standard_rayleigh}
\end{equation}

The matrix \(\mathcal W=C^{-1/2}RC^{-1/2}\) is symmetric. Therefore,
by the Rayleigh--Ritz theorem,
\begin{equation}
\begin{aligned}
\max_{v\in\mathcal S\setminus\{0\}}
\Gamma(v)
&=
\max_{u\neq0}
\frac{u^\top Ru}{u^\top Cu} \\
&=
\max_{w\neq0}
\frac{
w^\top\mathcal W w
}{
w^\top w
} \\
&=
\lambda_{\max}(\mathcal W).
\end{aligned}
\label{eq:app_relative_rayleigh_ritz}
\end{equation}

Let \(w_\star\) be a unit eigenvector associated with
\(\lambda_{\max}(\mathcal W)\). This vector is optimal in the whitened
coordinate system, rather than directly in the original representation
space. To recover the corresponding representation direction, we undo the
two coordinate transformations introduced above. First, based on equation \ref{eq:app_whitening_change_coordinates}, the whitening transformation
\(w=C^{1/2}u\) implies
\begin{equation}
u_\star
=
C^{-1/2}w_\star.
\label{eq:app_relative_optimal_subspace_coordinate}
\end{equation}
Thus, \(u_\star\) is the maximizing direction expressed in the coordinates
of the transport-contractive subspace \(\mathcal S\). Since the columns of
\(P\) form an orthonormal basis of \(\mathcal S\), a coordinate vector
\(u\) corresponds to the representation direction \(v=Pu\). Mapping
\(u_\star\) back to the original representation space therefore gives
\begin{equation}
\begin{aligned}
v_{\mathrm{rel}}
&=
Pu_\star \\
&=
PC^{-1/2}w_\star.
\end{aligned}
\label{eq:app_relative_optimal_direction}
\end{equation}
Hence \(w_\star\) is the optimal direction after normalizing by transport
contraction, \(C^{-1/2}w_\star\) is the same direction in the coordinates
of \(\mathcal S\), and \(PC^{-1/2}w_\star\) is the corresponding direction
in the original representation space. Substituting \(w_\star\) into
\eqref{eq:app_relative_standard_rayleigh} then yields
\begin{equation}
\begin{aligned}
\Gamma(v_{\mathrm{rel}})
&=
\frac{
w_\star^\top\mathcal W w_\star
}{
w_\star^\top w_\star
} \\
&=
\lambda_{\max}(\mathcal W),
\end{aligned}
\label{eq:app_relative_optimal_value}
\end{equation}
where the final equality follows because
\(\mathcal W w_\star
=
\lambda_{\max}(\mathcal W)w_\star\)
and \(\|w_\star\|_2=1\). Finally, if \(\lambda_{\max}(\mathcal W)>1\), then
\(\Gamma(v_{\mathrm{rel}})>1\). Since
\(-v_{\mathrm{rel}}^\top
\mathcal G_{\mathrm{transport}}^M
v_{\mathrm{rel}}>0\),
\begin{equation}
\begin{aligned}
\Gamma(v_{\mathrm{rel}})>1
&\Longrightarrow
v_{\mathrm{rel}}^\top
\Delta\mathcal G_{\mathrm{gating}}^M
v_{\mathrm{rel}}
>
-
v_{\mathrm{rel}}^\top
\mathcal G_{\mathrm{transport}}^M
v_{\mathrm{rel}} \\
&\Longrightarrow
v_{\mathrm{rel}}^\top
\mathcal G_{\mathrm{full}}^M
v_{\mathrm{rel}}
>0.
\end{aligned}
\label{eq:app_relative_switch_completion}
\end{equation}
At the same time,
\(v_{\mathrm{rel}}^\top
\mathcal G_{\mathrm{transport}}^M
v_{\mathrm{rel}}<0\)
because \(v_{\mathrm{rel}}\in\mathcal S\). Therefore
\(v_{\mathrm{rel}}\) undergoes a gating-induced regime switch.
\end{proof}

Relative strength is useful because it compares directions on a common scale
and admits an exact spectral optimum. However, it can become large when
transport contraction is very small. For example, a direction with transport
growth \(-0.001\) and gating correction \(1\) has
\(\Gamma(v)=1000\), even though only a very weak contraction must be
reversed. Relative strength therefore does not by itself quantify how severe
the underlying transport contraction is.

\paragraph{Absolute reversal severity.}
Relative switching strength identifies directions for which gating is large
compared with the transport contraction that it must overcome. This relative
comparison is useful, but it can favor directions that are already almost
neutral under transport. A complementary question is therefore whether
gating can reverse a direction that transport suppresses \emph{strongly}.

For a transport-contractive direction \(v\), the squared \(M\)-energy removed
by transport is
\(-v^\top\mathcal G_{\mathrm{transport}}^M v>0\).
Its raw magnitude, however, scales quadratically with \(v\), so it cannot be
compared meaningfully across directions of different size. The natural
quantity is instead the fraction of the incoming energy that transport
removes:
\begin{equation}
\tau(v)
=
\frac{
-v^\top\mathcal G_{\mathrm{transport}}^M v
}{
v^\top Mv
}.
\label{eq:app_absolute_switch_severity_ratio}
\end{equation}

This quantity has a direct interpretation in terms of the transported
representation. Recall that
\(\mathcal G_{\mathrm{transport}}^M
=
J_{\mathrm{transport}}^\top M J_{\mathrm{transport}}-M\).
Therefore,
\begin{equation}
\begin{aligned}
-v^\top\mathcal G_{\mathrm{transport}}^M v
&=
-v^\top
\left(
J_{\mathrm{transport}}^\top
M
J_{\mathrm{transport}}
-
M
\right)v \\
&=
v^\top Mv
-
v^\top
J_{\mathrm{transport}}^\top
M
J_{\mathrm{transport}}v \\
&=
\|v\|_M^2
-
\|J_{\mathrm{transport}}v\|_M^2.
\end{aligned}
\label{eq:app_absolute_removed_energy}
\end{equation}
Substituting this identity into
\eqref{eq:app_absolute_switch_severity_ratio} gives
\begin{equation}
\boxed{
\tau(v)
=
1-
\frac{
\|J_{\mathrm{transport}}v\|_M^2
}{
\|v\|_M^2
}.
}
\label{eq:app_absolute_switch_severity}
\end{equation}
Hence \(\tau(v)\) is exactly the fraction of incoming squared
\(M\)-energy removed by transport. Because \(v\) is
transport-contractive, \(0<\tau(v)\le1\). Values near zero correspond
to weak suppression, whereas \(\tau(v)\) near one means that transport
removes almost all of the incoming energy. For example,
\(\tau(v)=0.8\) means that only \(20\%\) of the incoming squared
\(M\)-energy remains after transport.

We are interested specifically in directions for which this contraction is
subsequently reversed by state-dependent gating. The strongest such reversal
is characterized by
\begin{equation}
\tau_{\max}
=
\begin{cases}
\displaystyle
\sup_{\substack{
v\in\mathcal S\setminus\{0\}\\
v^\top\mathcal G_{\mathrm{full}}^M v>0
}}
\tau(v),
&
\text{if a switching direction exists in }\mathcal S,\\[3mm]
0,
&
\text{otherwise}.
\end{cases}
\label{eq:app_max_absolute_switch_severity}
\end{equation}
Since every nonzero direction in \(\mathcal S\) is already
transport-contractive, the remaining constraint
\(v^\top\mathcal G_{\mathrm{full}}^M v>0\) selects exactly the directions
that undergo a regime switch. Thus, \(\tau_{\max}\) asks how severe the
transport contraction can be while still being overturned into expansion by
the full state-dependent dynamics.

Unlike the relative objective \(\Gamma(v)\), this optimization does not
reduce directly to a single generalized Rayleigh quotient. The direction
must simultaneously exhibit sufficiently strong transport contraction and
positive full growth. Rather than optimizing these two requirements jointly,
we first isolate directions whose transport contraction is guaranteed to
exceed a prescribed severity level, and then ask whether any direction in
that region switches.

For \(v=Pu\in\mathcal S\), the removed transport energy is
\(u^\top Cu\). Its incoming energy is
\(v^\top Mv=u^\top(P^\top MP)u\). Let
\(B:=P^\top MP\succ0\). Then
\begin{equation}
\tau(Pu)
=
\frac{u^\top Cu}{u^\top Bu}.
\label{eq:app_absolute_generalized_rayleigh}
\end{equation}
Absolute contraction severity is therefore a generalized Rayleigh quotient
for the pair \((C,B)\). This observation provides a natural coordinate
system in which transport modes are ordered by how strongly they contract.

Let \((\eta_i,q_i)\) satisfy
\(Cq_i=\eta_iBq_i\), with the generalized eigenvectors chosen
\(B\)-orthonormally so that \(q_i^\top Bq_j=\delta_{ij}\), and order the
eigenvalues as \(\eta_1\ge\cdots\ge\eta_r\). Along an individual generalized
eigenvector,
\[
\tau(Pq_i)
=
\frac{q_i^\top Cq_i}{q_i^\top Bq_i}
=
\eta_i,
\]
so \(\eta_i\) has a direct meaning: it is the absolute transport severity of
that mode. Moreover,
\[
B-C
=
P^\top J_{\mathrm{transport}}^\top
M
J_{\mathrm{transport}}P
\succeq0,
\]
and hence \(C\preceq B\). Together with \(C\succ0\), this implies
\(0<\eta_i\le1\).

Now fix a desired severity level \(\tau_0\in(0,1]\). We retain only the
generalized transport modes whose contraction is at least this threshold:
\begin{equation}
\mathcal S_{\tau_0}
=
P\,
\operatorname{span}
\left\{
q_i:\eta_i\ge\tau_0
\right\}
\subseteq\mathcal S.
\label{eq:app_absolute_threshold_subspace}
\end{equation}
The important point is not merely that each selected basis direction has
severity at least \(\tau_0\), but that \emph{every} nonzero direction in
their span does as well. Thus, \(\mathcal S_{\tau_0}\) is a region in which
transport contraction is uniformly strong. For example,
\(\tau_0=0.8\) restricts the analysis to a subspace in which transport removes
at least \(80\%\) of the incoming squared \(M\)-energy along every direction.

The remaining question is whether gating can reverse any direction inside
this strongly contractive subspace. Let \(P_{\tau_0}\) be an orthonormal
basis of \(\mathcal S_{\tau_0}\), assuming that the subspace is nontrivial,
and restrict the transport and gating-growth operators to it:
\begin{equation}
C_{\tau_0}
=
-
P_{\tau_0}^\top
\mathcal G_{\mathrm{transport}}^M
P_{\tau_0},
\qquad
R_{\tau_0}
=
P_{\tau_0}^\top
\Delta\mathcal G_{\mathrm{gating}}^M
P_{\tau_0}.
\label{eq:app_absolute_restricted_CR}
\end{equation}
Because
\(\mathcal S_{\tau_0}\subseteq\mathcal S\),
we still have \(C_{\tau_0}\succ0\). The same whitening used in the
existence theorem therefore yields the restricted switching operator
\begin{equation}
\mathcal W_{\tau_0}
=
C_{\tau_0}^{-1/2}
R_{\tau_0}
C_{\tau_0}^{-1/2}.
\label{eq:app_absolute_restricted_W}
\end{equation}

\paragraph{Theorem~\ref{thm:absolute_reversal_severity} (restated): Certified Absolute Reversal Severity.}
Fix $\tau_0\in(0,1]$ and suppose $\mathcal S_{\tau_0}\neq\{0\}$.
Every nonzero direction in $\mathcal S_{\tau_0}$ has absolute
transport severity at least $\tau_0$.
Moreover, $\lambda_{\max}(\mathcal W_{\tau_0})>1$ if and only if
$\mathcal S_{\tau_0}$ contains a gating-induced regime-switch direction.
Consequently, whenever $\lambda_{\max}(\mathcal W_{\tau_0})>1$,
the strongest absolute reversal satisfies $\tau_{\max}\geq\tau_0$.

\begin{proof}
We first verify that the severity threshold holds uniformly over
\(\mathcal S_{\tau_0}\). Any nonzero
\(v\in\mathcal S_{\tau_0}\) can be written as \(v=Pu\), where
\[
u
=
\sum_{i:\eta_i\ge\tau_0}
a_iq_i
\]
for coefficients not all zero. Since the generalized eigenvectors are
\(B\)-orthonormal and satisfy \(Cq_i=\eta_iBq_i\),
\begin{equation}
\begin{aligned}
u^\top Bu
&=
\sum_{i:\eta_i\ge\tau_0}a_i^2, \\
u^\top Cu
&=
\sum_{i:\eta_i\ge\tau_0}\eta_i a_i^2
\ge
\tau_0
\sum_{i:\eta_i\ge\tau_0}a_i^2.
\end{aligned}
\label{eq:app_absolute_threshold_energy}
\end{equation}
Therefore,
\begin{equation}
\tau(v)
=
\frac{u^\top Cu}{u^\top Bu}
=
\frac{
\sum_i\eta_i a_i^2
}{
\sum_i a_i^2
}
\ge\tau_0.
\label{eq:app_absolute_threshold_guarantee}
\end{equation}
Thus every nonzero direction in
\(\mathcal S_{\tau_0}\) is guaranteed to experience at least the prescribed
transport contraction.

It remains to determine whether any such direction becomes expanding after
gating is restored. Write \(v=P_{\tau_0}a\). Its transport-only and full
growths are
\begin{equation}
v^\top\mathcal G_{\mathrm{transport}}^M v
=
-a^\top C_{\tau_0}a,
\qquad
v^\top\mathcal G_{\mathrm{full}}^M v
=
a^\top(R_{\tau_0}-C_{\tau_0})a.
\label{eq:app_absolute_restricted_growth}
\end{equation}
Since \(C_{\tau_0}\succ0\), transport contracts every nonzero direction in
this subspace. A regime switch therefore exists precisely when some
\(a\neq0\) satisfies
\[
a^\top(R_{\tau_0}-C_{\tau_0})a>0,
\]
or, equivalently,
\[
\frac{
a^\top R_{\tau_0}a
}{
a^\top C_{\tau_0}a
}
>1.
\]

To evaluate the largest value of this ratio, set
\(y=C_{\tau_0}^{1/2}a\), so that
\(a=C_{\tau_0}^{-1/2}y\). Then
\[
\frac{
a^\top R_{\tau_0}a
}{
a^\top C_{\tau_0}a
}
=
\frac{
y^\top
\mathcal W_{\tau_0}
y
}{
y^\top y
}.
\]
Rayleigh--Ritz therefore gives
\begin{equation}
\max_{a\neq0}
\frac{
a^\top R_{\tau_0}a
}{
a^\top C_{\tau_0}a
}
=
\lambda_{\max}(\mathcal W_{\tau_0}).
\label{eq:app_absolute_spectral_certificate}
\end{equation}
Hence a switching direction exists inside
\(\mathcal S_{\tau_0}\) if and only if
\(\lambda_{\max}(\mathcal W_{\tau_0})>1\).

Whenever this condition holds, the resulting direction is both switching and
contained in \(\mathcal S_{\tau_0}\). The first part of the proof therefore
guarantees \(\tau(v)\ge\tau_0\), and consequently
\[
\tau_{\max}
\ge
\tau(v)
\ge
\tau_0.
\]
This proves the certificate.
\end{proof}

The theorem converts absolute reversal severity into a sequence of
increasingly demanding spectral tests. If
\(\tau_2>\tau_1\), then
\(\mathcal S_{\tau_2}\subseteq\mathcal S_{\tau_1}\): raising the threshold
removes weaker transport modes and asks whether switching persists after
progressively stronger contraction. The largest threshold certified in this
way is
\begin{equation}
\tau_{\mathrm{cert}}
=
\sup
\left\{
\tau_0\in(0,1]:
\mathcal S_{\tau_0}\neq\{0\},
\;
\lambda_{\max}(\mathcal W_{\tau_0})>1
\right\},
\label{eq:app_certified_absolute_severity}
\end{equation}
with \(\tau_{\mathrm{cert}}=0\) if no threshold is certified. The theorem
immediately implies
\begin{equation}
\boxed{
\tau_{\max}\ge\tau_{\mathrm{cert}}.
}
\label{eq:app_absolute_certified_lower_bound}
\end{equation}

The inequality is intentional. The condition
\(\tau(v)\ge\tau_0\) describes a quadratic region of representation space
and need not itself form a linear subspace. Our construction instead extracts
a linear subspace on which the severity threshold holds uniformly, and then
applies the exact spectral switch criterion inside that subspace. Consequently,
\(\tau_{\mathrm{cert}}\) is a certified lower bound on the strongest absolute
reversal, rather than necessarily its exact value.

When \(M=I\), the orthonormality of \(P\) gives \(B=I\). The generalized
eigenproblem then reduces to \(Cq_i=\eta_iq_i\), and
\(\mathcal S_{\tau_0}\) is simply the span of the transport-contraction
eigenvectors with eigenvalues at least \(\tau_0\). This is the
nested-subspace form used in our experiments.

\paragraph{Interpretation.}
Relative and absolute switching strength answer different questions.
Suppose \(v_1\) has transport growth \(-0.2\) and gating correction \(1.0\),
while \(v_2\) has transport growth \(-0.8\) and gating correction \(2.0\),
with both directions normalized to the same incoming \(M\)-energy. Both
directions switch, with full growth \(0.8\) and \(1.2\), respectively.
Relative strength favors \(v_1\), because
\(\Gamma(v_1)=5>\Gamma(v_2)=2.5\): gating is larger relative to what must be
overcome. Absolute severity instead favors \(v_2\), because transport removes
a substantially larger fraction of its incoming energy before the reversal.

Thus, \(\lambda_{\max}(\mathcal W)\) characterizes the most favorable
gating-to-transport competition, whereas \(\tau_{\max}\) asks how severe
transport contraction can be while still being overturned. The certified
quantity \(\tau_{\mathrm{cert}}\) makes the latter operational: a large value
shows that regime switching persists even after restricting attention to
directions that transport suppresses strongly, rather than arising only near
the transport-neutral boundary.

\subsection{No-Switch Conditions and Effective Switching Dimension}
\label{subsec:app_no_switch}

The previous results characterize whether a regime switch exists and how
strong such a reversal can be. We now ask two complementary questions. First,
when is gating provably unable to overturn transport contraction anywhere in
\(\mathcal S\)? Second, when switching is possible, is it confined to a
low-dimensional portion of representation space or supported over a broader
family of directions? Both questions are determined by the spectrum of the
whitened switching operator
\(\mathcal W=C^{-1/2}RC^{-1/2}\).

Recall that every nonzero \(v\in\mathcal S\) can be written as \(v=Pu\), and
that its full growth satisfies
\begin{equation}
\begin{aligned}
v^\top\mathcal G_{\mathrm{full}}^M v
&=
u^\top(R-C)u.
\end{aligned}
\label{eq:app_full_growth_RC}
\end{equation}
Thus, a regime switch occurs precisely when the quadratic form \(R-C\) is
positive along at least one direction.

\paragraph{Theorem~\ref{thm:effective_switching_dimension} (restated): No-Switch Criterion and Effective Switching Dimension.}
Let \(\mathcal W=C^{-1/2}RC^{-1/2}\), with eigenvalues
\(\lambda_1(\mathcal W)\ge\cdots\ge\lambda_r(\mathcal W)\). There is no
gating-induced regime switch anywhere in \(\mathcal S\) if and only if
\(\mathcal W\preceq I\), equivalently
\(\lambda_{\max}(\mathcal W)\le1\), or equivalently \(R\preceq C\).
Moreover, if
\(d_{\mathrm{sw}}:=\#\{i:\lambda_i(\mathcal W)>1\}\), then
\(d_{\mathrm{sw}}\) is the largest dimension of a linear subspace of
\(\mathcal S\) in which every nonzero direction undergoes a regime switch.

\begin{proof}
For any nonzero \(v\in\mathcal S\), write \(v=Pu\) with \(u\neq0\).
Using the definitions of \(C\) and \(R\), the full directional growth becomes
\begin{equation}
\begin{aligned}
v^\top\mathcal G_{\mathrm{full}}^M v
&=
v^\top
\left(
\mathcal G_{\mathrm{transport}}^M
+
\Delta\mathcal G_{\mathrm{gating}}^M
\right)
v \\
&=
(Pu)^\top
\mathcal G_{\mathrm{transport}}^M
(Pu)
+
(Pu)^\top
\Delta\mathcal G_{\mathrm{gating}}^M
(Pu) \\
&=
u^\top
P^\top\mathcal G_{\mathrm{transport}}^M P
u
+
u^\top
P^\top\Delta\mathcal G_{\mathrm{gating}}^M P
u \\
&=
-u^\top C u
+
u^\top R u \\
&=
u^\top(R-C)u.
\end{aligned}
\label{eq:app_no_switch_growth_expand}
\end{equation}
Because every direction in \(\mathcal S\) is already transport-contractive,
a regime switch exists exactly when there is some \(u\neq0\) such that
\(u^\top(R-C)u>0\). Therefore, no switch occurs anywhere in
\(\mathcal S\) exactly when
\begin{equation}
\begin{aligned}
u^\top(R-C)u
&\le0
\qquad
\text{for every }u\neq0,
\end{aligned}
\end{equation}
which is equivalent to
\begin{equation}
R-C\preceq0,
\qquad\text{or equivalently}\qquad
R\preceq C.
\label{eq:app_no_switch_RC}
\end{equation}

Since \(C\succ0\), we may again whiten the transport contraction by writing
\(w=C^{1/2}u\), so that \(u=C^{-1/2}w\). Substituting this relation gives
\begin{equation}
\begin{aligned}
u^\top(R-C)u
&=
(C^{-1/2}w)^\top
(R-C)
(C^{-1/2}w) \\
&=
w^\top
C^{-1/2}RC^{-1/2}
w
-
w^\top
C^{-1/2}CC^{-1/2}
w \\
&=
w^\top\mathcal W w
-
w^\top w \\
&=
w^\top(\mathcal W-I)w.
\end{aligned}
\label{eq:app_no_switch_whitened}
\end{equation}
Hence, no switch exists if and only if
\begin{equation}
\begin{aligned}
w^\top(\mathcal W-I)w
&\le0
\qquad
\text{for every }w\neq0.
\end{aligned}
\end{equation}
Because \(\mathcal W-I\) is symmetric, this condition is equivalent to
\begin{equation}
\mathcal W-I\preceq0,
\qquad\text{or equivalently}\qquad
\mathcal W\preceq I.
\end{equation}
By the spectral theorem,
\(\mathcal W\preceq I\) holds exactly when every eigenvalue of
\(\mathcal W\) is at most one, which yields
\begin{equation}
\lambda_{\max}(\mathcal W)\le1.
\end{equation}
This proves the no-switch equivalences.

We next characterize the largest linear subspace on which switching holds
uniformly for every nonzero direction. Let
\(w_1,\ldots,w_r\) be an orthonormal eigenbasis of \(\mathcal W\), with
\(\mathcal Ww_i=\lambda_i(\mathcal W)w_i\), and define
\begin{equation}
\mathcal E_{\mathrm{sw}}
=
\operatorname{span}
\left\{
w_i:
\lambda_i(\mathcal W)>1
\right\}.
\label{eq:app_switch_eigenspace}
\end{equation}
Its dimension is
\(\dim(\mathcal E_{\mathrm{sw}})=d_{\mathrm{sw}}\). For any nonzero
\(w\in\mathcal E_{\mathrm{sw}}\), write
\begin{equation}
w
=
\sum_{i:\lambda_i(\mathcal W)>1}
a_iw_i,
\end{equation}
with coefficients not all zero. Then
\begin{equation}
\begin{aligned}
w^\top(\mathcal W-I)w
&=
\left(
\sum_i a_iw_i
\right)^\top
(\mathcal W-I)
\left(
\sum_j a_jw_j
\right) \\
&=
\sum_{i:\lambda_i(\mathcal W)>1}
a_i^2
\left(
\lambda_i(\mathcal W)-1
\right) \\
&>
0.
\end{aligned}
\label{eq:app_switch_positive_subspace}
\end{equation}
Therefore every nonzero vector in
\(\mathcal E_{\mathrm{sw}}\) produces positive full growth after mapping back
to \(\mathcal S\). The corresponding representation subspace is
\begin{equation}
\mathcal V_{\mathrm{sw}}
=
PC^{-1/2}\mathcal E_{\mathrm{sw}}
=
\left\{
PC^{-1/2}w:
w\in\mathcal E_{\mathrm{sw}}
\right\},
\label{eq:app_representation_switch_subspace}
\end{equation}
and the invertibility of \(C^{-1/2}\), together with the full column rank of
\(P\), implies
\begin{equation}
\dim(\mathcal V_{\mathrm{sw}})
=
\dim(\mathcal E_{\mathrm{sw}})
=
d_{\mathrm{sw}}.
\end{equation}
Hence every nonzero direction in the
\(d_{\mathrm{sw}}\)-dimensional subspace
\(\mathcal V_{\mathrm{sw}}\) undergoes a regime switch.

It remains to show that no larger all-switching linear subspace can exist.
The symmetric matrix \(\mathcal W-I\) has exactly
\(d_{\mathrm{sw}}\) positive eigenvalues, since
\begin{equation}
\lambda_i(\mathcal W-I)
=
\lambda_i(\mathcal W)-1.
\end{equation}
By the inertia theorem, the maximum dimension of a linear subspace on which
the quadratic form
\(w^\top(\mathcal W-I)w\) is positive for every nonzero vector is exactly the
number of positive eigenvalues of \(\mathcal W-I\). Therefore no linear
subspace of dimension greater than \(d_{\mathrm{sw}}\) can have the property
that every nonzero direction switches. This proves that
\(d_{\mathrm{sw}}\) is the effective switching dimension.
\end{proof}

\paragraph{Interpretation.}
The effective switching dimension measures the breadth of the regime-switching
phenomenon rather than the number of individual switching directions.
Specifically, \(d_{\mathrm{sw}}\) is the largest dimension of a linear
subspace in which \emph{every} nonzero direction is converted from transport
contraction into full-dynamics expansion. This distinction is important:
individual switching directions may exist outside
\(\mathcal V_{\mathrm{sw}}\), but arbitrary linear combinations of such
directions need not themselves switch. Thus, \(d_{\mathrm{sw}}\) should not
be interpreted as the number of linearly independent switching vectors.

When \(d_{\mathrm{sw}}=0\), gating cannot overturn transport contraction
anywhere in \(\mathcal S\). When \(d_{\mathrm{sw}}=1\), there exists an
all-switching one-dimensional subspace, but there is no two-dimensional
linear subspace in which every nonzero direction is guaranteed to switch.
Larger values indicate increasingly broad switching behavior: gating can
reverse transport contraction uniformly over a higher-dimensional region of
representation space rather than only along specially selected directions.

This distinction also has a practical interpretation. A small
\(d_{\mathrm{sw}}\) indicates that routing-induced reversal is geometrically
narrow, so an intervention or perturbation must be well aligned with a
specific portion of representation space to benefit from the effect. A large
\(d_{\mathrm{sw}}\), in contrast, indicates that an entire
higher-dimensional family of perturbations is guaranteed to undergo the same
contraction-to-expansion reversal. The quantity can therefore be used as a
measure of how broadly routing-induced sensitivity or controllability is
distributed across representation space, and can support comparisons across
layers, heads, or model architectures.

A simpler sufficient no-switch condition can also be obtained without
forming the whitened operator explicitly. Since
\(u^\top Ru\le\lambda_{\max}(R)\|u\|_2^2\) and
\(u^\top Cu\ge\lambda_{\min}(C)\|u\|_2^2\), the condition
\(\lambda_{\max}(R)\le\lambda_{\min}(C)\) implies
\(u^\top(R-C)u\le0\) for every \(u\), and therefore rules out regime
switching throughout \(\mathcal S\). This bound is generally more
conservative than the exact spectral criterion, but makes the competition
between the largest gating correction and the weakest transport contraction
explicit.

\subsection{Robustness to Operator Estimation Error}
\label{subsec:app_switch_robustness}

The previous results are stated for exact local operators. In practice,
however, Jacobians and growth operators may be obtained through numerical
differentiation, finite samples, or approximate Jacobian-vector products.
We therefore ask whether the predicted regime switch remains valid when these
operators are only approximately known. The relevant quantity is the margin
to the switching boundary: conclusions far from zero directional growth, or
far from the spectral threshold \(1\), remain stable under sufficiently small
operator perturbations.

\paragraph{Robustness of an individual switching direction.}
Let
\(\widehat{\mathcal G}_{\mathrm{transport}}^M\) and
\(\widehat{\mathcal G}_{\mathrm{full}}^M\) denote estimated growth operators
satisfying
\[
\left\|
\widehat{\mathcal G}_{\mathrm{transport}}^M
-
\mathcal G_{\mathrm{transport}}^M
\right\|_2
\le \varepsilon_{\mathrm{t}},
\qquad
\left\|
\widehat{\mathcal G}_{\mathrm{full}}^M
-
\mathcal G_{\mathrm{full}}^M
\right\|_2
\le \varepsilon_{\mathrm{f}}.
\]

\begin{proposition}[Margin-based robustness of regime switching]
\label{prop:app_directional_robustness}
For any unit direction \(v\), if
\(v^\top\widehat{\mathcal G}_{\mathrm{transport}}^M v<-\varepsilon_{\mathrm{t}}\)
and
\(v^\top\widehat{\mathcal G}_{\mathrm{full}}^M v>\varepsilon_{\mathrm{f}}\),
then the exact operators satisfy
\(v^\top\mathcal G_{\mathrm{transport}}^M v<0\) and
\(v^\top\mathcal G_{\mathrm{full}}^M v>0\). Hence the estimated switch is
guaranteed to persist under the stated operator errors.
\end{proposition}

\begin{proof}
For the transport operator, write
\begin{equation}
\begin{aligned}
v^\top\mathcal G_{\mathrm{transport}}^M v
&=
v^\top\widehat{\mathcal G}_{\mathrm{transport}}^M v
+
v^\top
\left(
\mathcal G_{\mathrm{transport}}^M
-
\widehat{\mathcal G}_{\mathrm{transport}}^M
\right)
v \\
&\le
v^\top\widehat{\mathcal G}_{\mathrm{transport}}^M v
+
\left|
v^\top
\left(
\mathcal G_{\mathrm{transport}}^M
-
\widehat{\mathcal G}_{\mathrm{transport}}^M
\right)
v
\right| \\
&\le
v^\top\widehat{\mathcal G}_{\mathrm{transport}}^M v
+
\left\|
\mathcal G_{\mathrm{transport}}^M
-
\widehat{\mathcal G}_{\mathrm{transport}}^M
\right\|_2
\|v\|_2^2 \\
&\le
v^\top\widehat{\mathcal G}_{\mathrm{transport}}^M v
+
\varepsilon_{\mathrm{t}} \\
&<0.
\end{aligned}
\end{equation}
Similarly, for the full operator,
\begin{equation}
\begin{aligned}
v^\top\mathcal G_{\mathrm{full}}^M v
&=
v^\top\widehat{\mathcal G}_{\mathrm{full}}^M v
+
v^\top
\left(
\mathcal G_{\mathrm{full}}^M
-
\widehat{\mathcal G}_{\mathrm{full}}^M
\right)
v \\
&\ge
v^\top\widehat{\mathcal G}_{\mathrm{full}}^M v
-
\left|
v^\top
\left(
\mathcal G_{\mathrm{full}}^M
-
\widehat{\mathcal G}_{\mathrm{full}}^M
\right)
v
\right| \\
&\ge
v^\top\widehat{\mathcal G}_{\mathrm{full}}^M v
-
\left\|
\mathcal G_{\mathrm{full}}^M
-
\widehat{\mathcal G}_{\mathrm{full}}^M
\right\|_2
\|v\|_2^2 \\
&\ge
v^\top\widehat{\mathcal G}_{\mathrm{full}}^M v
-
\varepsilon_{\mathrm{f}} \\
&>0.
\end{aligned}
\end{equation}
Thus the direction remains contractive under transport-only dynamics and
expansive under the full dynamics.
\end{proof}

The proposition shows that a regime switch is reliable whenever both signs
are separated from zero by more than the corresponding estimation errors.
In particular, a strong reversal such as a substantially negative transport
growth followed by substantially positive full growth is intrinsically more
robust than a switch occurring arbitrarily close to the zero-growth boundary.

\paragraph{Robustness of the spectral switching criterion.}
The same margin principle applies to the whitened switching operator. Let
\(\widehat{\mathcal W}\) be an estimate of
\(\mathcal W=C^{-1/2}RC^{-1/2}\) satisfying
\[
\|\widehat{\mathcal W}-\mathcal W\|_2
\le \varepsilon_{\mathcal W}.
\]

\begin{theorem}[Robust spectral certification]
\label{thm:app_spectral_robustness}
Let the eigenvalues of \(\mathcal W\) and \(\widehat{\mathcal W}\) be ordered
nonincreasingly. Then
\(\left|\lambda_i(\widehat{\mathcal W})-\lambda_i(\mathcal W)\right|
\le\varepsilon_{\mathcal W}\) for every \(i\). Consequently,
\(\lambda_{\max}(\widehat{\mathcal W})>1+\varepsilon_{\mathcal W}\)
certifies the existence of a true regime switch, while
\(\lambda_{\max}(\widehat{\mathcal W})<1-\varepsilon_{\mathcal W}\)
certifies that no regime switch exists in \(\mathcal S\). Moreover, the true
effective switching dimension \(d_{\mathrm{sw}}\) satisfies
\[
\#\left\{i:
\lambda_i(\widehat{\mathcal W})>1+\varepsilon_{\mathcal W}
\right\}
\le
d_{\mathrm{sw}}
\le
\#\left\{i:
\lambda_i(\widehat{\mathcal W})>1-\varepsilon_{\mathcal W}
\right\}.
\]
\end{theorem}

\begin{proof}
Since both \(\mathcal W\) and \(\widehat{\mathcal W}\) are symmetric,
Weyl's eigenvalue perturbation inequality gives
\begin{equation}
\begin{aligned}
\left|
\lambda_i(\widehat{\mathcal W})
-
\lambda_i(\mathcal W)
\right|
&\le
\left\|
\widehat{\mathcal W}
-
\mathcal W
\right\|_2 \\
&\le
\varepsilon_{\mathcal W}
\end{aligned}
\end{equation}
for every \(i\). Therefore, if
\(\lambda_{\max}(\widehat{\mathcal W})>1+\varepsilon_{\mathcal W}\),
then
\begin{equation}
\begin{aligned}
\lambda_{\max}(\mathcal W)
&\ge
\lambda_{\max}(\widehat{\mathcal W})
-
\varepsilon_{\mathcal W} \\
&>
1,
\end{aligned}
\end{equation}
which, by the spectral switching criterion, guarantees the existence of a
true regime switch. Conversely, if
\(\lambda_{\max}(\widehat{\mathcal W})<1-\varepsilon_{\mathcal W}\), then
\begin{equation}
\begin{aligned}
\lambda_{\max}(\mathcal W)
&\le
\lambda_{\max}(\widehat{\mathcal W})
+
\varepsilon_{\mathcal W} \\
&<
1,
\end{aligned}
\end{equation}
so no regime switch can occur in \(\mathcal S\).

For the switching dimension, suppose
\(\lambda_i(\widehat{\mathcal W})>1+\varepsilon_{\mathcal W}\). Then
\begin{equation}
\begin{aligned}
\lambda_i(\mathcal W)
&\ge
\lambda_i(\widehat{\mathcal W})
-
\varepsilon_{\mathcal W} \\
&>
1,
\end{aligned}
\end{equation}
so every estimated eigenvalue above
\(1+\varepsilon_{\mathcal W}\) corresponds to a true switching mode.
This yields the lower bound
\begin{equation}
\#\left\{i:
\lambda_i(\widehat{\mathcal W})
>
1+\varepsilon_{\mathcal W}
\right\}
\le d_{\mathrm{sw}}.
\end{equation}
Conversely, if \(\lambda_i(\mathcal W)>1\), then
\begin{equation}
\begin{aligned}
\lambda_i(\widehat{\mathcal W})
&\ge
\lambda_i(\mathcal W)
-
\varepsilon_{\mathcal W} \\
&>
1-\varepsilon_{\mathcal W}.
\end{aligned}
\end{equation}
Hence every true switching mode must appear among the estimated eigenvalues
above \(1-\varepsilon_{\mathcal W}\), giving
\begin{equation}
d_{\mathrm{sw}}
\le
\#\left\{i:
\lambda_i(\widehat{\mathcal W})
>
1-\varepsilon_{\mathcal W}
\right\}.
\end{equation}
Combining the two bounds proves the result.
\end{proof}

The theorem gives a direct interpretation of spectral margin. Eigenvalues far
above \(1\) correspond to switching modes that are stable to substantial
operator error, whereas eigenvalues close to \(1\) are intrinsically
ambiguous under finite estimation accuracy. In particular, if no estimated
eigenvalue lies within
\([1-\varepsilon_{\mathcal W},1+\varepsilon_{\mathcal W}]\), then the
effective switching dimension is exactly identified despite the operator
perturbation.

\section{EXPERIMENTAL DETAILS AND ADDITIONAL RESULTS}
\label{sec:app_experiments}

The main experiments in Section~\ref{sec:experiments} establish the central
empirical observations: query gating has restricted geometric support, query
and key pathways can reverse transport contraction, and spectral operators
identify switching directions without direction-by-direction search. Here we
examine what those headline results do not fully resolve. We first clarify
the operating-point and measurement protocol, then use additional directional
counterfactuals to identify how a switch is produced. We subsequently examine
what a severity certificate does and does not establish, quantify the
interaction between the gating pathways in extreme directions, and report the
search budgets and projection sensitivities that limit the interpretation of
the experiments. Throughout, the reported properties are local to the tested
operating points and perturbation subspaces.

\subsection{Models, Datasets, and Operator Estimation}
\label{subsec:app_exp_setup}

\paragraph{Operating points and local operators.}
We use pretrained GPT-2 Medium, Gemma-2-2B, Llama-3.2-1B, and
Qwen2.5-0.5B with inputs from the WikiText-2 validation set. Detailed
mechanism and severity measurements use a selected operating point in layer
23 of GPT-2 Medium. The cross-architecture experiment instead probes several
layers of each model, as specified in
Section~\ref{subsec:app_exp_cross_arch}. We analyze residual self-attention
sublayers and construct transport, query-gating, and key-gating Jacobians
$J_T$, $J_Q$, and $J_K$ using the decomposition in
Appendix~\ref{app:attention_jacobian} and the architectural extensions in
Appendix~\ref{subsec:app_modern_blocks}. The common residual pathway is
included consistently in transport-only and full-attention comparisons.
Unless otherwise specified, $M=I$.

\paragraph{What is evaluated.}
For a perturbation direction $v$, we compare the quadratic growth forms
$g_T^M(v)$, $g_{T+Q}^M(v)$, $g_{T+K}^M(v)$, and $g_{\mathrm{full}}^M(v)$,
corresponding to the transport-only, transport-plus-query,
transport-plus-key, and complete Jacobians. Their signs specify switching
status and pathway sufficiency; differences between the values quantify how
far a direction lies from a growth-sign boundary. Analytical Jacobians are
checked against automatic differentiation and finite-difference or
finite-forward evaluations. Restricted growth and switching operators are
formed in the chosen perturbation subspace and analyzed by symmetric
eigendecomposition. Experiments were performed on an NVIDIA A100 GPU with
40\,GB of memory. These measurements concern local residual-attention
dynamics, not end-to-end changes in model predictions.

\subsection{Anisotropic Dynamics and Routing Mechanisms}
\label{subsec:app_exp_mechanisms}

\paragraph{What the attended representations actually contribute.}
The main paper tests the support condition for query gating in
Figure~\ref{fig:exp_query_support}. A separate issue is how to interpret the
attended representations from which that support is constructed. For the
selected head, we center the token representations at their
attention-weighted mean $\mu_r$ and visualize a three-dimensional
projection in Figure~\ref{fig:app_attended_variation}. The points are not
weighted equally: their marker sizes and colors encode the attention weights
$\alpha_{ri}$ that enter $\Sigma_r$. Consequently, the relevant geometry is
not merely the unweighted spread of visible points, but the weighted
variation of the attended representations in the original hidden space.

\begin{wrapfigure}{r}{0.46\columnwidth}
    \vspace{-8pt}
    \centering
    \includegraphics[width=\linewidth]{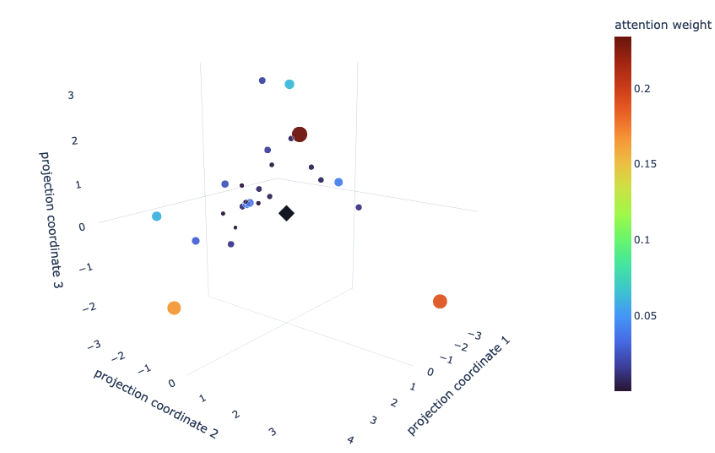}
    \caption{\textbf{Attended-representation geometry.}
    Three-dimensional projection of the centered representations
    $x_i-\mu_r$; marker size and color indicate attention weights
    $\alpha_{ri}$. The plot visualizes a selected head, not the full
    attended-variation subspace.}
    \label{fig:app_attended_variation}
    \vspace{-8pt}
\end{wrapfigure}

The attended-variation subspace at this operating point has dimension
$\dim(\mathcal U_r)=25$. Thus, Figure~\ref{fig:app_attended_variation}
shows only a three-dimensional view of a higher-dimensional object. In
particular, neither the apparent separation of points in this projection nor
the attention weights alone establish which input perturbations lie in the
null space of $J_r^{(q)}$. That question depends on the full operator
$J_r^{(q)}=W_{VO}\Sigma_rW_{QK}^{\top}$, as characterized by
Proposition~\ref{prop:query_gating_support} and tested in the main paper.

\paragraph{Insight.}
The attended-representation plot gives a concrete interpretation of the
covariance entering query gating, but it is not itself a test of query-gating
rank or response. This distinction matters: the attended tokens can exhibit
substantial visible variation while a given query perturbation still
produces no first-order output response, because the perturbation may miss
the transformed support or be annihilated by the output projection. Geometry
therefore constrains the \emph{available routing response}; the sign of
representation growth must be determined by its interaction with transport.
The next experiment examines that interaction through explicit pathway
counterfactuals.

\subsection{Empirical Regime Switching and Spectral Validation}
\label{subsec:app_exp_switching}

\paragraph{Disentangling the pathway conditions.}
The regime maps in Figure~\ref{fig:exp_regime_switching} establish that the
analytic and finite-forward classifications can agree on a selected slice.
Figure~\ref{fig:app_regime_switching_geometry} provides a more detailed
account of the classification itself. Its five columns progressively
introduce the growth-sign conditions on the same fixed three-dimensional
subspace. First, directions with $g_T^M(v)\geq0$ are excluded from the
transport-contraction question. Among the remaining directions,
$g_{\mathrm{full}}^M(v)\leq0$ identifies contraction that persists after
routing is restored. Switching directions satisfy both
$g_T^M(v)<0$ and $g_{\mathrm{full}}^M(v)>0$; the single-pathway
counterfactuals then distinguish the mechanism of that reversal.

\begin{figure*}[t]
    \centering
    \includegraphics[width=\textwidth]{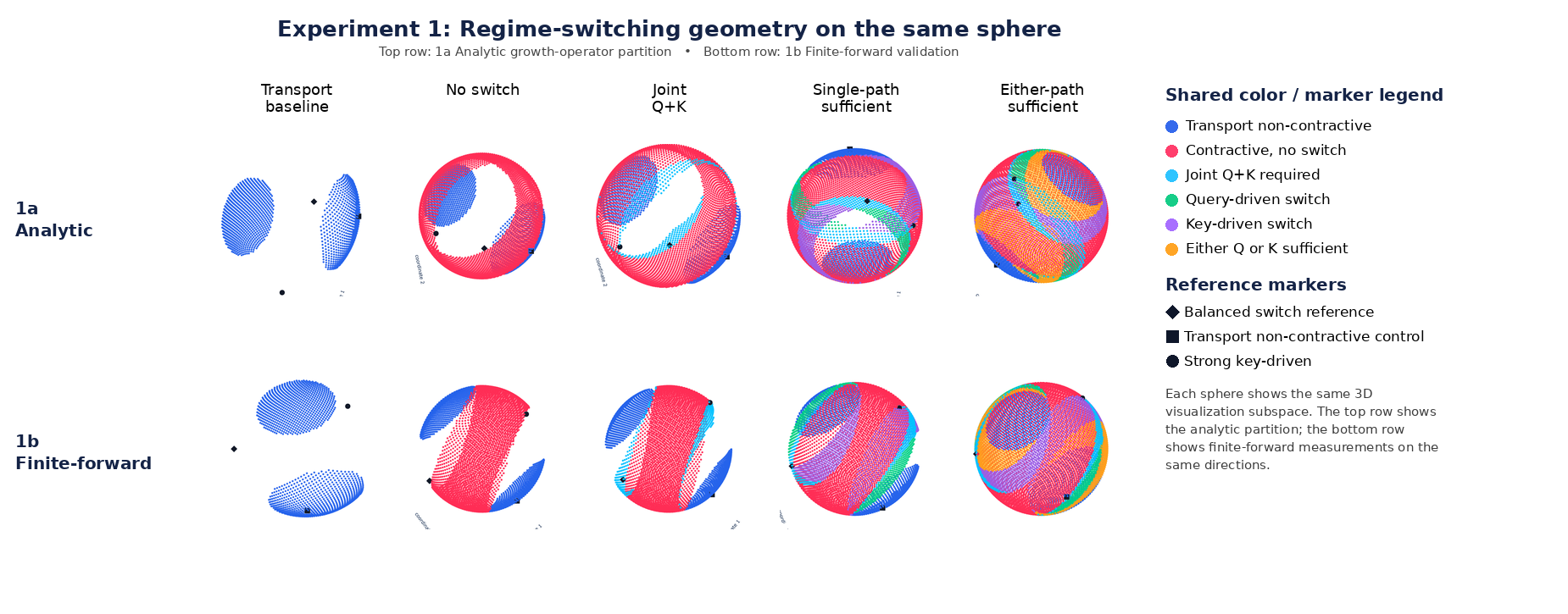}
    \caption{\textbf{Progressive classification of local switching regimes.}
    Five columns reveal the conditions defining transport non-contraction
    (blue), persistent contraction (red), joint query--key necessity
    (cyan), query-only or key-only sufficiency (green or purple), and
    either-pathway sufficiency (orange). The top and bottom rows apply
    analytic and finite-forward tests to the same selected directional
    slice. The classifications concern counterfactual growth signs, not an
    additive partition of total growth.}
    \label{fig:app_regime_switching_geometry}
\end{figure*}

In particular, the cyan class requires
$g_{T+Q}^M(v)\leq0$ and $g_{T+K}^M(v)\leq0$ even though
$g_{\mathrm{full}}^M(v)>0$: neither gating pathway suffices individually.
By contrast, green and purple directions cross the growth boundary with
query or key gating alone, while orange directions cross it under either
single-pathway counterfactual. These categories should not be interpreted
as independent contributions that sum to full growth; the squared-norm
growth operator contains interactions between pathways.

\paragraph{How large are the counterfactual margins?}
A regime label records a sign, but does not say whether the reversal is
marginal or substantial. Figure~\ref{fig:app_pathway_growth_examples}
therefore reports two selected directions. In panel (a), transport growth is
approximately $-0.963$, transport plus query gating reaches $0.006$, and
full growth reaches $0.289$. Query gating is sufficient, but its
single-pathway result lies close to the zero-growth boundary. Panel (b)
illustrates a different mechanism: transport growth is about $-0.967$;
transport plus query and transport plus key remain negative at $-0.184$
and $-0.952$, respectively; yet full growth is positive at $0.072$.
The sign reversal in panel (b) therefore genuinely requires the joint
query--key counterfactual at this operating point. These deliberately
selected directions expose possible mechanisms and margins; they do not
estimate their frequency in the full representation space.

\begin{figure*}[t]
    \centering
    \includegraphics[width=0.94\textwidth]{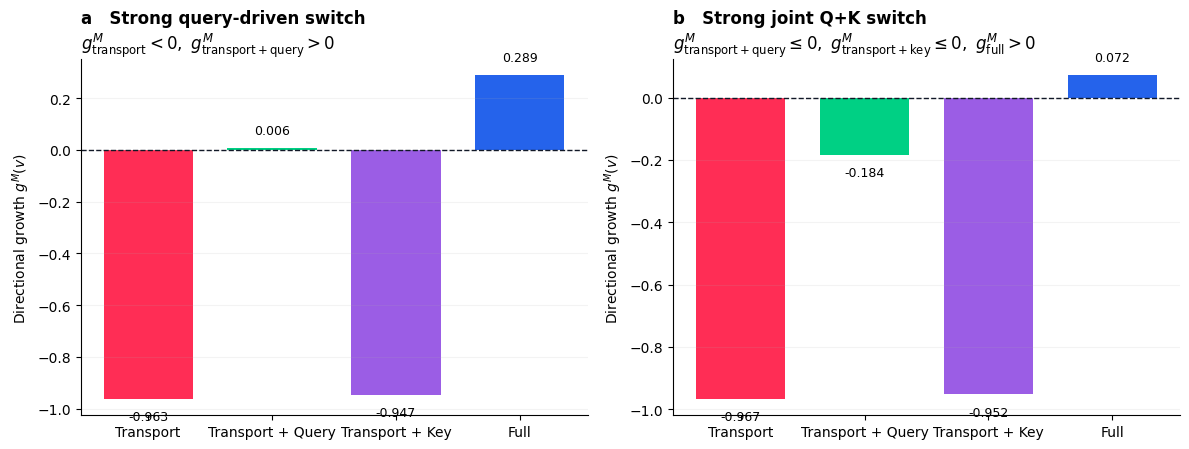}
    \caption{\textbf{Growth margins under pathway ablations.}
    \textbf{(a)} Query gating alone crosses the zero-growth boundary.
    \textbf{(b)} Neither query nor key gating alone produces expansion,
    whereas their joint inclusion does. Bars in each panel evaluate the
    same perturbation direction under different counterfactual dynamics;
    dashed lines mark zero growth.}
    \label{fig:app_pathway_growth_examples}
\end{figure*}

\paragraph{Agreement and its scope.}
The finite-forward row of Figure~\ref{fig:app_regime_switching_geometry}
reproduces the analytic mechanism partition on the tested slice. The
reported held-out evaluation at this selected operating point yields
$100\%$ agreement in switch/no-switch and mechanism labels, with an
analytic--finite-forward full-growth correlation reported as
$r=1.000000$ to the displayed precision. These results support the local
Jacobian approximation at the perturbation scale used. They do not certify
agreement for arbitrary perturbation magnitudes, other operating points,
or directions outside the evaluated slice.

\paragraph{From individual directions to severity-restricted subspaces.}
The main paper has already established the switching spectrum and effective
dimension at this operating point
(Figure~\ref{fig:exp_spectral_certification_severity}). We retain the
spectrum in Figure~\ref{fig:app_effective_dimension_severity}(a) here only
as context for the more restrictive experiment in panel (b). The spectrum
has $\lambda_{\max}(W)\approx1729.98$ and $d_{\mathrm{sw}}=47$ within
a $357$-dimensional transport-contractive subspace. These quantities
certify existence and the dimension of a uniformly switching subspace;
they do not measure how much contraction its switching directions reverse.

\begin{figure*}[t]
    \centering
    \includegraphics[width=0.96\textwidth]{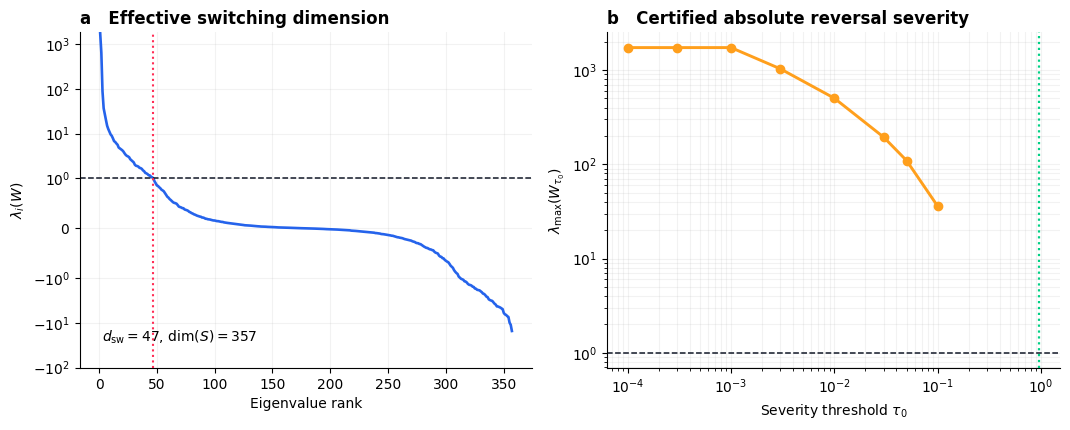}
    \caption{\textbf{Switching spectrum and severity-restricted certification.}
    \textbf{(a)} Ordered eigenvalues of $W=C^{-1/2}RC^{-1/2}$, shown for
    comparison with the severity test; the switching boundary is
    $\lambda=1$. \textbf{(b)} Leading eigenvalue of $W_{\tau_0}$ on
    subspaces constrained to have transport severity at least $\tau_0$.
    The green reference line indicates a separately computed
    $\tau_{\mathrm{cert}}\approx0.967$; the displayed orange sweep does
    not reach that threshold and is not, on its own, the certificate.}
    \label{fig:app_effective_dimension_severity}
\end{figure*}

Panel (b) asks a stronger question: after excluding directions whose
transport contraction is weaker than a threshold $\tau_0$, can gating
still reverse the remaining contraction? The restricted subspaces
$\mathcal S_{\tau_0}$ are nested as $\tau_0$ increases, so the spectral
test is applied to progressively more demanding candidates. The plotted
sweep illustrates this restriction, whereas the independently evaluated
certificate gives $\tau_{\mathrm{cert}}\approx0.967$. By
Theorem~\ref{thm:absolute_reversal_severity}, the certified threshold
establishes the existence of a switching direction on which transport
removes at least $96.7\%$ of the incoming squared magnitude. It is a
lower bound on the supremal reversal severity, not an estimate that can
be read directly from the displayed sweep.

\paragraph{Insight.}
The mechanism plots and spectral tests answer different questions. The
counterfactuals identify \emph{how} a particular direction crosses zero
growth; the spectrum establishes \emph{whether} switching is possible
and how many dimensions can switch uniformly; severity restriction asks
\emph{how much contraction} can still be overcome. Neither a large leading
switching eigenvalue nor a high agreement rate alone answers the last
question. We next examine the directions selected by relative and absolute
optimization and quantify the query--key interaction responsible for a
severe reversal.

\subsection{Strongest Switching Directions}
\label{subsec:app_exp_strongest}

\paragraph{Two objectives select different perturbations.}
Corollary~\ref{cor:max_relative_switching} identifies the direction
$v_{\mathrm{rel}}$ maximizing the relative switching ratio $\Gamma(v)$,
but this ratio is normalized by the contraction that gating must overcome.
As discussed in Section~\ref{subsec:exp_spectral_severity}, a large ratio
need not represent a severe contraction reversal. Figure~\ref{fig:app_relative_absolute_strength}
and Table~\ref{tab:app_strongest_switching_directions} provide the
operating-point values behind that distinction. The relative optimum
achieves $\Gamma(v_{\mathrm{rel}})\approx1729.98$ with transport
severity $\tau(v_{\mathrm{rel}})\approx0.003361$. The numerically
identified severe-switch direction instead has
$\Gamma(v_{\max})\approx1.07465$ and
$\tau(v_{\max})\approx0.967349$. Thus, a ratio only slightly above the
switching threshold can correspond to overcoming far stronger transport
suppression than the direction maximizing the ratio.

\begin{figure*}[t]
    \centering
    \includegraphics[width=0.94\textwidth]{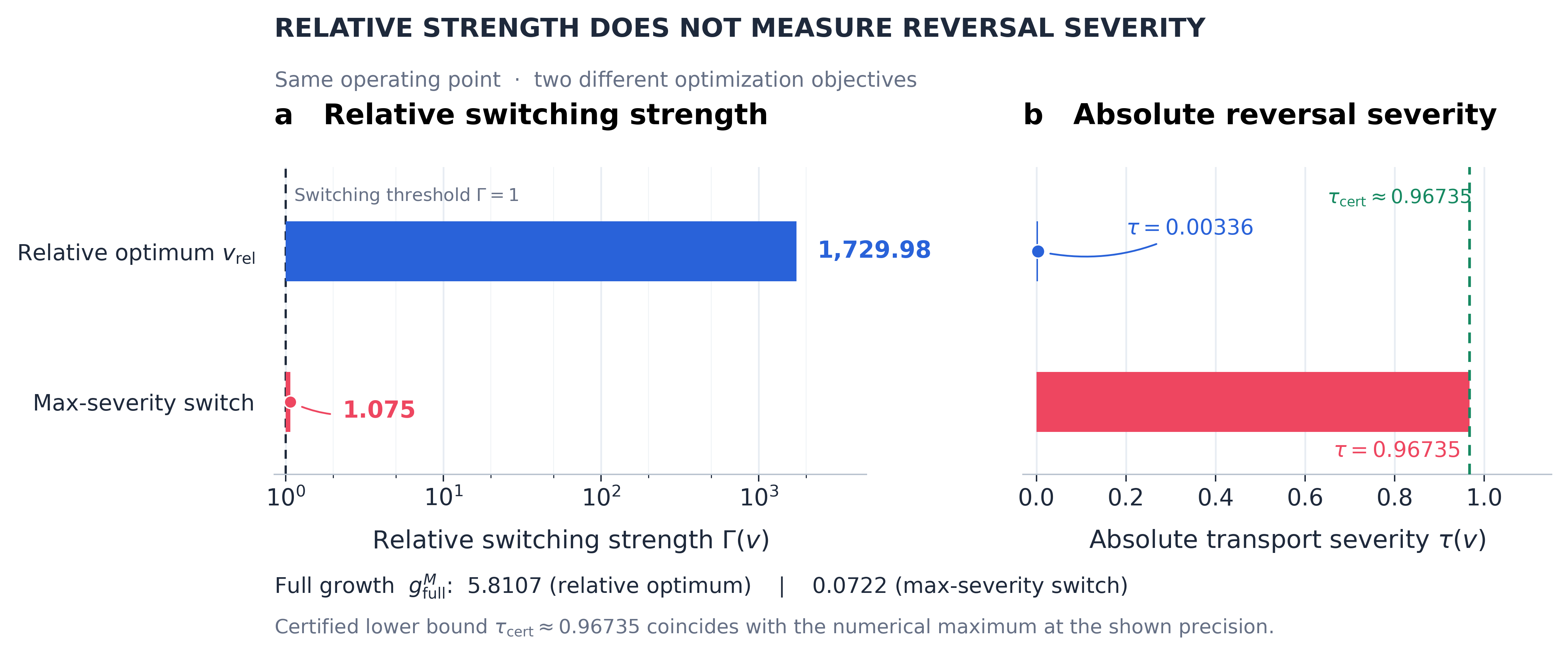}
    \caption{\textbf{Relative strength and absolute contraction severity.}
    \textbf{(a)} Relative switching ratios for the relative optimum and a
    numerically identified severe switch (logarithmic axis; dashed line at
    $\Gamma=1$). \textbf{(b)} Fractions of incoming squared magnitude
    removed by transport for the same directions. The green reference
    marks $\tau_{\mathrm{cert}}\approx0.967$. The two panels compare
    different objectives, not two estimates of the same quantity.}
    \label{fig:app_relative_absolute_strength}
\end{figure*}

\paragraph{Why the severe switch requires both gating pathways.}
The numerical contrast becomes more informative when we inspect the
counterfactual growth values in
Table~\ref{tab:app_strongest_switching_directions}. For
$v_{\mathrm{rel}}$, transport plus query gating already produces positive
growth ($5.7437$), while transport plus key gating remains negative
($-0.002801$). Query gating is therefore sufficient for this direction.
For the severe direction, both single-pathway outcomes remain negative,
$g_{T+Q}^M\approx-0.1842$ and $g_{T+K}^M\approx-0.952478$,
whereas $g_{\mathrm{full}}^M\approx0.0722$. The combined pathways
cross zero where neither does alone.

\begin{table}[t]
\centering
\caption{\textbf{Relative-optimal and severe-switch counterfactuals.}
Directional growth, transport severity, and relative strength at the
selected GPT-2 Medium operating point. Values are rounded.}
\label{tab:app_strongest_switching_directions}
\small
\setlength{\tabcolsep}{3.5pt}
\begin{tabular}{lrrrrrr}
\toprule
Direction & $g_{\rm T}^M$ & $g_{\rm T+Q}^M$ & $g_{\rm T+K}^M$ & $g_{\rm full}^M$ & $\tau(v)$ & $\Gamma(v)$ \\
\midrule
$v_{\rm rel}$ & $-0.003361$ & $5.7437$ & $-0.002801$ & $5.8107$ & $0.003361$ & $1729.98$ \\
Severe switch & $-0.967349$ & $-0.1842$ & $-0.952478$ & $0.0722$ & $0.967349$ & $1.07465$ \\
\bottomrule
\end{tabular}
\end{table}

The growth decomposition lets us go further than the sign labels. Because
$J_{\mathrm{full}}=J_T+J_Q+J_K$, the interaction between query and key
gating can be isolated from the four counterfactual evaluations:
\begin{equation}
\begin{aligned}
2\langle J_Qv,J_Kv\rangle_M
={}&g_{\mathrm{full}}^M(v)-g_{T+Q}^M(v)\\
&-g_{T+K}^M(v)+g_T^M(v).
\end{aligned}
\label{eq:app_exp_query_key_interaction}
\end{equation}
For the severe direction, the values in Table~\ref{tab:app_strongest_switching_directions}
yield a positive query--key interaction of approximately $0.2415$.
Indeed, including the key pathway on top of transport plus query changes
growth by about $0.2564$, of which approximately $0.2415$ comes from
this cross term. This makes the joint mechanism more precise: the severe
switch is not explained by adding two independent single-pathway growth
improvements. It depends substantially on their interaction in the full
quadratic growth operator.

\paragraph{What the severity comparison certifies.}
Theorem~\ref{thm:absolute_reversal_severity} and the definition of
$\tau_{\mathrm{cert}}$ imply
$\tau_{\max}\geq\tau_{\mathrm{cert}}$. At this operating point, the
reported certificate and numerical severe-switch value are both close to
$0.967349$. Their agreement is specific to the computation; it does not
establish the identity $\tau_{\max}=\tau_{\mathrm{cert}}$ for arbitrary
operators, nor does a relative optimum generally solve the absolute
severity problem.

\paragraph{Insight.}
The important distinction is not merely that the two optima have different
values. They probe different mechanisms: the relative optimum can succeed
through query gating against almost neutral transport, while a severe
reversal can require a positive query--key cross interaction to rescue a
direction that both single-pathway counterfactuals still suppress. Thus,
relative strength measures a \emph{ratio}, absolute severity measures the
\emph{contraction overcome}, and pathway ablations identify the
\emph{mechanism of rescue}. Conflating these three questions can make a
spectrally strong switch appear more substantial than it is.

\subsection{Cost of Theory-Free Search}
\label{subsec:app_exp_search_cost}

\paragraph{Ordinary switches versus severe switches.}
Figure~\ref{fig:exp_search_cost} in the main paper compares blind
finite-forward search, random sampling with oracle access to the
transport-contractive subspace $\mathcal S$, and spectral guidance. The
underlying budgets clarify what the comparison establishes. In the
reported experiment, blind ambient-space sampling finds no switch among
$8192$ tested directions. When sampling is restricted to $\mathcal S$,
$2928$ of $4096$ directions switch (approximately $71.5\%$), but none
reaches severity $\tau\geq0.80$; the best sampled severity is only about
$0.0718$. These outcomes should not be interpreted as estimates of a
zero ambient-space switching probability or of the probability of
recovering an extreme direction. They instead show that making the search
subspace favorable can render \emph{ordinary} switches easy to find without
making \emph{severe} switches easy to recover within the tested budget.

\paragraph{When does spectral preprocessing pay off?}
The two approaches also have different cost structures. Finite-forward
probing incurs work for each candidate direction, whereas spectral
analysis incurs an initial operator-construction and eigendecomposition
cost and then reuses quadratic forms for additional directional tests. The
reported setup time is approximately $0.0492\,\mathrm{s}$, with measured
break-even budgets of about $3993$ directions for switch detection and
$1981$ for full mechanism classification. Below these budgets,
finite-forward checks can be cheaper; beyond them, reusing the spectral
operators becomes advantageous in this implementation. These crossover
points depend on the hardware, implementation, operating point, and
precisely which evaluations are charged to each method. They are not
asymptotic complexity guarantees.

\paragraph{Insight.}
Spectral guidance provides two distinct benefits, and neither should be
overstated. It identifies a direction satisfying a global spectral
criterion within the specified subspace instead of hoping that random
samples approach an extreme direction; and, after preprocessing, it
amortizes the cost of evaluating many directions. The former addresses a
\emph{search-quality limitation}; the latter is a \emph{query-budget
trade-off}. Faster evaluation for large budgets does not imply that
spectral preprocessing is always preferable for a single perturbation.

\subsection{Cross-Architecture Switching and Probe Sensitivity}
\label{subsec:app_exp_cross_arch}

\paragraph{What the cross-model comparison actually samples.}
The architecture-level trends summarized in
Section~\ref{subsec:exp_cross_architecture} are based on projected, local
measurements rather than full hidden-space spectral decompositions. For
each of GPT-2 Medium, Gemma-2-2B, Llama-3.2-1B, and Qwen2.5-0.5B,
we use two WikiText-2 validation prompts, the final token, and seven
normalized depth positions, with a $64$-dimensional random orthonormal
probe subspace. Accordingly, the quantities below characterize the
projected transport-contractive region at those operating points. This
protocol matters when interpreting differences across architectures,
particularly for statistics normalized by the dimension of the retained
contractive subspace.

\begin{figure*}[t]
    \centering
    \includegraphics[width=0.94\textwidth]{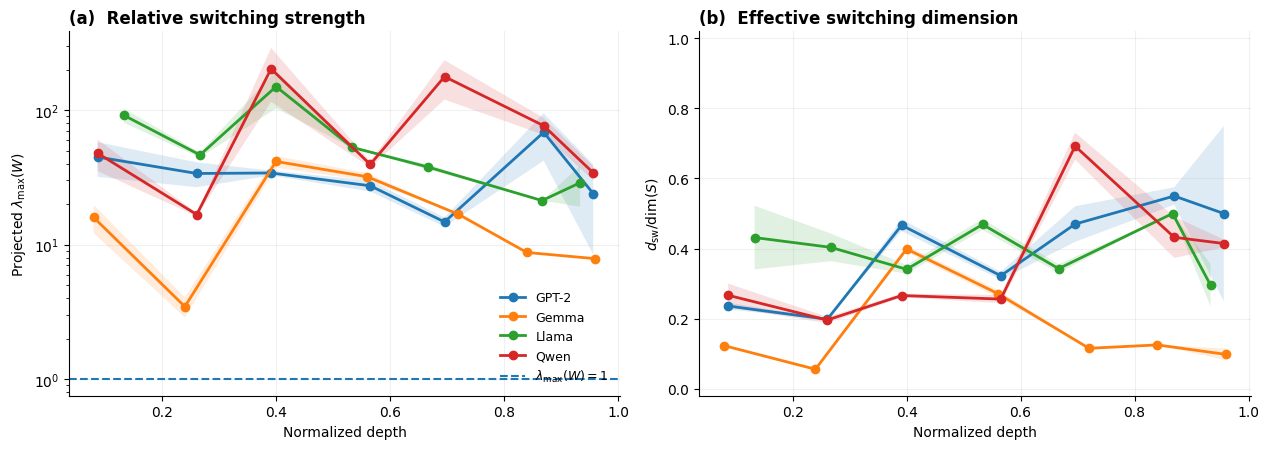}
    \caption{\textbf{Projected spectra across architecture and depth.}
    \textbf{(a)} Median leading switching eigenvalue and the $\lambda=1$
    threshold. \textbf{(b)} Median normalized effective switching
    dimension $d_{\mathrm{sw}}/\dim(\mathcal S)$ within the retained probe
    subspace. Shaded interquartile bands summarize variation across the
    sampled prompts and are not confidence intervals.}
    \label{fig:app_cross_arch_depth}
\end{figure*}

\paragraph{Spectral strength is not switching-subspace size.}
Figure~\ref{fig:app_cross_arch_depth} compares two summaries of the
projected switching operator. The leading eigenvalue in panel (a) measures
the strongest relative growth correction, whereas panel (b) measures the
fraction of retained contractive dimensions that support uniform
switching. The two statistics vary differently with model and depth: a
large leading eigenvalue need not accompany a large normalized switching
dimension. Neither panel shows a common monotonic depth law. This is an
important qualification of cross-model comparisons based on one spectral
summary alone.

\begin{figure*}[t]
    \centering
    \includegraphics[width=0.83\textwidth]{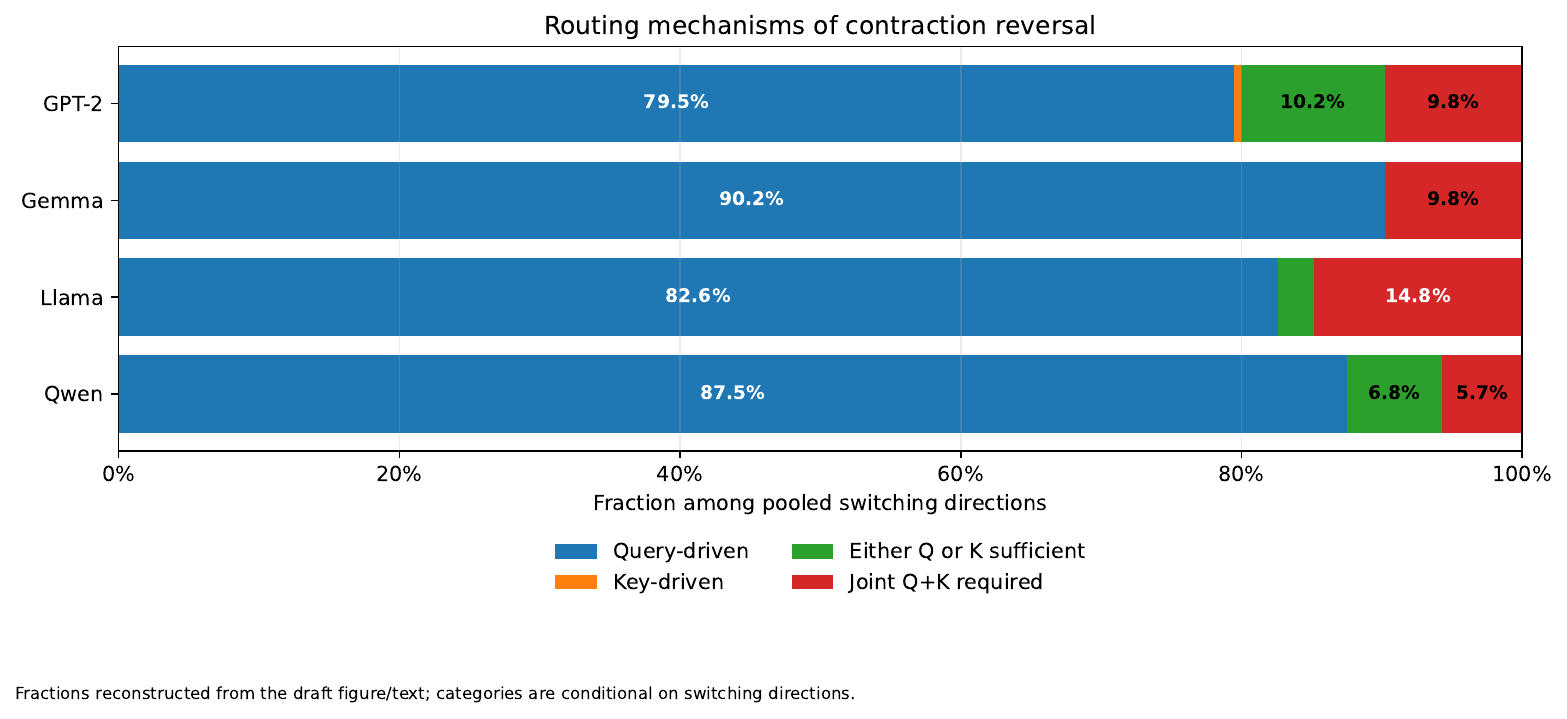}
    \caption{\textbf{Mechanism composition among projected switching directions.}
    Fractions are pooled over the sampled switching directions within each
    architecture. Query-only, key-only, joint query--key, and
    either-pathway-sufficient classes are disjoint under the specified
    counterfactual tests. Fractions are conditional on full-dynamics
    switching and are not ambient-space probabilities.}
    \label{fig:app_cross_arch_mechanisms}
\end{figure*}

\paragraph{Mechanism frequencies.}
Figure~\ref{fig:app_cross_arch_mechanisms} expands the aggregate
mechanism comparison in Figure~\ref{fig:exp_cross_architecture}(b).
Among the sampled switching directions, query-only sufficiency accounts
for approximately $79.5\%$ in GPT-2, $90.2\%$ in Gemma, $82.6\%$ in
Llama, and $87.5\%$ in Qwen. Joint query--key switching occurs in every
evaluated family, but its proportion and those of the other pathway
classes vary. The query-heavy composition is conditional on switching,
the selected prompts, and the query-token perturbation protocol; it does
not imply that query gating causally dominates other pathways for
arbitrary tokens, layers, or inputs.

\paragraph{Sensitivity to probe dimension.}
The projected measurements above leave a final methodological question:
would the same numerical summaries be obtained using a different probe
dimension? Figure~\ref{fig:app_probe_sensitivity} repeats the analysis
using nested orthonormal probe spaces of dimensions $32$, $64$,
$128$, and $256$ at three normalized depths and two prompts per
architecture. Increasing the probe dimension
generally raises the projected leading eigenvalue, consistent with
making more directions available to the restricted spectral problem.
However, the normalized effective switching dimension and the fraction
of query-sufficient sampled switches also change, and need not stabilize
at the same rate. In particular, the experiments do not establish
convergence to full-space switching statistics at dimension $256$.

\begin{figure*}[t]
    \centering
    \includegraphics[width=\textwidth]{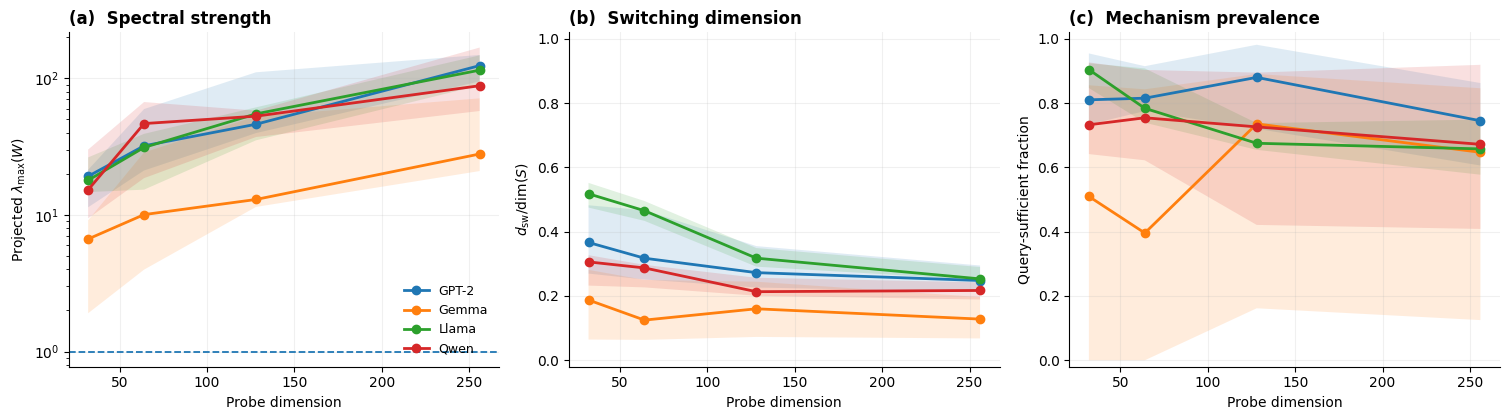}
    \caption{\textbf{Probe-dimension sensitivity of projected switching statistics.}
    Nested orthonormal probe spaces of dimensions $32$, $64$, $128$,
    and $256$ are evaluated at matched model--prompt--depth operating
    points. \textbf{(a)} Median projected $\lambda_{\max}(W)$;
    \textbf{(b)} median normalized effective switching dimension;
    \textbf{(c)} median query-sufficient fraction among sampled switches.
    Shaded regions indicate interquartile ranges across operating points,
    not confidence intervals or convergence guarantees.}
    \label{fig:app_probe_sensitivity}
\end{figure*}

\paragraph{Insight.}
The cross-architecture results support the \emph{applicability of the same
local mechanism and spectral tests} to several Transformer families, but
they do not establish architecture-independent switching frequencies or
universal depth laws. The sensitivity experiment explains why: a larger
probe can reveal stronger projected switching modes while changing the
apparent size and pathway composition of the switching set. Consequently,
claims about the existence of local switching within a measured subspace
are firmer than claims about its prevalence in the full hidden space.
Broader prompt sampling, other token positions, and more extensive probe
convergence tests would be needed to strengthen the latter claims.

\end{document}